\documentclass[lettersize,onecolumn]{IEEEtran}
\usepackage{amsmath,amsfonts}
\usepackage{titletoc}
\usepackage{array}
\usepackage[numbers]{natbib}
\usepackage{textcomp}
\usepackage{stfloats}
\usepackage{url}
\usepackage{verbatim}
\usepackage{wrapfig}
\usepackage{graphicx}
\usepackage{cite}
\usepackage{amsmath,amsfonts,bm}

\def\eqref#1{equation~\ref{#1}}
\def\1{\bm{1}}

\DeclareMathAlphabet{\mathsfit}{\encodingdefault}{\sfdefault}{m}{sl}
\SetMathAlphabet{\mathsfit}{bold}{\encodingdefault}{\sfdefault}{bx}{n}

\usepackage{hyperref}
\usepackage{cleveref}
\usepackage{url}
\usepackage{graphicx}
\usepackage{epstopdf}
\usepackage{float}
\usepackage{floatflt}

\usepackage{booktabs}
\usepackage{multirow}
\usepackage{makecell}
\usepackage{comment}
\usepackage{glossaries}
\usepackage{microtype}

\usepackage{amssymb}
\usepackage{mathtools}
\usepackage{amsthm}
\usepackage{thmtools}
\usepackage{listings}

\usepackage{tikz}
\usetikzlibrary{positioning,shapes,arrows,fit,calc,shadows}

\usepackage[shell, subfolder]{gnuplottex}
\usepackage{gnuplot-lua-tikz}

\usepackage{caption}

\usepackage[mode=buildnew]{standalone}

\usepackage{fontawesome5}
\usepackage{custom, customboxes}

\DeclareMathOperator{\proba}{\Pi}
\DeclareMathOperator{\copula}{\mathrm{C}}
\newcommand{\jproba}[2]{\proba_{#1,#2}}
\newcommand{\mproba}[2]{\proba_{#1} \! \otimes \proba_{#2}}
\def\PDF{\pi}

\def\CDF{F}
\newcommand{\jcopula}[2]{\copula_{#1,#2}}

\DeclareMathOperator{\expect}{\mathbb{E}}     % Expectation.
\DeclareMathOperator{\variance}{\mathrm{Var}} % Variance.
\def\identity{\mathrm{I}}

\DeclareMathOperator{\MI}{\mathsf{I}}
\DeclareMathOperator{\dent}{\mathsf{h}}
\DeclareMathOperator{\sent}{\mathsf{H}}

\DeclareMathOperator{\PMI}{\mathrm{PMI}}

\DeclareMathOperator*{\DKLoperator}{\mathsf{KL}}%{D_{\mathrm{KL}}}
\newcommand{\KL}[2]{\DKLoperator \left[ #1 \, \Vert \, #2 \right]}

\def\naturals{\mathbb{N}}

\def\reals{\mathbb{R}}

\def\bernoulli{\mathrm{Be}}
\def\normal{\mathcal{N}}
\def\uniform{\mathrm{U}}
\def\loggamma{\mathrm{L}}

\def\digamma{\psi}
\def\const{\textnormal{const}}

\def\blankarg{\,\cdot\,}
\DeclareMathOperator{\dif}{\mathrm{d} \!} %% Our commands.
\newcommand{\eqdef}{\overset{\text{def}}{=}
}
\newcommand{\expected}{\mathbb{E}}

\newcommand{\flow}{\hat{f}}
\newcommand{\vfield}{v}
\newacronym{ODE}{\textnormal{\textsc{ode}}}{Ordinary Differential Equation}
\newacronym{CTMC}{\textnormal{\textsc{ctmc}}}{Continuous Time Markov Chain}
\newacronym{KL}{\textnormal{\textsc{kl}}}{Kullback-Leibler}
\newacronym{MI}{\textnormal{\textsc{mi}}}{Mutual Information}
\newacronym{PMI}{\textnormal{\textsc{pmi}}}{Pointwise Mutual Information}
\newacronym{MLP}{\textnormal{\textsc{MLP}}}{Multi-Layer-Perceptron}
\newacronym{SDE}{\textnormal{\textsc{sde}}}{Stochastic Differential Equation}
\newacronym{CDF}{\textnormal{\textsc{cdf}}}{Cumulative Distribution Function}
\newacronym{PDF}{\textnormal{\textsc{pdf}}}{Probability Density Function}
\newacronym{CNF}{\textnormal{\textsc{CNF}}}{Continuous Normalizing Flow}
\newacronym{DPI}{\textnormal{\textsc{DPI}}}{Data Processing Inequality}

\newacronym{KDE}{\textnormal{\textsc{kde}}}{Kernel Density Estimation}
\newacronym{kNN}{$k$-\textnormal{\textsc{nn}}}{$k$-Nearest Neighbours}
\newacronym{GEBM}{\textnormal{\textsc{gebm}}}{Generalized Energy-Based Models}

\newacronym{EMA}{\textnormal{\textsc{ema}}}{Exponential Moving Average}

\theoremstyle{plain}
\declaretheorem[name=Theorem,numberwithin=section]{theorem}
\declaretheorem[name=Lemma,numberwithin=section,sibling=theorem]{lemma}
\declaretheorem[name=Proposition,numberwithin=section,sibling=theorem]{proposition}
\declaretheorem[name=Corollary,numberwithin=section,sibling=theorem]{corollary}

\declaretheorem[name=Example,numberwithin=section,sibling=theorem]{example}

\theoremstyle{definition}

\theoremstyle{remark}

\usepackage[commandRef=cref]{proof-at-the-end}
\newEndThm[end, restate, text link=]{theoremE}{theorem}
\newEndThm[end, restate, text link=]{corollaryE}{corollary}
\newEndThm[end, restate, text link=]{statementE}{statement}
\newEndThm[end, restate, text link=]{lemmaE}{lemma}
\newEndThm[end, restate, text link=]{remarkE}{remark}
\newEndProof{proofE}

\usepackage{lastpage}
\usepackage{subcaption}
\begin{document}

\title{Towards Diverse and Comprehensive Benchmarks for \\ Mutual Information Estimation}

\author{Alberto~Foresti$^{*}$, Giulio~Franzese, Pietro~Michiardi, Ivan~Butakov$^{*}$, Alexander~Tolmachev, and~Alexey~Frolov%
\thanks{$^{*}$A.~Foresti and I.~Butakov contributed equally to this work.}%
\thanks{A.~Foresti, G.~Franzese, and P.~Michiardi are with the Department of Data Science, EURECOM, 06410 Biot, France (e-mail: alberto.foresti@eurecom.fr; giulio.franzese@eurecom.fr; pietro.michiardi@eurecom.fr).}%
\thanks{I.~Butakov and A.~Tolmachev are with the Skolkovo Institute of Science and Technology, Bolshoy Boulevard, 30, bld. 1, 121205 Moscow, Russia, and also with the Moscow Independent Research Institute of Artificial Intelligence, 117218 Moscow, Russia (e-mail: i.butakov@skoltech.ru; a.tolmachev@skoltech.ru).}%
\thanks{A.~Frolov is with the Skolkovo Institute of Science and Technology, Bolshoy Boulevard, 30, bld. 1, 121205 Moscow, Russia (e-mail: al.frolov@skoltech.ru).}}

% \author{IEEE Publication Technology,~\IEEEmembership{Staff,~IEEE,}\thanks{This paper was produced by the IEEE Publication Technology Group. They are in Piscataway, NJ.}\thanks{Manuscript received October 27, 2026; under review.}}

%\markboth{IEEE Transactions on Information Theory -- Submitted October 2026}%
%\markboth{Preprint -- Under Review}%

%\IEEEpubid{0000--0000~\copyright~2023 IEEE}

\maketitle

\begin{abstract}
Mutual information estimation is a central problem in machine learning and statistics; however, existing benchmarks typically evaluate estimators on simplified, low-dimensional distributions, leaving their performance on complex, realistic data largely unexplored.
We address this gap with a comprehensive benchmarking framework grounded in a unified copula-theoretic perspective that subsumes existing benchmarks as special cases.
Within this framework, we propose two complementary families of tests:
a \emph{copula-first} family that systematically varies ground-truth mutual information, dimensionality, and marginal complexity using synthetic and flow-based transformations;
and a \emph{marginals-first} family that couples real-world image data with controlled dependency structures, extending the classic same-class-pairing paradigm.
We use this suite to extensively evaluate three classes of estimators: non-parametric, variational, and generative.
Contrary to prevailing assumptions, our results indicate that \emph{there is no universal winner}:
each category can systematically outperform all other estimators under specific setups.
By analyzing these cases, we identify fundamental estimation barriers and propose new tests that more effectively stress these specific limitations. We share the open source code at \url{https://github.com/VanessB/mutinfo}.
\end{abstract}


\section{Introduction}

\gls{MI} quantifies the statistical dependence between random variables and is fundamental to both theoretical frameworks and practical applications in machine learning, information theory, and statistical inference~\citep{cover2006information_theory,polyanskiy2024information_theory}.
Accurate estimation of \gls{MI} is essential for a wide array of tasks, ranging from representation learning~\citep{hjelm2018deep_infomax,oord2019infoNCE,tschannen2020on_DIM} and generative modeling~\citep{chen2016InfoGAN,wang2025PMI_alignment,wang2025RFMI,franzese2025latent_abstractions} to feature selection~\citep{sulaiman2015MI_feature_selection,huang2024MI_feature_selection_TS}
and hypothesis testing~\citep{runge2018conditional_independence_testing_MI}.
Despite significant progress, existing \gls{MI} estimation methods exhibit substantial variance in performance, interpretability, and scalability, highlighting the need for comprehensive and systematic benchmarking frameworks~\citep{czyz2023beyond_normal}. In this work, we review and extend previous efforts to establish such frameworks
by proposing an extensive benchmarking suite designed to rigorously evaluate classical,
neural-based, and advanced continuous-time \gls{MI} estimators across diverse scenarios.
Our study encompasses both classical statistical distributions and novel,
specifically engineered tasks intended to isolate the strengths and weaknesses of different estimator classes.
We also disentangle estimator-specific limitations from some inherent problem challenges
(e.g., numerical instability and intrinsic variance),
thereby providing new insights into the fundamental difficulties of \gls{MI} estimation.

\subsection{Background and Motivation}

Classical \gls{MI} estimators, particularly those based on $k$-nearest neighbor ($k$-NN) approaches~\citep{kozachenko1987entropy_of_random_vector,kraskov2004KSG,berrett2019efficient_knn_entropy_estimation},
have traditionally provided robust and theoretically grounded baselines.
However, these methods often struggle in high-dimensional or complex settings, necessitating the development of parametric neural approaches.
Methods such as \gls{MI} Neural Estimation (MINE)~\citep{belghazi2018mine} and InfoNCE~\citep{oord2019infoNCE}
shifted the paradigm by leveraging neural networks to capture intricate, nonlinear dependencies,
significantly improving performance in higher-dimensional contexts.

While early neural \gls{MI} estimators can be viewed as energy-based models~\citep{arbel2021generalized_EBMs},
recent advancements leverage more sophisticated architectures.
These include normalizing flows~\citep{duong2023normflows_for_conditional_independence_testing,duong2023dine,butakov2024normflows,dahlke2025flowbased},
which utilize invertible transformations for density-based estimation,
and continuous-time estimators based on the Girsanov transform,
such as MINDE and its variants~\citep{franzese2024minde,kholkin2025InfoBridge,foresti2026information}.
While more computationally expensive,
these modern approaches introduce greater flexibility and scalability.
%further expanding the landscape of \gls{MI} estimation.

Despite these methodological advances, existing evaluations are typically restricted to simplified or artificial scenarios,
leaving a comprehensive and realistic benchmarking landscape largely unexplored.
For instance, current tests often fail to assess robustness against increasingly complex distributions
or systematic variations in critical parameters such as dimensionality, sample size, and augmentation strategies.
Furthermore, the absence of a unified theoretical framework hinders the systematic organization of benchmarks
and the principled exploration of new methods.
Consequently, existing results often provide a superficial overview that fails to explain the underlying reasons for estimator failure.
This has led to the oversimplified and \emph{potentially misleading} conclusion that computationally expensive estimators are invariably superior.
We challenge this notion by identifying specific setups that are uniquely difficult for distinct classes of estimators.

%\subsection{Challenges and Limitations of Existing Benchmarks}

%Current MI benchmarks often fail to comprehensively assess estimator performance, particularly regarding robustness to increasingly complex distributions or systematic variations in critical parameters such as dimensionality, sample size, and augmentation strategies.
%Furthermore, the absence of a common theoretical framework makes these benchmarks difficult to organize and hinders the exploration of new methods.
%Consequently, evaluation results provide neither a comprehensive nor an explanatory overview of the strengths and weaknesses of different \gls{MI} estimators.

%Moreover, we observe that popular classifier-based estimators face notable challenges in complex scenarios, often rendering classification tasks comparably difficult to generative modeling \cite{}.

\subsection{Our Contributions}

\paragraph{Unified theoretical framework}
We demonstrate that existing \gls{MI} benchmarks can be unified under a single copula-theoretic perspective centered on the data processing inequality.
This unification reveals the structural assumptions underlying various benchmark families and enables principled comparison and extension.
Moreover, it allows us to decompose the challenges of \gls{MI} estimation into two primary sources:
\emph{complex copulas} (dependency structures) and \emph{complex marginal distributions}.

\paragraph{Comprehensive benchmark suite}
We introduce two complementary families of benchmarks covering the spectrum from purely synthetic to real-world data:
(i) a \emph{copula-first} family of synthetic and flow-based tests
that enables controlled variation of dependency structure, dimensionality, and marginal complexity; and
(ii) a \emph{marginals-first} family that extends same-class pairing to general data-coupling, enabling benchmarks on datasets such as MNIST and CIFAR-10 with known ground-truth \gls{MI}. The former facilitates the exploration of how estimators react to different dependency structures, while the latter provides setups with realistic and highly complex marginal distributions.

\paragraph{Analysis of \gls{MI} estimation barriers}
We revisit known difficulties in \gls{MI} estimation and argue that they are insufficient to fully explain observed performance patterns.
We therefore identify and analyze an additional set of fundamental barriers,
%which we use to design stress-test benchmarks:
which can be used to further stress-test \gls{MI} estimators:
\begin{itemize}
    \item \textbf{Exponential sample complexity in terms of \gls{MI}};
    \item \textbf{Numerical instability};
    \item \textbf{High intrinsic variance of the problem};
    \item \textbf{Sensitivity to small noise injections}.
\end{itemize}

\paragraph{Systematic evaluation of four estimator classes}
We conduct an extensive evaluation of selected non-parametric (KSG, WKL), variational (MINE, NWJ, InfoNCE), generative (MINDE) estimators across all benchmark families.
%Results show that all methods degrade at high MI, that diffusion-based and flow-based generative estimators are most robust in these regimes, and that variational estimators benefit substantially from feature extraction on raw image data while generative methods do not.
We document estimator sensitivity to dimensionality, sample size, and model-specific parameters,
providing actionable guidance for practical estimator selection under various settings.
\begin{simplebox}
\textbf{Main Takeaway.} Contrary to prevailing assumptions, our results indicate that there is \emph{no universal winner}:
for each broad class of methods, we identify specific scenarios in which they systematically outperform all other approaches.
\end{simplebox}

%\noindent \textbf{Analysis of hyperparameter sensitivity.}
%We document estimator sensitivity to dimensionality, sample size, and model-specific parameters,
%providing actionable guidance for practical estimator selection.

The remainder of the paper is structured as follows: \Cref{section:background} provides the necessary theoretical background; \Cref{section:benchmarks} details the proposed benchmarking tasks; Section~\ref{section:experiments} presents the experimental results and comparative analyses; and Section~\ref{section:discussion} summarizes our contributions and outlines future research directions.
\section{Background}
\label{section:background}

\subsection{Theory}
\label{subsection:theory}

\textbf{Vector copulas.} For any pair of random vectors $X$ and $Y$, the joint distribution $\jproba{X}{Y}$ fully captures their probabilistic dependencies.
However, $\jproba{X}{Y}$ also includes information about marginal distributions $\mproba{X}{Y}$,
which are generally not relevant for dependency analysis.
In many cases, this marginal information can be discarded by ``factoring'' $\jproba{X}{Y}$ w.r.t. $\mproba{X}{Y}$:
two joint distributions are considered equivalent if one can be obtained from the other
via a push-forward using a pair of bijective measurable maps
(or, more generally, injective measurable stochastic maps) applied to the marginals individually:
$\jproba{X}{Y}^2 = \jproba{X}{Y}^1 \circ (f_X^{1 \to 2} \times f_Y^{1 \to 2})$.
Fixing the marginal distributions then selects at most one representative from each equivalence class. By pinning the marginals to be uniform on a fixed support
(typically the unit hypercube in $\reals^d$ for the continuous case),
one obtains a \emph{copula distribution} or \emph{copula}~\citep{fan2021vector_copulas}; hereinafter, we denote this by $\jcopula{X}{Y} = \copula(\jproba{X}{Y})$.
Previous works have established sufficient conditions for copulas to exist, which happen to be very mild for absolutely continuous distributions~\citep{mccann1995measure_preserving_maps,fan2021vector_copulas}.
Furthermore, in this case $\jcopula{X}{Y}$ admits a density w.r.t. the corresponding uniform distribution if and only if the Radon-Nikodym derivative $\partial \jproba{X}{Y} / \partial \mproba{X}{Y}$ exists.
Finally, by construction, $\jcopula{X}{Y}$ is invariant under bijective measurable transformations of the individual marginals:
$\jcopula{X}{Y} = \jcopula{X}{g(Y)}$.
For a \emph{discrete} generalization of this concept, we refer the reader to~\citep{geenens2020discrete_copula}. Although copulas provide a complete description of dependencies,
their complexity often makes them impractical to use directly.
Therefore, $\jcopula{X}{Y}$ is often projected onto a single scalar that summarizes the overall strength
of the dependency between $X$ and $Y$~--- \gls{MI}.

\textbf{Mutual Information.} \gls{MI} is a well-established metric for quantifying non-linear dependencies between two random vectors $X,Y$
of any nature (continuous, discrete or mixed),
defined as KL-divergence between the joint $\jproba{X}{Y}$ and marginal $\mproba{X}{Y}$ distributions:
\begin{equation}
    \label{eq:mutual_information}
    \MI(X;Y)=\KL{\jproba{X}{Y}}{\mproba{X}{Y}} = \expect \log \frac{\partial \jproba{X}{Y}}{\partial \mproba{X}{Y}} (X,Y),
\end{equation}
where $\log (\partial \jproba{X}{Y} / \partial \mproba{X}{Y})$ is \gls{PMI}.
Note that \gls{PMI} is essentially the log-density of the $(X,Y)$-copula $\jcopula{X}{Y}$,
and \gls{MI} is $\KL{\jcopula{X}{Y}}{\uniform_{\mathcal{X} \times \mathcal{Y}}}$.
Provided that the corresponding \gls{PDF} $\PDF(x,y)$ exists,
\begin{equation}
    \label{eq:mutual_information_PDF}
    \MI(X;Y)=\expect\left[\log\frac{\PDF(X,Y)}{\PDF(X)\PDF(Y)}\right] = \dent(X) + \dent(Y) - \dent(X,Y),
\end{equation}
where $\dent(X) = -\expect_X \log \PDF(X)$ is \emph{differential entropy}. Moreover, in this case $\MI(X;Y) = -\dent(\jcopula{X}{Y})$~\citep{ma2011MI_is_copula_entropy,chen2025vector_copula}. Note that \gls{MI} equals zero if and only if the random vectors are independent. It also does not increase under data processing and is \emph{invariant} to bijective mappings:% if $g$ is one-to-one and measurable, then $\MI(X;Y) = \MI(X;g(Y))$~\citep[Theorem 3.2(d)]{polyanskiy2024information_theory}.
\begin{theorem}[Theorem 3.2(d) in~\citep{polyanskiy2024information_theory}]
    \label{theorem:MI_invariance}
    If $g$ is measurable, $\MI(X;g(Y)) \leq \MI(X;Y)$.
    If it is also one-to-one with measurable inverse, $\MI(X;g(Y)) = \MI(X;Y)$.
\end{theorem}
The first property is a special case of the \gls{DPI}, %~\citep[Theorem 3.7(c)]{polyanskiy2024information_theory}, % 
while the second can be derived either from applying \gls{DPI} both ways,
or from similar invariance properties of copulas and \gls{PMI}~\citep{butakov2024normflows,czyz2025PMI}.
The \gls{DPI} itself is formulated as follows:
    \begin{theorem}[Theorem 3.7(c) in~\citep{polyanskiy2024information_theory}]
    \label{theorem:MI_DPI}
    For a Markov chain $X \to W \to V \to Y$ the following holds if corresponding \gls{MI}s exist: $\MI(X;Y) \leq \MI(V;W)$.
\end{theorem}

\subsection{Mutual Information Estimators}
\label{subsection:estimators}

Understanding contemporary MI estimation techniques and their limitations is essential for establishing a comprehensive set of benchmarks. \gls{MI} estimators fall into two broad categories: \emph{parametric} and \emph{non-parametric}. Within parametric estimators, we further distinguish between \emph{variational} and \emph{generative} approaches. Among the latter, \emph{diffusion-based} methods are particularly noteworthy.

\textbf{Non-parametric estimators.} This class of estimators relies on non-parametric methods for density estimation, such as histogram-based~\citep{moddemeijer1989binning_MI}, \gls{KDE}~\citep{moon1995KDE_MI, %wkeglarczyk2018KDE,
goldfeld2018estimating} or \gls{kNN}~\citep{kozachenko1987entropy_of_random_vector,berrett2019efficient_knn_entropy_estimation,kraskov2004KSG}, for approximating the quantities in \eqref{eq:mutual_information_PDF} and computing \gls{MI} with Monte Carlo estimation.
While these methods typically provide a very cheap estimate, previous works~\citep{goldfeld2018estimating,czyz2023beyond_normal,butakov2024normflows} have shown that they easily break in high-dimensional settings.

\textbf{Variational estimators.} This category measures \gls{MI} by estimating $\KL{\jproba{X}{Y}}{\mproba{X}{Y}}$ directly,
e.g., by learning the density ratio $\PDF(x,y)/\PDF(x)\PDF(y)$, without learning the densities themselves. Currently, the variational family is represented by classifier-based estimators, which rely on training a critic network $T$ to distinguish
whether a pair $(x, y)$ was sampled from the joint distribution $\jproba{X}{Y}$ or from the independent marginals $\mproba{X}{Y}$,
hence the name \textit{variational}.
As shown in \Cref{tab:mi_estimators}, typical classification objectives provide a lower bound on \gls{MI};
moreover, corresponding optimal classifiers can be expressed through \gls{PMI}~\citep{poole2019variational_bounds}.
Other losses can also be used, given they also result in estimating \gls{PMI}~\citep{letizia2024fDIME}.
%Althogh
These objectives are similar to those of \gls{GEBM}~\citep{arbel2021generalized_EBMs}:
\gls{PMI} is the log-Radon-Nikodym derivative of $\jproba{X}{Y}$ w.r.t. $\mproba{X}{Y}$ and thus acts as the negative energy function.%,
%they are generally not considered to be generative, as $\proba_X$, $\proba_Y$ and $\jproba{X}{Y}$ are not estimated.

{\begin{table}[h]
\centering
\begin{tabular}{ll}
\toprule
\textbf{Estimator} & \textbf{\gls{MI} Estimation Formula} \\
\midrule
MINE \citep{belghazi2018mine}& 
$\displaystyle \MI(X;Y) = \sup_T \left[ \expected_{\jproba{X}{Y}}[T(X,Y)] - \log \expected_{\mproba{X}{Y}}[e^{T(X,Y)}] \right]$ \\[10pt]

NWJ \citep{nguyen2010NWJ}& 
$\displaystyle \MI(X;Y) = \sup_T \left[ \expected_{\jproba{X}{Y}}[T(X,Y)] - \expected_{\mproba{X}{Y}}[e^{T(X,Y) - 1}] \right]$ \\[10pt]

InfoNCE \citep{oord2019infoNCE}& 
$\displaystyle \MI(X;Y) \geq \sup_T \left[ \expected_{\jproba{X}{Y}^k}\left[\frac{1}{K}\sum_{i=1}^K\log \frac{e^{T(X_i, Y_i)}}{\sum_{j=1}^K e^{T(X_i, Y_j)}}\right] \right]$ 
\\
\bottomrule
\end{tabular}
\caption{Popular discriminative mutual information estimators and their lower bound formulations}
\label{tab:mi_estimators}
\end{table}

While classifier-based estimators are universal ($\jproba{X}{Y}$ can be of any kind) and well-suited for high-dimensional and complex distributions,
they exhibit significant theoretical limitations \citep{mcallester2020limitations_MI, song2020understanding_limitations},
particularly when actual \gls{MI} is large, which can hinder their practical applicability.
This mirrors the failure of \gls{GEBM}s when the modeled distribution diverges significantly from a reference.

Recent works suggest that these limitations can only be overcome using a telescoping chain of classifiers
(either discrete~\citep{rhodes2020TRE} or continuous~\citep{choi2022DRE-infinity}).
This approach, however, requires constructing a non-trivial probability path between $\jproba{X}{Y}$ and $\mproba{X}{Y}$
and is more computationally expensive.

%optimizing a lower bound on the \gls{KL} divergence by
%estimate a lower bound on \gls{MI}.
%As shown in \Cref{tab:mi_estimators}, these bounds typically require $T$ to simultaneously maximize an expectation over the joint \gls{PDF} $\pi(x,y)$ and minimize an expectation over the product of the marginals $\pi(x)\pi(y)$.
%As shown in \Cref{tab:mi_estimators}, these estimators typically employ lower bounds on \gls{KL} divergence.
%In practice, the critic $T$ is trained to distinguish whether a pair $(x, y)$ was sampled from the joint distribution $\pi(x, y)$ or from the independent marginals $\pi(x)\pi(y)$, hence the name \textit{variational}.

\textbf{Non-diffusion-based generative estimators.} Neural generative models, like their non-parametric counterparts, can learn the data distribution's \gls{PDF}
and provide a plug-in \gls{MI} estimate:
$\hat{\PDF}(x)$, $\hat{\PDF}(y)$, and $\hat{\PDF}(x, y)$ are substituted into~\eqref{eq:mutual_information_PDF} to yield $\hat{\MI}(X;Y)$. This approach was adopted in~\citep{mcallester2020limitations_MI,duong2023dine},
where variational autoencoders~\citep{kingma2013VAE} and normalizing flows~\citep{tabak2013nonparametric_PDF_estimators} were employed to estimate the densities.
However, na\"{i}ve \gls{PDF}-based \gls{MI} estimation has been shown to introduce significant bias~\citep{mcallester2020limitations_MI},
as true marginals of $\hat{\PDF}(x,y)$ typically do not align with $\hat{\PDF}(x)$ and $\hat{\PDF}(y)$
(e.g., $\hat{\PDF}(x) \neq \int \hat{\PDF}(x,y) \dif y$).
While this issue can be partially mitigated by using a single network to model all densities at once~\citep{ni2025neural_DoE},
such methods lack formal theoretical guarantees. To provide a theoretical justification for this bias-reducing strategy,
one can consider generative models that inherently yield correct \gls{PMI}.
Gaussian mixtures are one such example~\citep{czyz2025PMI}.
However, vanilla Gaussian \gls{KDE} is not suitable for complex, high-dimensional data.
Addressing this, \citep{butakov2024normflows} employ normalizing flows $\flow_1,\flow_2$ to separate copula and marginals estimation:
$X,Y$ are transformed to a new pair of random vectors $\flow_1(X),\flow_2(Y) \eqdef \xi,\eta$
%as a way to map $X,Y$ to a new pair of random vectors $\flow_1(X),\flow_2(Y) \eqdef \xi,\eta$
whose \gls{MI} estimation should be easier due to simple marginals (note that \gls{MI} is preserved since flows satisfy~\Cref{theorem:MI_invariance}). Using Gaussian mixtures or even just a single correlated Gaussian mode to estimate $\PDF(\flow_1(x),\flow_2(y))$ enables tractable \gls{PMI} or even \gls{MI}.
Other choices of coupling distributions are also possible;
however, they always follow a trade-off between the expressivity and the \gls{MI} estimation complexity.
This trade-off is formalized by the following result~\citep[Corollary 3.2]{butakov2024normflows},
which holds for estimates $\hat{\PDF}$ with aligned joint and marginal densities
(i.e., $\hat{\PDF}(x) = \int \hat{\PDF}(x,y) \dif y$ and vice versa):
\begin{equation}
    \label{eq:MI_estimation_error}
    |\MI(\xi;\eta) - \MI_{\hat{\PDF}}(\xi;\eta)| \leq \KL{\PDF}{\hat{\PDF}},
    \qquad \text{where }
    \MI_{\hat{\PDF}}(\xi;\eta) \eqdef \expect_{\xi,\eta \sim \PDF} \log \frac{\hat{\PDF}(\xi,\eta)}{\hat{\PDF}(\xi)\hat{\PDF}(\eta)}
\end{equation}
Thus, simple $\hat{\PDF}$ facilitate MI estimation, but also generally increase $\KL{\PDF}{\hat{\PDF}}$.
Normalizing flows $\flow_1$,$\flow_2$ tighten this trade-off
by allowing for complex marginals while maintaining tractable density ratios.
Still, they cannot eliminate it completely, since a true copula can be much more intricate than a corresponding estimate. This \gls{MI}-preserving approach is also featured in~\citep{duong2023dine} (Gaussian case only) and~\citep{dahlke2025flowbased,chen2025vector_copula}.
While this class of estimators performs well even on high-dimensional, high-\gls{MI} tests, it introduces a strong inductive bias
(as one has to choose a narrow family of prior/latent distributions with tractable \gls{PMI}),
which hinders applications to real-world data. We summarize this section in~\Cref{tab:mi_non_diffusion_estimators}.

%However, since the family of the distribution of $\flow_1(X),\flow_2(X)$ is unknown, it is approximated by a distribution $\hat{\Pi}$ constrained to be part of a known family of distributions $\family$ decided by the user.
%This introduces a trade-off in the choice of $\family$: if it is too simple the \gls{MI} estimate might be too biased, while if $\family$ is a universal \gls{PDF} approximator family, one acquires a consistent and expensive \gls{MI} estimate.
%Then, the problem reduces to learning appropriate normalizing flows $\flow_1,\flow_2$ and joint \gls{PDF} $\hat{\pi}$ through maximum likelihood.
%After training, if $\family$ allows it, one can analytically compute \gls{MI}, otherwise one resorts to a Monte Carlo estimate of
%\begin{equation}
%    \expected_{x,y\sim\pi}\left[\log\frac{\hat{\pi}(\flow_1(x),\flow_2(y))}{\hat{\pi}(\flow_1(x))\hat{\pi}(\flow_2(y))}\right].
%\end{equation}
%While this approach and other similar works \citep{dahlke2025flowbased,duong2023diffeomorphic} excel in synthetic benchmarks, they introduce an inductive bias with $\family$ which can lead to poor performances in real world scenarios.

\begin{table}[h]
\centering
\scriptsize
\begin{tabular}{lll}
\toprule
\textbf{Estimator} & \textbf{Training Objective} & \textbf{MI Estimation Formula} \\
\midrule
\makecell[l]{DoE, DINE  \citep{mcallester2020limitations_MI,duong2023dine,ni2025neural_DoE}} & Log-likelihood &
$\displaystyle \MI(Z_1;Z_2) = \expect \log \PDF(Z_1 \mid Z_2) - \expect \PDF(Z_1) $ \\[2.5em]

\makecell[l]{DINE-Gaussian, $\normal$-MIENF \\ \citep{duong2023dine,butakov2024normflows}} & \makecell[l]{Log-likelihood \\ (marginal for DINE, \\ joint for MIENF)} &
$\displaystyle \MI(Z_1;Z_2) \approx \frac{1}{2} \expect \log \frac{\det{\Sigma_{f_1(Z_1)} \det \Sigma_{f_2(Z_2)}}}{\det \Sigma_{f_1(Z_1), f_2(Z_2)}} $ \\[2em]

\makecell[l]{MIENF, VCE  \citep{butakov2024normflows,chen2025vector_copula}} & \makecell[l]{Log-likelihood (MIENF), \\ flow matching (VCE)} &
$\displaystyle \MI(Z_1;Z_2) = - \expect \log C(f_1(Z_1),f_2(Z_2)) $ \\
\bottomrule
\end{tabular}
\caption{
    Non-diffusion-based Generative Mutual Information Estimators and Their Estimation Formulas.
    Here $f_1,f_2$ are diffeomorphisms, $\Sigma$ denote covariance matrices and $C$ is a copula density.
}
\label{tab:mi_non_diffusion_estimators}
\end{table}

\textbf{Diffusion-based generative estimators.} A third family of estimators casts \gls{MI} estimation as a comparison between two stochastic processes:
one whose initial distribution is the joint $\pi(x,y)$ and another whose initial distribution is the product of marginals $\pi(x)\pi(y)$.
Both processes share the same diffusion coefficient and reference drift, so their \gls{KL} divergence reduces~--- via Girsanov's theorem~--- to a time-integrated squared difference between their score functions (or drift functions), which can be learned efficiently by neural networks from samples alone.
Setting the initial distributions to the joint and factored marginals then yields a tractable estimator of $\MI(X;Y) = \KL{\pi(x,y)}{\pi(x)\pi(y)}$, without requiring explicit density evaluation.
As shown in \Cref{tab:mi_diffusion_estimators}, \citep{franzese2024minde} (MINDE) parametrize the processes using score functions
learned via score matching, yielding an approximate estimator;
\citep{kholkin2025InfoBridge} (InfoBridge) instead use bridges with drift functions
learned via conditional bridge matching, yielding an exact expression
whose error depends solely on the accuracy of the learned drifts.

\begin{table}[h]
    \centering
    {
        \scriptsize
        \begin{tabular}{lll}
        \toprule
        \textbf{Estimator} & \textbf{Training Objective} & \textbf{MI Estimation Formula} \\
        \midrule
        \makecell[l]{MINDE \citep{franzese2024minde}} & Score matching &
        $\displaystyle \MI(X;Y) \approx \int_0^T\expected_{x_t\sim q_t^A}\!\left[\frac{g(t)^2}{2}\|s^A(x_t,t)-s^B(x_t,t)\|^2\right] \dif t$ \\[10pt]
        
        \makecell[l]{InfoBridge \citep{kholkin2025InfoBridge}} & Bridge matching &
        $\displaystyle \MI(X;Y) = \frac{1}{2\epsilon}\int_0^T\expected_{q^A_{0,t}}\!\left[\|v_A(x_t,t,x_0)-v_B(x_t,t,x_0)\|^2\right] \dif t$ \\
        \bottomrule
        \end{tabular}
    }
    \caption{Diffusion-based Mutual Information Estimators and Their Estimation Formulas.
      In both cases, $Q^A$ and $Q^B$ are stochastic processes whose initial distributions are the joint $\pi(x,y)$ and the product of marginals $\pi(x)\pi(y)$, respectively.
      MINDE learns score functions $s^A, s^B$ via score matching;
      InfoBridge learns drift functions $v_A, v_B$ via conditional bridge matching.}
    \label{tab:mi_diffusion_estimators}
\end{table}

Unlike variational estimators, this family does not suffer from theoretical saturation barriers at high \gls{MI}~\citep{mcallester2020limitations_MI,song2020understanding_limitations}.
Unlike non-diffusion generative estimators, it requires no choice of latent family and therefore introduces no inductive bias that would hinder applications to complex, real-world distributions.
In practical terms, this unlocks applications in tasks typically harder than classification,
such as text-to-image alignment, self-supervised learning, and deep neural network analysis.

\subsection{Existing Benchmarks}

\textbf{Low-dimensional synthetic benchmarks.} \gls{MI} is notoriously hard to compute analytically, even for simple distributions. For this reason, the most basic and popular approach to benchmarking \gls{MI} estimators relies on multivariate Gaussian distributions, which admit a closed-form expression for \gls{MI}~\citep[Example 3.2]{polyanskiy2024information_theory}. However, in this particular case, \gls{MI} depends solely on the covariance matrix, making its estimation equivalent to correlation analysis, which is much easier in general. \citep{czyz2023beyond_normal} addressed this limitation by proposing an extensive benchmark of different distributions with known \gls{MI} and varying complexity. This set can be further extended with two similar in spirit continuous examples from~\citep{kraskov2004KSG}. To cover discrete-continuous cases, \citep{lee2024benchmark} provide four additional distributions with non-continuous and non-discrete marginals. Finally, \gls{MI}'s invariance to invertible mappings (see~\Cref{theorem:MI_invariance}) can also be leveraged to further complicate marginal distributions of these synthetic tests.
% \begin{equation}
%     \label{eq:MI_invariance}
%     \underbrace{\MI(X;Y)}_{\text{tractable}} = \MI(f(X);g(Y)), \qquad \text{$f,g$ satisfy~\citep[Theorem 3.2(d)]{polyanskiy2024information_theory}}%~\citep[Statement 1]{butakov2024lossy_compression}}
%     \vspace{-0.5em}
% \end{equation}
This approach, however, is of limited expressivity, as it also preserves copula and \gls{PMI}~\citep{czyz2025PMI}. In other words, not all joint distributions can be expressed as $\jproba{X}{Y} \circ (g_X^{-1} \times g_Y^{-1})$ for fixed $\jproba{X}{Y}$.

\textbf{High-dimensional synthetic benchmarks.} Any low-dimensional distribution can be extended to a high-dimensional counterpart by stacking independent random vectors. However, this approach is trivial and can not model intricate cases, such as data distributed on a manifold. Therefore, \citep{butakov2024lossy_compression} proposed mapping low-dimensional data to high-dimensional ambient space via handcrafted invertible, \gls{MI}-preserving transforms (see~\Cref{figure:synthetic_images_generation}).
To date, this approach was limited to the mappings from the original article, which transform Gaussian latents to images of (a) rectangles and (b) Gaussian blobs (see~\Cref{figure:synthetic_images_examples}).
%However, this methodology can be extended to any other data-generation function, provided it satisfies~\Cref{theorem:MI_invariance}.

\begin{figure}[ht!]
    \centering
    \includegraphics[height=4.0cm]{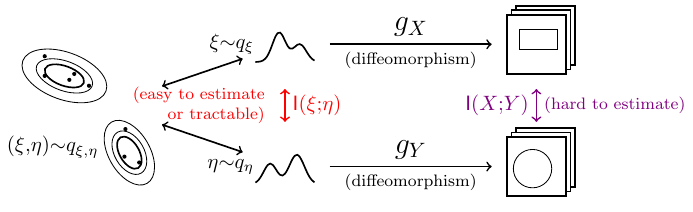}
    \caption{
        A general method for producing complex, high-dimensional datasets with tractable \gls{MI}.
        Here, $q$ is a latent distribution with tractable or easily estimable \gls{MI}.
        Smooth, injective mappings $g_X$ and $g_Y$ transform the marginals of $q$,
        thus achieving a complex distribution while preserving the \gls{MI}.
    }
    \label{figure:synthetic_images_generation}
\end{figure}

\begin{figure}[hb!]
    \centering
    \begin{subfigure}[b]{0.5\textwidth}
        \captionsetup{skip=0.5\baselineskip}
        \centering
        \frame{\includegraphics[width=0.22\textwidth]{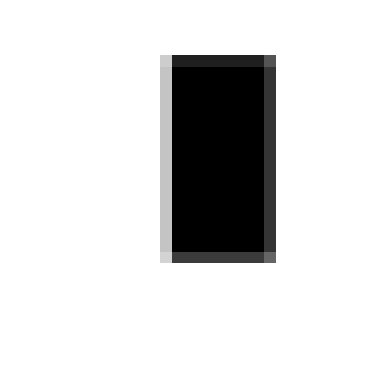}}
        \frame{\includegraphics[width=0.22\textwidth]{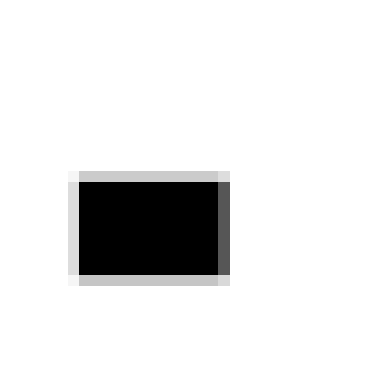}}
        \frame{\includegraphics[width=0.22\textwidth]{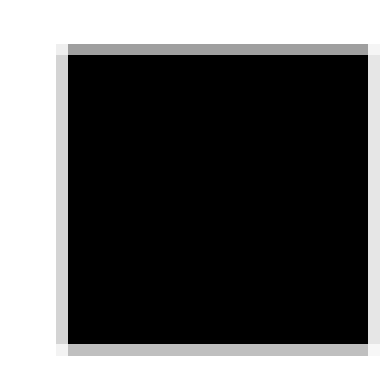}}
        \frame{\includegraphics[width=0.22\textwidth]{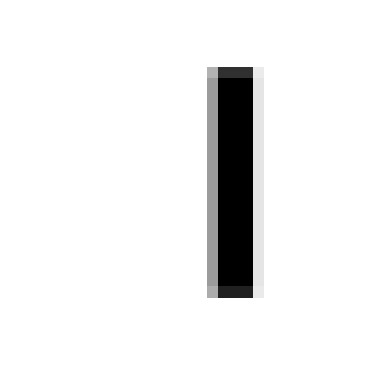}}
        \caption{Rectangles}
    \end{subfigure}%
    \begin{subfigure}[b]{0.5\textwidth}
        \captionsetup{skip=0.5\baselineskip}
        \centering
        \frame{\includegraphics[width=0.22\textwidth]{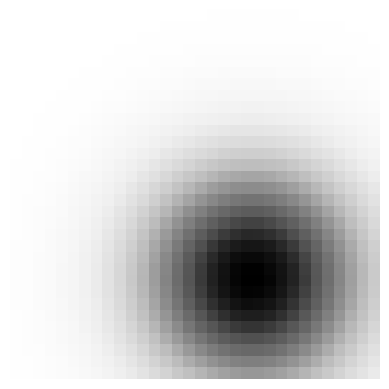}}
        \frame{\includegraphics[width=0.22\textwidth]{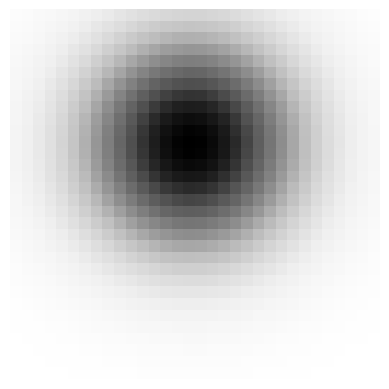}}
        \frame{\includegraphics[width=0.22\textwidth]{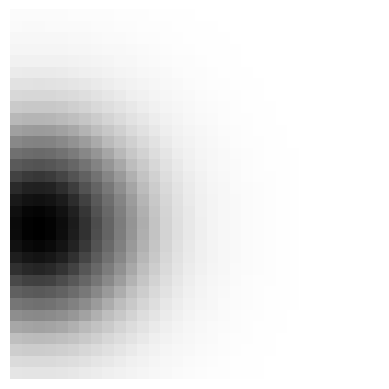}}
        \frame{\includegraphics[width=0.22\textwidth]{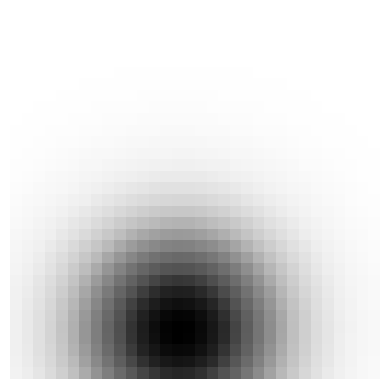}}
        \caption{Gaussian blobs}
    \end{subfigure}
    \caption{Examples of synthetic images from~\citep{butakov2024lossy_compression}. Note that images are high-dimensional, but admit latent structure, which is similar to real datasets}
    \label{figure:synthetic_images_examples}
\end{figure}

\textbf{Benchmarks based on real data.} \citep{lee2024benchmark,gowri2024LMI}, instead, focus on benchmarking \gls{MI} estimators on unstructured data, such as images and texts. Their approach relies on \textit{same class pairing} for constructing datasets with known \gls{MI} under the assumptions that the random vectors within the dataset only share \textit{class} information and that this class information is easily decodable. While this approach allows the usage of complex datasets to evaluate \gls{MI} estimators, it does so by ignoring the data generating process, limiting the difficulty of the benchmark to classification tasks.

\textbf{Consistency tests.} Other works \citep{song2020understanding_limitations} propose consistency tests on real world data not limited to class information, with the drawback that, while they consider the complexity of the data generating process, they cannot provide a ground truth value of \gls{MI}.
\section{Proposed Benchmarks}
\label{section:benchmarks}

From the background overview, we know that existing benchmarking approaches are diverse, complex,
and often described from different perspectives, which pits them against one another. More importantly, their diversity does not translate into comprehensiveness,
since contemporary tests are often proposed to address universal obstacles
(e.g., long-tailed distributions) and do not provide a systematic coverage of challenges that are specific to \gls{MI} estimation.
We therefore establish our \emph{first major contribution} by describing all existing benchmarks from a unified perspective.
Specifically, we posit that all tests essentially evaluate how well estimators capture the \gls{DPI}:
\begin{align}
    &V \to \boxed{X} \to A \longrightarrow B \to \boxed{Y} \to W\,, \qquad \MI(V;W) \leq \MI(X;Y) \leq \MI(A;B)
\end{align}
\begin{enumerate}
    \item \emph{Synthetic low- and high-dimensional tests}
    employ $V = A$ and $W = B$ with tractable (or easy to compute) $\MI(V;W)$ and construct $X = g_X(V)$, $Y = g_Y(W)$,
    which satisfy~\Cref{theorem:MI_invariance} by design.

    \item \emph{Real data class-pairing benchmarks} leverage discrete $V$ and $W$ with tractable $\MI(V;W)$
    and Markov kernels $V \to X$, $W \to Y$ which, for a given class label, sample a random data point from that class.
    Under the perfect classification assumption, $V$ and $W$ are deterministic functions of $X$ and $Y$, respectively;
    thus, one can select $A = V$ and $B = W$ and get $\MI(V;W) = \MI(X;Y)$.

    \item \emph{Consistency tests} define $X$ and $Y$ as increasingly corrupted versions of $V$, discarding $A$, $B$, and $W$.
    An estimator is then validated by checking that it satisfies $\MI(V;X) \geq \MI(V;Y)$.
\end{enumerate}

Equipped with this knowledge, we propose a \emph{comprehensive benchmark suite} for evaluating \gls{MI} estimators,
spanning from purely synthetic settings to complex scenarios involving real-world data.
Our benchmark is divided into two complementary parts which incorporate all the previous approaches but also include new tests:

\begin{itemize}
    \item \emph{Copula first, marginals second}:
    all existing synthetic benchmarks with tractable \gls{MI} can be viewed as push-forwards of some copulas with tractable entropy
    under Cartesian products of \gls{MI}-preserving transforms:
    \begin{equation}\label{eq:copula_first}
        \jproba{X}{Y} = \underbrace{\jcopula{\xi}{\eta}}_{\mathclap{\text{tractable entropy}}} \quad \circ \quad \underbrace{(g_X^{-1} \times g_Y^{-1})}_{\mathclap{\text{preserves \gls{MI}}}}
    \end{equation}
    We therefore gather and organize all members of this family that are unique up to the said transforms
    and propose our own examples which highlight specific challenges of \gls{MI} estimation.
    This relatively simple benchmark enables extensive evaluation of \gls{MI} estimators under controlled conditions, allowing us to vary the ground truth \gls{MI}, dimensionality, and sample size.

    \item \emph{Marginals first, copula second}:
    %we show that all existing benchmarks based on real data can be viewed as coupling of marginals according to
    we extend the class-pairing approach to a general data-coupling method
    \[
        \jproba{X}{Y} = \underbrace{\jcopula{\xi}{\eta}}_{\mathclap{\text{tractable entropy}}} \quad \cdot \quad \underbrace{\proba_{X \mid \xi} \otimes \proba_{Y \mid \eta}}_{\text{real data}},
    \]
    where $\proba_{\blankarg \mid \blankarg}$ are Markov kernels launching some $\alpha$ to an element of parametric distribution family $\{\proba_\alpha\}_{\alpha \in \mathcal{A}}$ with disjoint supports.
    %we take real-world marginal distributions \( \pi_1 \) and \( \pi_2 \), and combine them using a synthetic copula:
    %\[
    %    \pi = C(\pi_1, \pi_2),
    %\]
    %where \( C \) denotes a copula function constructed using auxiliary random variables and theoretical tools such as the data processing inequality.
    This setup allows us to evaluate \gls{MI} estimators on realistic marginals while preserving control over the underlying dependence structure.
\end{itemize}

\subsection{Synthetic distributions}
\label{section:distributions}

This section defines the synthetic distributions that form the basis of our benchmark. 
Other distributions are derived from these using the mechanism described in~\Cref{section:benchmarks}. 
\Cref{table:distributions} lists the names, aliases, brief descriptions of the sampling procedures, and formulas for the ground-truth \gls{MI}.

\begin{table}[h]
    \centering
    {
        \setcellgapes{1em}
        \scriptsize
        \begin{tabular}{llll}
        \toprule
        \textbf{Name} & \textbf{Alias} & \textbf{Definition} & \textbf{Mutual Information} \\
        \midrule
        Correlated Normal & Corr. $\normal$
        & $X,Y \sim \normal\left(0, \begin{bmatrix} \identity & \Sigma_{X,Y} \\ \Sigma_{X,Y}^T & \identity \end{bmatrix}\right) $
        & $\MI(X;Y) = -\frac{1}{2} \log \det \Sigma_{X,Y}$\\
        \addlinespace[1.0em]
        Correlated Uniform & Corr. $\uniform$
        & \makecell[l]{$X,Y = \CDF(\xi), \CDF(\eta)$; $\xi,\eta \sim \text{corr. $\normal$}$, \\[0.5em] $\CDF$ is a Gaussian CDF.}
        & $\MI(X;Y) = -\frac{1}{2} \log \det \Sigma_{X,Y}$ \\
        \addlinespace[1.0em]
        Correlated Student's $t$ & Corr. St.
        & \makecell[l]{$X = \xi \cdot \sqrt{k/Z}$, $Y = \eta \cdot \sqrt{k/Z}$ \\[0.5em]
          $\xi,\eta \sim \text{corr. $\normal$}$, $Z \sim \chi^2(k)$}
        & \makecell[l]{$\MI(X;Y) = -\frac{1}{2} \log \det \Sigma_{X,Y} + c(k,d_X,d_Y)$ \\[0.5em]
          $\begin{aligned}
            c(k,n,m) &= f(k) + f(k+n+m) \\ &- f(k+n) - f(k+m) \\
            f(x) &= \log \Gamma\left(\frac{x}{2}\right) - \frac{x}{2} \psi \left( \frac{x}{2} \right)
            \end{aligned}$} \\
        \addlinespace[1.0em]
        Smoothed Uniform & Sm. $\uniform$
        & $\begin{aligned}X &\sim \uniform[0;1], \\ Z &\sim \uniform[-\varepsilon;\varepsilon], \end{aligned} \; \; Y = X + Z$
        & $\MI(X; Y) = \begin{dcases} \varepsilon - \log (2 \varepsilon), &\; \varepsilon < 1/2 \\ (4 \varepsilon)^{-1}, &\; \varepsilon \geq 1/2 \end{dcases}$ \\
        \addlinespace[1.0em]
        Log Gamma Exponential & LGE
        & $e^X, e^Y \sim \PDF(t,s) = \frac{1}{\Gamma(\theta)} t^\theta e^{-t-ts}$
        & $\MI(X;Y) = \psi(\theta + 1) - \log \theta$ \\
        \addlinespace[1.0em]
        Rare Event Channel & REC
        & \makecell[l]{
          $S \sim \bernoulli(p)$. \\
          If $S = 0$, $X' = Y' = \const$. \\
          If $S = 1$, $X' = Y' = Z$. \\
          $X,Y$ are dequantized $X',Y'$.
          }
        & $\MI(X;Y) = h(p) + (1 - p) \sent(Z)$ \\
        \addlinespace[1.0em]
        Uniformly Quantized & UQ
        & $X \sim \normal(0,1)$, $Y = \operatorname{quant}(X)$
        & $\MI(X;Y) = \sent (Y)$ \\
        \addlinespace[1.0em]
        \makecell[l]{Smoothed Discrete \\ Uniform} & Sm. d. $\uniform$
        & $\begin{aligned}X &\sim \uniform\{0, \ldots, K-1\}, \\ Z &\sim \uniform[0;a], \\ Y &= X + Z \end{aligned}$
        & Please, refer to~\Cref{lemma:smoothed-discrete-uniform}\\
        \bottomrule
        \end{tabular}
    }
    \caption{Base distributions definitions and corresponding ground-truth values of \gls{MI}.}
    \label{table:distributions}
\end{table}

A more detailed description is provided below:
\begin{itemize}
    \item
        \emph{Correlated Normal}: A standard benchmark based on the multivariate Gaussian distribution. 
        In this case, the copula and \gls{MI} are completely determined by the covariance structure. 
        For normalization purposes, we require $X$ and $Y$ to have unit covariance matrices, 
        thereby enforcing a block structure onto the joint covariance matrix.
    \item
        \emph{Correlated Uniform}: A copula-based variant of the Correlated Normal test, 
        where each individual component is mapped to a uniform distribution via the standard Gaussian \gls{CDF}.
        While this distribution shares the exact same copula as the Correlated Normal test,
        we list it separately for consistency reasons.
    \item
        \emph{Correlated Student's $t$}: This distribution is derived from the Gaussian case 
        by dividing the variables by $\sqrt{Z/k}$, where $Z$ is an independent random variable 
        following a chi-square distribution with $k$ degrees of freedom. 
        For small $k$, this test is challenging due to heavy tails. 
        Furthermore, it cannot be constructed for \gls{MI} values smaller than a certain correction term, 
        as the division by $\sqrt{Z/k}$ unavoidably introduces a baseline level of dependence.
    \item
        \emph{Smoothed Uniform}: Constructed by the summation of two independent, 
        uniformly distributed continuous random variables.
        Its copula differs significantly from the Gaussian copula. 
        %We apply linear normalization to constrain the support to the interval $[-1, 1]$.
    \item
        \emph{Smoothed Discrete Uniform}: A mixed counterpart to Smoothed Uniform,
        where $X$ follows a discrete uniform distribution supported on $\{0, \ldots, K-1\}$, and $Y = X + \uniform[0;a]$.
    \item
        \emph{Log-Gamma Exponential}: This distribution features a non-Gaussian copula and heavy tails. 
        It is constructed by summing two independent continuous random variables 
        following the log-gamma distribution with distinct parameters.
    \item
        \emph{Rare Event Channel}: Constructed as a continuous extension of an errorless 
        discrete channel with one dominating symbol (the ``failure'' symbol), where 
        the probabilities of discrete outcomes correspond to segments within $[0, 1]$. 
        This dominating symbol artificially increases the variance of $\text{PMI}(X;Y)$.
    \item
        \emph{Uniformly Quantized}: This setup employs a continuous $X$ (defaulting to $\mathcal{N}(0,1)$) 
        which is quantized to produce a discrete $Y$. For a specified ground-truth \gls{MI}, 
        the minimum number of quantization levels is used to ensure the distribution of $Y$ 
        remains as uniform as possible.
\end{itemize}
For some tests, we also apply linear normalization to constrain the variance or support.
For a comprehensive discussion of these distributions, including the derivation of ground-truth \gls{MI}, 
probability densities, and sampling procedures, please refer to \Cref{appendix:distributions}.

\subsection{Synthetic copula}\label{synthetic_copula}
In the previous section we designed a synthetic benchmark starting from synthetic distribution with analytical derivation of \gls{MI}. In this section, instead, we only use synthetic copula while using marginal distributions from real data. Our approach relies on injective mappings: consider two random variables $Z_1$, $Z_2$, for which we already know the ground truth \gls{MI}. With the help of injective mappings $f_1,f_2$, we can map $Z_1$ and $Z_2$ to another pair of random variables $W_1,W_2$. This transformation constitutes a Markov chain
\begin{equation}\label{eq:mc}
    Z_1\to Z_2\overset{f_2}{\to} W_2 \overset{f_2^{-1}}{\to} Z_2 \to Z_1 \overset{f_1}{\to} W_1 \overset{f_1^{-1}}{\to} Z_1.
\end{equation}

Where $f_1^{-1}$ and $f_2^{-1}$ are left inverses.
Using the data processing inequality, we can prove that $I(Z_1,Z_2)\geq I(W_1,W_2) \geq I(Z_1,Z_2)$, which means that $I(Z_1,Z_2)=I(W_1,W_2)$. In general, $f_1$ and $f_2$ need not be deterministic: they only need to be invertible Markov kernels. 

\textbf{Same class sampling.}  As a motivating example, consider the case $Z_1 = Z_2 = Z$, where $Z$ is a discrete random variable with finite support of size $k$. Let $\{ \pi_i \}_{1 \leq i \leq k}$ be a collection of distributions with disjoint supports, each corresponding to a distinct class label $i$. We define a Markov kernel $f$ such that $f(Z = i) \sim \pi_i$, i.e., class label $i$ is mapped to a random sample drawn from distribution $\pi_i$.

Applying $f$ independently to two instances of $Z$, we obtain random variables $W_1 = f(Z_1)$ and $W_2 = f(Z_2)$. Substituting $f_1 = f_2 = f$ in the Markov chain formulation in~\Cref{eq:mc}, the data processing inequality yields:
\[
    \MI(W_1; W_2) = \sent(Z).
\]

This construction extends the method introduced by~\citep{lee2024benchmark}, where a uniformly distributed latent variable $Z$ is used to sample images or text from distinct classes.

\textbf{Regression.}  We extend the previous approach to the continuous setting. Let $(Z_1,Z_2)$ be a continuous pair of random variable with known \gls{MI} $\MI(Z_1;Z_2)$. Instead of a finite collection of distributions, we define a \emph{family of distributions} $\{ \pi_\theta \}_{\theta \in \Theta}$, parameterized by $\theta \in \Theta\). We use this family to construct a Markov kernel $f$ such that $f(Z = \theta) \sim \pi_\theta$. As in the discrete case, we assume the supports of the $\pi_\theta$ distributions are disjoint to ensure information is preserved.

By applying this kernel independently to $Z_1$ and $Z_2$,
we obtain complex random variables $W_1 = f(Z_1)$ and $W_2 = f(Z_2)$, for which the \gls{MI} satisfies:
\[
    \MI(W_1; W_2) = \MI(Z_1;Z_2).
\]

A concrete example of this construction is as follows: let $Z$ represent the \emph{brightness level} of an image, modeled as a continuous random variable. For each brightness level $\theta$, define $\pi_\theta$ as a distribution over natural images with brightness $\theta$. Then, $W_1$ and $W_2$ correspond to two independently sampled images from the brightness levels $Z_1$ and $Z_2$,
resulting in a dataset where complex dependencies are present, but the \gls{MI} remains analytically tractable.

\paragraph*{Continuous Normalizing Flows} \glspl{CNF} \citep{chen2018neural} offer a structured methodology to construct invertible mappings between distributions: starting from a probability path \(p_t:[0,T]\times\mathbb{R}^d\to\mathbb{R}_{>0}\) and a time-dependent vector field \(\vfield:[0,T]\times\mathbb{R}^d\to\mathbb{R}^d\) to construct a time-dependent diffeomorphic map \(\phi:[0,T]\times\mathbb{R}^d\to\mathbb{R}^d\) called flow through a \gls{ODE}:
\begin{align*}
    \frac{\partial}{\partial t}\phi_t(x)&=\vfield_t(\phi_t(x)) \\
    \phi_0(x) &= x_0
\end{align*}
\glspl{CNF} are parametric models of $\phi$, which can be used to reshape $p_0$ into $p_1$ and, since they form invertible mappings, we can use them to construct \(g_X,g_Y\) in \cref{eq:copula_first}. This allows us to use powerful generative modeling paradigms, such as flow matching \citep{lipman2022flow}, to construct realistic datasets with known ground truth \gls{MI} and non-trivial mappings.
\section{Experiments}
\label{section:experiments}

In this section we implement the proposed benchmarks, measuring the performance of three different classes of \gls{MI} estimators: generative estimators (MINDE-C and MINDE-J), variational estimators (MINE-DV, MINE-NWJ and InfoNCE) and k-NN-based estimators (WKL and KSG).
For readability, verbose tables are provided in \Cref{app:num_synth}.

\begin{figure}[!ht]
    \centering
    \includegraphics[width=0.75\linewidth]{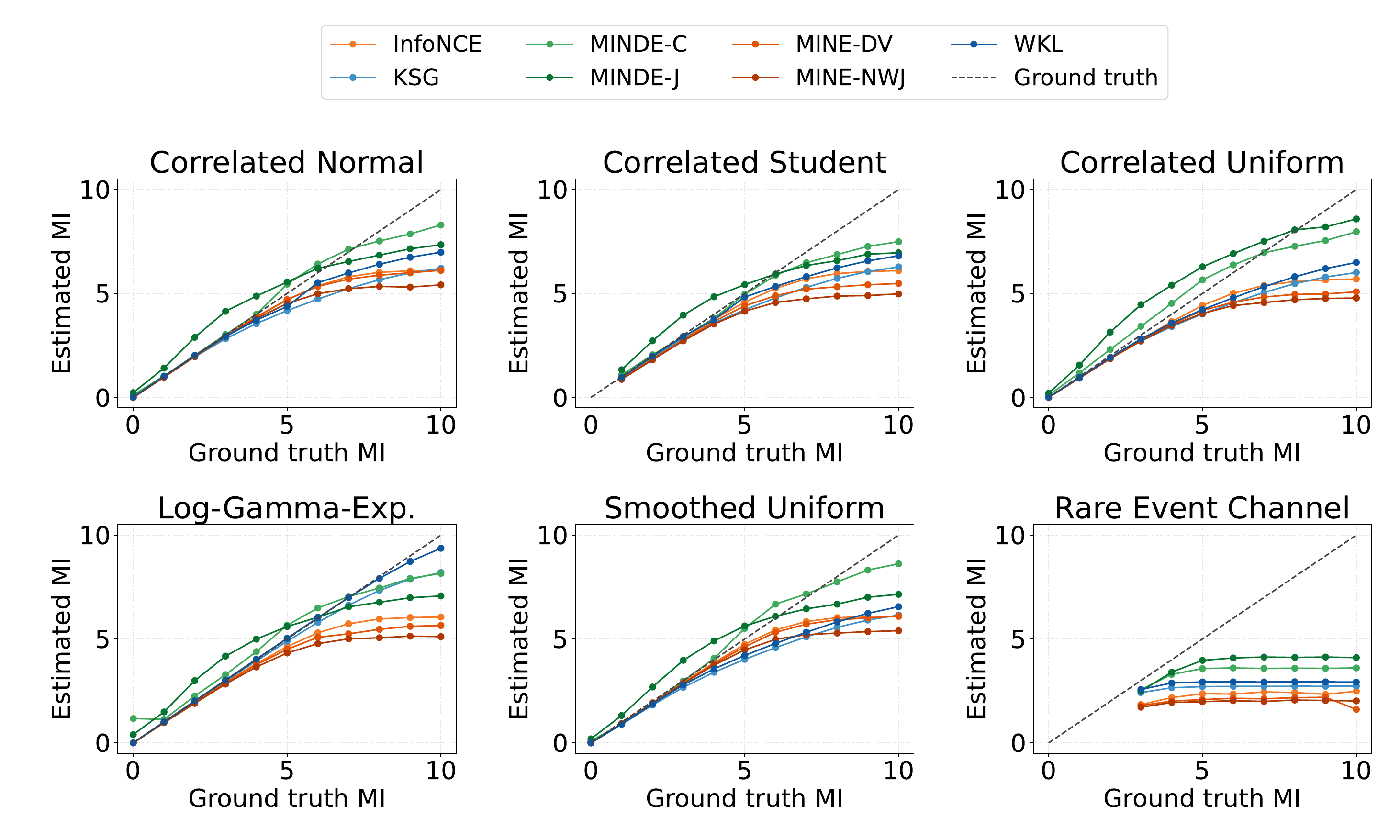}
    \caption{MI Estimation in Continuous Synthetic Experiment (All estimators)}
    \label{fig:cont_synth}
\end{figure}

\subsection{Synthetic Distributions}
\label{subsection:synthetic}
\label{sec:synthetic}

We evaluate all estimators on three classes of synthetic benchmarks: continuous (\Cref{table:distributions}) and discrete distributions with closed-form \gls{MI}, as well as synthetic distributions generated via \glspl{CNF}. 
To construct the latter, we sample a pair of bivariate correlated normal random variables, $A$ and $B$, with known \gls{MI}. We then train two distinct flow matching models to learn the mappings $g_X(A)=X$ and $g_Y(B)=Y$. This procedure is repeated using the \textit{moons}, \textit{circles}, and \textit{SCurve} distributions from the \textit{scikit-learn} library \citep{scikit-learn}. For each distribution, we generate datasets spanning 11 linearly spaced \gls{MI} values between 0 and 10, across dimensions $d \in \{1, 2, 3\}$. Higher-dimensional random variables are constructed by stacking independent and identically distributed univariate variables. To ensure robustness, we average the results of 10 independent runs across different random seeds for each setting. In order to compare estimator families, we aggregate the results by taking the best estimate within each class for every distribution and ground-truth \gls{MI} pair, visualizing these group-level performances with line plots consistent with our individual estimator results. 
Finally, we conduct two different ablations on sample complexity and dimensionality; in the former we fix the dimensionality \(d_X=d_Y=4\) and change the training set size, in the latter we do the opposite: we fix the training set size to \(10^5\) and change the dimensionality. The results of these evaluations reveal distinct behavior patterns across estimator families, particularly regarding dimensionality, data complexity, and sample efficiency.

\begin{figure}[!b]
    \centering
    \begin{minipage}{0.48\textwidth}
        \centering
        \includegraphics[width=\linewidth]{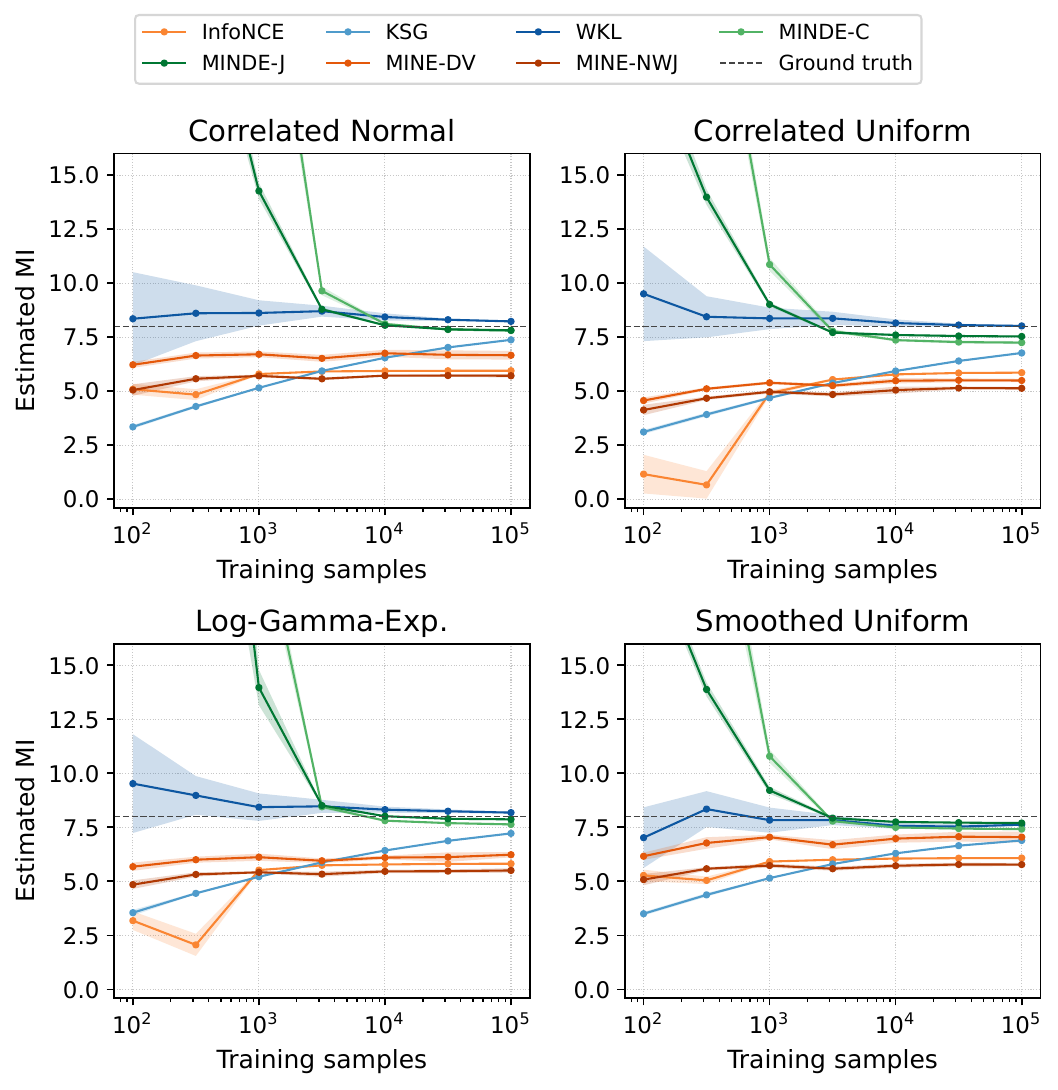}
        \caption{Sample complexity experiments. Shaded area reports the standard deviation.}
        \label{fig:cont_synth_d4_all}
    \end{minipage}\hfill
    \begin{minipage}{0.48\textwidth}
        \centering
        \includegraphics[width=\linewidth]{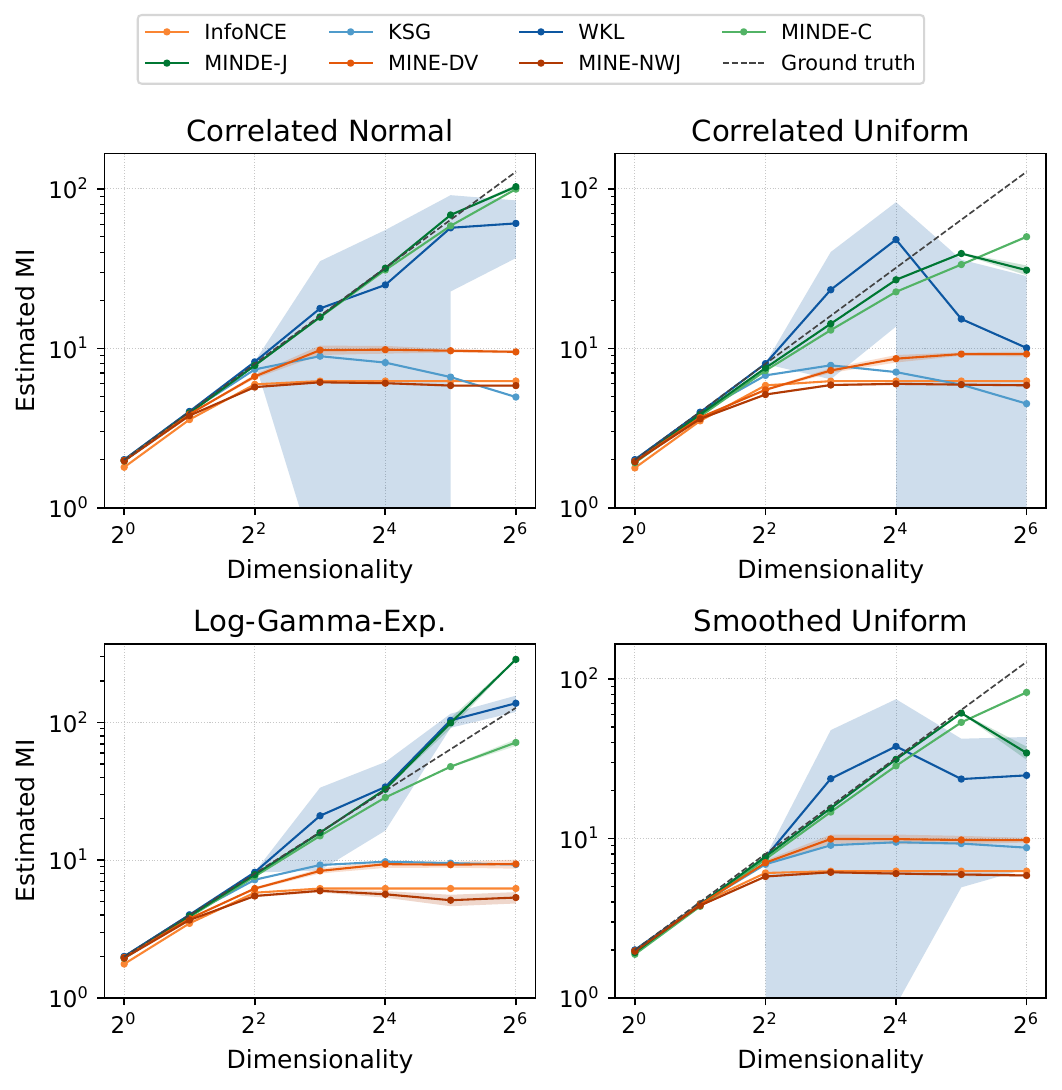}
        \caption{Dimensionality scaling experiments. Shaded area reports the standard deviation.}
        \label{fig:cont_synth_dim_all}
    \end{minipage}
\end{figure}

\begin{figure}[!b]
    \centering
    \includegraphics[width=0.75\linewidth]{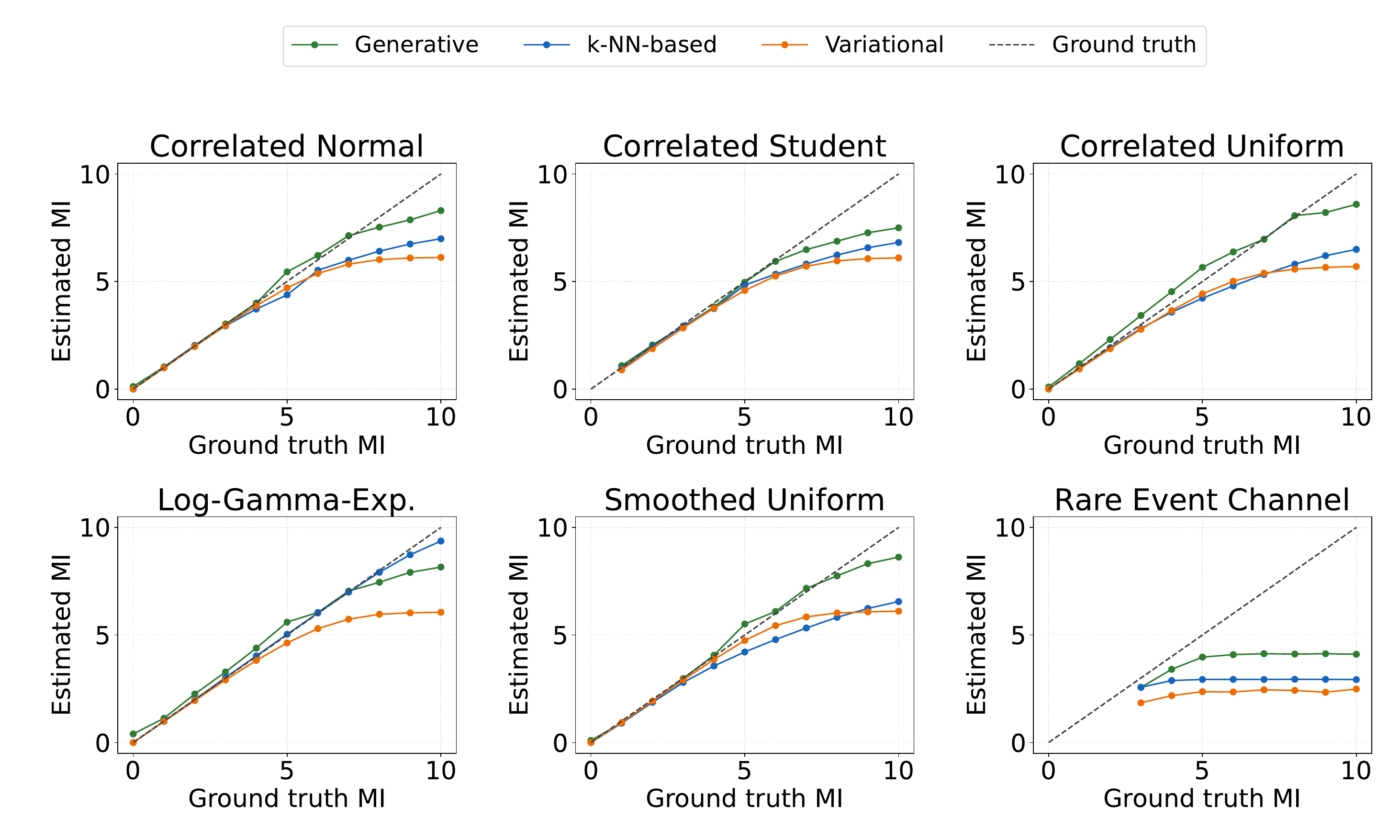}
    \caption{Comparison of estimator families on continuous synthetic distributions. To clearly contrast k-NN, variational, and generative approaches, we aggregate the results by estimator class. For any given point on the x-axis (representing a specific distribution and ground-truth \gls{MI} pair), the plotted value shows the most accurate estimate achieved by any model within that specific family.}
    \label{fig:cont_synth_by_group}
\end{figure}

\textbf{$k$-NN-based estimators outperform variational estimators in low dimensions.} Overall, our low-dimensional experiments (\Cref{fig:cont_synth_by_group}, \Cref{fig:disc_group}, and \Cref{fig:flow_by_group}) show that variational estimators lag behind k-NN-based estimators in terms of accuracy, especially at high \gls{MI}. This comes as a surprise: k-NN estimators are decisively less compute-intensive in this regime, requiring less time and memory, which leaves little practical reason to prefer variational estimators. Note that, while \cref{fig:cont_synth_dim_all} shows that WKL does better than variational estimators and KSG in higher dimensions, it does so at the price of high variance within its estimates.

\textbf{Non-linear transformations via \glspl{CNF} make estimation much harder.} Compared to the correlated normal benchmark in \Cref{fig:cont_synth}, \Cref{fig:flow} shows that all estimators struggle significantly more when the data is transformed via an invertible mapping. This allows constructing highly challenging \gls{MI} benchmarks, which until now have largely been limited to simple toy distributions.%\newline%
\begin{wraptable}[15]{r}{0.35\textwidth} % 'r' for right side, 'l' for left. Adjust width as needed.
    % Input your table contents here
    \resizebox{\linewidth}{!}{
\begin{tabular}{llcccccc}
\toprule
 &  & \multicolumn{6}{c}{$N$ samples} \\
Estimator & Hidden dim & 100 & 500 & 1000 & 2000 & 5000 & 10000 \\
\midrule
\multirow[l]{4}{*}{MINE-DV} & 16 & $ 2.50 $ & $ 1.50 $ & $ 1.63 $ & $ 1.51 $ & $ 1.45 $ & $ 1.51 $ \\
 & 32 & $ 1.99 $ & $ 1.63 $ & $ 1.24 $ & $ 1.40 $ & $ 1.47 $ & $ 1.37 $ \\
 & 64 & $ 1.92 $ & $ 1.67 $ & $ 1.41 $ & $ 1.42 $ & $ 1.45 $ & $ 1.36 $ \\
 & 128 & $ 1.90 $ & $ 1.54 $ & $ 1.57 $ & $ 1.41 $ & $ 1.56 $ & $ 1.50 $ \\
\cmidrule{1-8}
\multirow[l]{4}{*}{MINE-NWJ} & 16 & $ 2.85 $ & $ 1.76 $ & $ 1.76 $ & $ 1.72 $ & $ 1.59 $ & $ 1.67 $ \\
 & 32 & $ 2.80 $ & $ 1.79 $ & $ 1.46 $ & $ 1.57 $ & $ 1.59 $ & $ 1.54 $ \\
 & 64 & $ 2.68 $ & $ 1.86 $ & $ 1.58 $ & $ 1.56 $ & $ 1.60 $ & $ 1.54 $ \\
 & 128 & $ 2.73 $ & $ 1.83 $ & $ 1.71 $ & $ 1.67 $ & $ 1.69 $ & $ 1.67 $ \\
\cmidrule{1-8}
\multirow[l]{4}{*}{InfoNCE} & 16 & $ 2.70 $ & $ 1.47 $ & $ 1.17 $ & $ 1.14 $ & $ 1.14 $ & $ 1.14 $ \\
 & 32 & $ 2.90 $ & $ 1.55 $ & $ 1.16 $ & $ 1.11 $ & $ 1.10 $ & $ 1.10 $ \\
 & 64 & $ 3.07 $ & $ 1.65 $ & $ 1.20 $ & $ 1.09 $ & $ 1.08 $ & $ 1.07 $ \\
 & 128 & $ 3.16 $ & $ 1.82 $ & $ 1.29 $ & $ 1.11 $ & $ 1.07 $ & $ 1.07 $ \\
\bottomrule
\end{tabular}
}
\caption{
    Mutual information estimation error,
    varying sampling size and estimators' network hidden dimensionality
    ($d_X = d_Y  = 1$).
}
\label{tab:continuous:double_descend}

\end{wraptable}
\textbf{Variational estimators are robust to sample size and network dimension.} We ran an additional ablation on the continuous distribution benchmark to measure the sensitivity of variational \gls{MI} estimators to training sample size and parameter count. The average estimation error across all distributions (\Cref{tab:continuous:double_descend}) demonstrates that variational estimators maintain their accuracy even with as few as 500 samples if the \gls{MI} is low. However, \cref{fig:cont_synth_d4_all,fig:cont_synth_dim_all} show that they are, as expected from theoretical results, incapable of estimating high values of \gls{MI}. On the other hand, neural network size does not appear to affect the quality of the \gls{MI} estimates. The results also exhibit double descent-like patterns (estimation accuracy may sometimes improve when the sampling size is reduced).

\textbf{Generative estimators scale better.} The sample complexity and dimensionality benchmarks in \cref{fig:cont_synth_d4_all,fig:cont_synth_dim_all} truly show that generative estimators escape the curse of \citep{mcallester2020limitations_MI}. On the other hand, the same results show that generative estimators are also the worst in terms of sample efficiency, where they tend to overestimate \gls{MI}. Finally, the same results show that the difficulty of estimating \gls{MI} is distribution-specific, with the Correlated Normal being the easiest benchmark. 

% Choosing $\varepsilon^2=n^{-1/2}$ as in \citep{zhu2023sample} gives $B_n\asymp n^{-1/2}$ and hence $|\hat I-I_{[t_0,T]}|=O\big(n^{-1/2}+\sqrt{\MI(X;Y)}\,n^{-1/4}\big)$. 

% \input{tables/continuous/double-descend}

\textbf{$k$-NN-based methods do not require large $k$.}
Our additional ablation on the estimator hyperparameter $k$ shows that these methods achieve the best performance when $k$ is low
($k=1$ for KSG and $k=5$ for WKL).
Complete results are presented in~\Cref{tab:KSG_WKL_kNN_ablation}.

\begin{table}[ht!]
    \centering
    \scriptsize
    \begin{subtable}[t]{0.49\textwidth}
        \centering
        %\begin{table}
%\center
\begin{tabular}{lcccccc}
\toprule
 & \multicolumn{6}{c}{$k_{NN}$} \\
Distribution & 1 & 2 & 3 & 5 & 10 & 20 \\
\midrule
Corr $\normal$  & $ 1.80 $ & $ 1.99 $ & $ 2.12 $ & $ 2.30 $ & $ 2.56 $ & $ 2.85 $ \\
Corr St.        & $ 1.95 $ & $ 2.17 $ & $ 2.31 $ & $ 2.51 $ & $ 2.80 $ & $ 3.13 $ \\
Corr $\uniform$ & $ 1.91 $ & $ 2.10 $ & $ 2.23 $ & $ 2.39 $ & $ 2.64 $ & $ 2.92 $ \\
LGE             & $ 1.09 $ & $ 1.26 $ & $ 1.39 $ & $ 1.57 $ & $ 1.86 $ & $ 2.20 $ \\
Sm. $\uniform$  & $ 1.90 $ & $ 2.09 $ & $ 2.22 $ & $ 2.38 $ & $ 2.63 $ & $ 2.91 $ \\
UQ              & $ 2.23 $ & $ 2.59 $ & $ 2.79 $ & $ 3.02 $ & $ 3.33 $ & $ 3.63 $ \\
\bottomrule
\end{tabular}
%\caption{Mean absolute error for the KSG estimator across different distributions, averaged over varying dimensionalities and ground-truth \gls{MI} values.}
%\label{tab:KSG_kNN_ablation}
%\end{table}
    \end{subtable}
    \begin{subtable}[t]{0.49\textwidth}
        \centering
        %\begin{table}
%\center
\begin{tabular}{lccccc}
\toprule
 & \multicolumn{5}{c}{$k_{NN}$} \\
Distribution & 1 & 5 & 9 & 13 & 17 \\
\midrule
Corr $\normal$  & $ 1.69 $ & $ 1.48 $ & $ 1.51 $ & $ 1.55 $ & $ 1.56 $ \\
Corr St.        & $ 1.70 $ & $ 1.59 $ & $ 1.61 $ & $ 1.66 $ & $ 1.67 $ \\
Corr $\uniform$ & $ 2.22 $ & $ 1.59 $ & $ 1.61 $ & $ 1.66 $ & $ 1.67 $ \\
LGE             & $ 2.53 $ & $ 0.76 $ & $ 0.75 $ & $ 0.74 $ & $ 0.73 $ \\
Sm. $\uniform$  & $ 2.18 $ & $ 1.63 $ & $ 1.66 $ & $ 1.71 $ & $ 1.72 $ \\
UQ              & $ 4.84 $ & $ 4.94 $ & $ 4.95 $ & $ 4.97 $ & $ 4.97 $ \\
\bottomrule
\end{tabular}
%\caption{Mean absolute error for the WKL estimator across different distributions, averaged over varying dimensionalities and ground-truth \gls{MI} values.}
%\label{tab:WKL_kNN_ablation}
%\end{table}
    \end{subtable}
    \caption{Mean absolute error for the KSG (left) and WKL (right) estimators across different distributions, averaged over varying dimensionalities and ground-truth \gls{MI} values.}
    \label{tab:KSG_WKL_kNN_ablation}
\end{table}

\begin{figure}[!ht]
    \centering
    \begin{subfigure}[b]{0.7\linewidth}
        \centering
        \includegraphics[width=\linewidth]{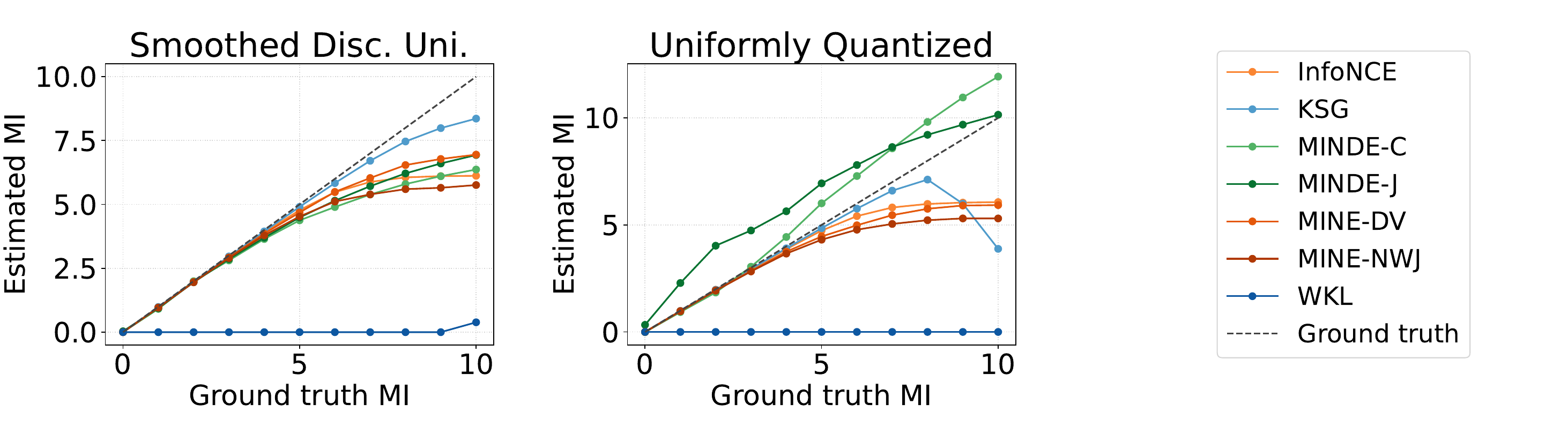}
        \caption{All estimators}
        \label{fig:disc_est}
    \end{subfigure}
    \vspace{0.5cm} % Adjust this value to increase/decrease vertical gap
    \begin{subfigure}[b]{0.7\linewidth}
        \centering
        \includegraphics[width=\linewidth]{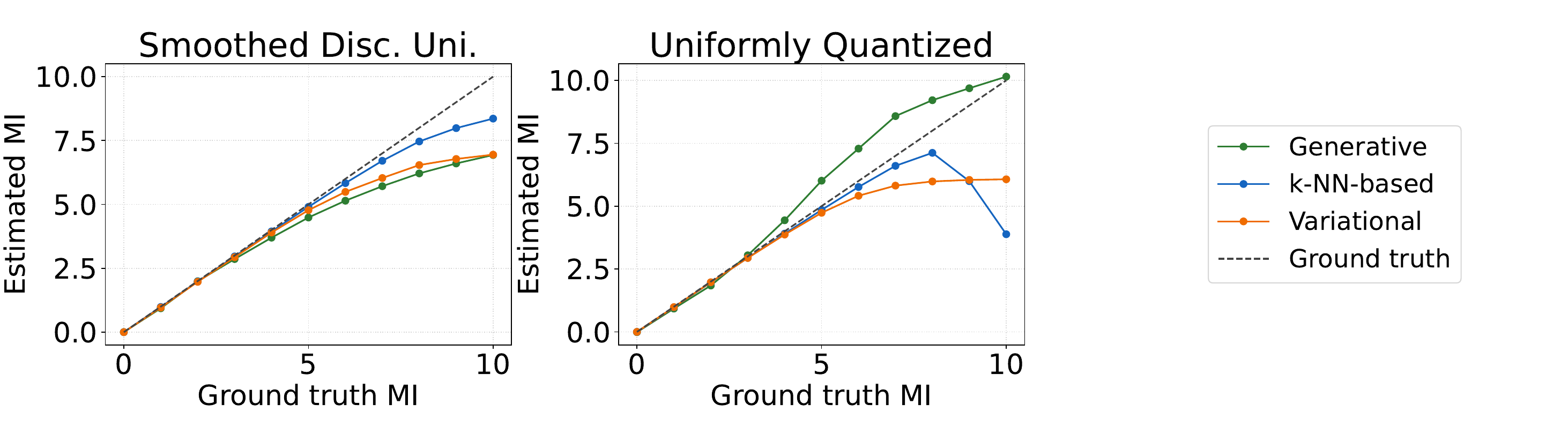}
        \caption{Estimator families}
        \label{fig:disc_group}
    \end{subfigure}
    \caption{MI Estimation for discrete experiments. \textbf{(a)} Individual performance of all estimators. \textbf{(b)} Comparison of estimator families on discrete synthetic distributions. To clearly contrast k-NN, variational, and generative approaches, we aggregate the results by estimator class. For any given point on the x-axis (representing a specific distribution and ground-truth \gls{MI} pair), the plotted value shows the most accurate estimate achieved by any model within that specific family.}
    \label{fig:disc_combined}
\end{figure}

\begin{figure}[!ht]
    \centering
    \begin{subfigure}[b]{0.8\linewidth}
        \centering
        \includegraphics[width=\linewidth]{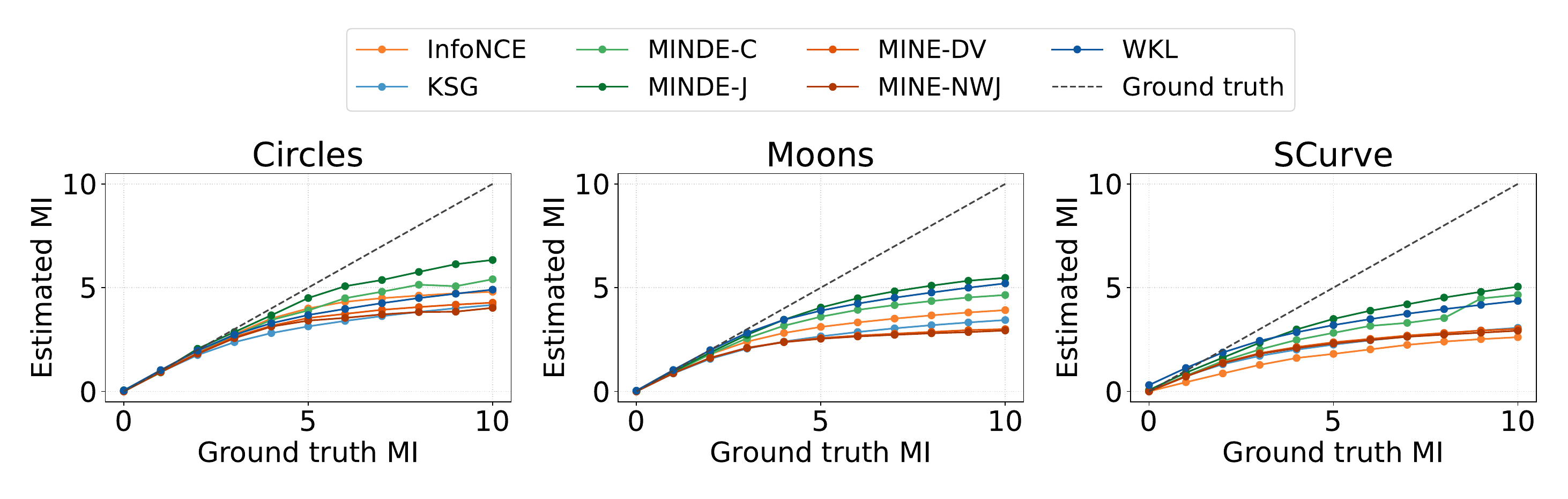}
        \caption{All estimators}
        \label{fig:flow}
    \end{subfigure}
    
    \vspace{0.1cm} % Adds vertical space between the figures; adjust as needed
    \begin{subfigure}[b]{0.8\linewidth}
        \centering
        \includegraphics[width=\linewidth]{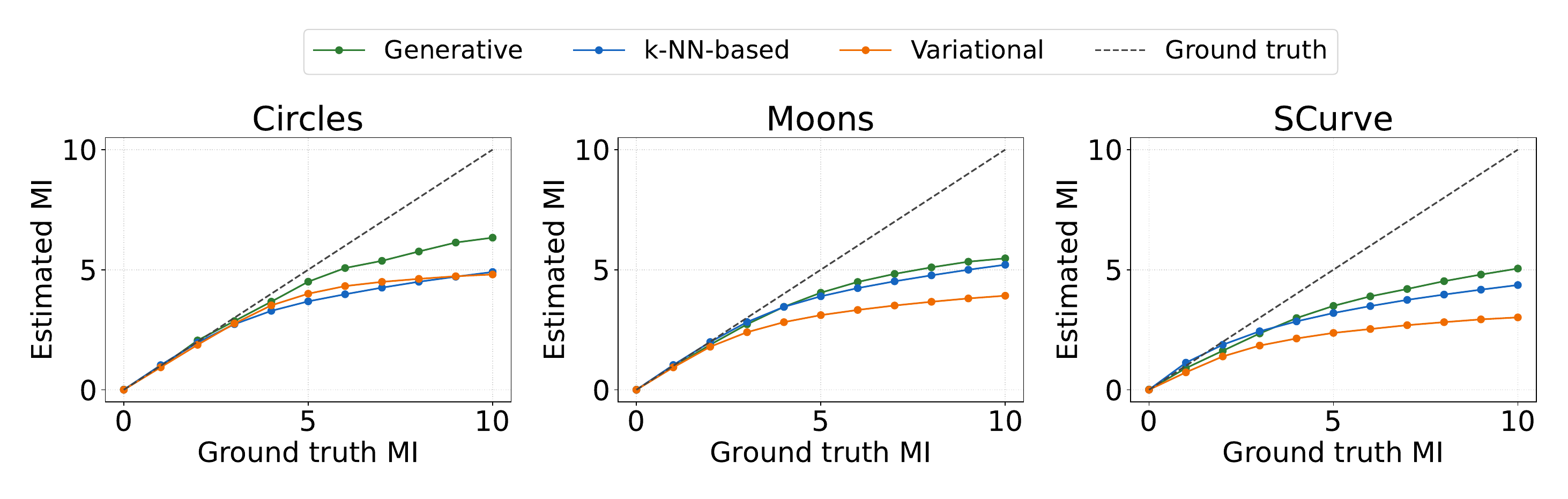}
        \caption{Best-in-class aggregation}
        \label{fig:flow_by_group}
    \end{subfigure}
    \caption{Normalizing Flow Experiments. \textbf{(a)} Individual performance of all evaluated \gls{MI} estimators. \textbf{(b)} Aggregated performance by estimator family. To clearly contrast the approaches, each line plots the single most accurate estimate within its respective class for every given distribution and ground-truth \gls{MI} pair.}
    \label{fig:flow_combined}
\end{figure}

\begin{comment}
\begin{figure}[!b]
    \centering
    \includegraphics[width=0.95\linewidth]{images/disc_est.pdf}
    \caption{Mi Estimation for discrete experiments (all estimators)}
    \label{fig:disc_est}
\end{figure}

\begin{figure}[!b]
    \centering
    \includegraphics[width=0.95\linewidth]{images/disc_group.pdf}
    \caption{Comparison of estimator families on discrete synthetic distributions. To clearly contrast k-NN, variational, and generative approaches, we aggregate the results by estimator class. For any given point on the x-axis (representing a specific distribution and ground-truth \gls{MI} pair), the plotted value shows the most accurate estimate achieved by any model within that specific family.}
    \label{fig:disc_group}
\end{figure}    
\end{comment}

% \input{tables/mi_moons}
% \input{tables/mi_scurve}
% \input{tables/mi_circles}

% \subsection{Synthetic Data~-- Discrete-Continuous mixtures experiments} This is the section where our experimental results for discrete-continuous mixtures will be located.

\subsection{Real Data~-- Synthetic Copula Experiments}

Following \Cref{synthetic_copula}, we leverage the \emph{synthetic copula} framework to generate paired samples from real-world datasets while maintaining a known ground-truth \gls{MI}. We evaluate our estimators on two image-based benchmarks: MNIST and CIFAR-10. For MNIST, we construct 50{,}000 paired, non-normalized $28 \times 28$ images without duplicates. We explore two settings here: a \textit{label pairing} approach utilizing digit class information via \emph{same-class sampling}, and a \textit{regression} approach where \gls{MI} is modulated through multiplicative brightness scaling. For the more complex CIFAR-10 dataset, operating directly on raw pixels is prohibitively challenging. Instead, we train a ResNet feature extractor---using standard augmentations and image normalization---to project the images into a lower-dimensional embedding space. We then apply the \emph{same-class sampling} pairing scheme to these CIFAR-10 embeddings to construct paired samples with varying, known \gls{MI}.

\begin{table}[t]
\centering
\scriptsize

\begin{subtable}[t]{0.49\textwidth}
\centering
\resizebox{\textwidth}{!}{
\begin{tabular}{lcccccc}
\toprule
  & \multicolumn{6}{c}{Ground-truth Mutual Information} \\
    Estimator & 0.0 & 0.5 & 1.0 & 1.5 & 2.0 & 2.3 \\
    \midrule
    KSG       & \faStopwatch & \faStopwatch & \faStopwatch & \faStopwatch & \faStopwatch & \faStopwatch \\
    WKL       & \faStopwatch & \faStopwatch & \faStopwatch & \faStopwatch & \faStopwatch & \faStopwatch \\
    MINE-DV  & 0.00 & 0.24 & 0.34 & 0.30 & 0.37 & 0.40 \\
    MINE-NWJ & 0.00 & 0.41 & 0.74 & 1.04 & 1.20 & 1.36 \\
    InfoNCE       & \faStopwatch & \faStopwatch & \faStopwatch & \faStopwatch & \faStopwatch & \faStopwatch \\
    MINDE-C  & 0.06 & 0.18 & 0.27 & 0.34 & 0.51 & 0.51 \\
    MINDE-J   & - & - & - & - & - & - \\
    \bottomrule
    \end{tabular}
    }
    \caption{Labels pairing}
    \label{tab:mnist_labels_pairing}
    \end{subtable}
    \hfill
    \begin{subtable}[t]{0.49\textwidth}
    \centering
    \resizebox{\textwidth}{!}{
    \begin{tabular}{lcccccc}
    \toprule
  & \multicolumn{6}{c}{Ground-truth Mutual Information} \\
Estimator & 0 & 1 & 2 & 3 & 4 & 5 \\
\midrule
KSG       & \faStopwatch & \faStopwatch & \faStopwatch & \faStopwatch & \faStopwatch & \faStopwatch \\
WKL       & \faStopwatch & \faStopwatch & \faStopwatch & \faStopwatch & \faStopwatch & \faStopwatch \\
MINE-DV  & 0.00 & 0.08 & 0.24 & 0.63 & 1.40 & 2.30 \\
MINE-NWJ & 0.00 & 0.27 & 0.98 & 1.64 & 2.28 & 3.49 \\
InfoNCE       & \faStopwatch & \faStopwatch & \faStopwatch & \faStopwatch & \faStopwatch & \faStopwatch \\
MINDE-C  & 0.06 & 0.02 & 0.39 & 0.97 & 1.38 & 0.93 \\
MINDE-J   & - & - & - & - & - & - \\
\bottomrule
\end{tabular}
}
\caption{Brightness multiplication}
\label{tab:mnist_brightness_multiplication}
\end{subtable}

\caption{Mutual information estimation error on MNIST. Comparison across two transformations. The ``\faStopwatch'' denotes a missing result due to time limit, and ``-'' denotes a missing value due to numerical instability.
Raw image data was used, no compression was applied.}
\label{tab:mnist_combined}
\end{table}

\textbf{K-NN-based estimators are not a viable choice in high dimensional regimes.} Throughout our synthetic copula experiments, we could not use KSG and WKL except for \Cref{tab:cifar_d16} due to expensive high dimensional distance computation. KSG, however, proved extremely effective in \Cref{tab:cifar_d16}, suggesting that pairing it with a powerful feature extractor could present a powerful alternative to current neural estimators.

\textbf{Variational estimators are robust to representation dimension.} \Cref{tab:cifar_combined} show that variational estimators do not suffer the dimensionality of CIFAR embeddings as long as those contain the required information for \gls{MI} estimation. Generative estimators, instead, are required to learn irrelevant information about the marginals even if it is not shared among the two random variables.

\textbf{Plain InfoNCE with MLP critic is impractical in high dimensions.}
In our experiments involving MNIST class pairing and brightness modulation,
we observed that standard InfoNCE fails to achieve acceptable computational efficiency for reasonably high batch sizes.
The high dimensionality of the raw pixel space, coupled with the quadratic complexity $\mathcal{O}(N^2)$ relative to the batch size $N$,
results in prohibitive overhead when utilizing an MLP critic.
To mitigate this, one must either perform preliminary dimensionality reduction % (as in our experiments with CIFAR)
or employ separable critics~\citep{tschannen2020on_DIM},
which enable some optimizations of the similarity matrix computation.

\textbf{MINDE-C is superior to MINDE-J in real world tasks.} Across all of the synthetic copula experiments MINDE-J consistently breaks, showing that its parameterization is a suboptimal choice compared with MINDE-C, which, instead, only fails on high dimensional embedding experiments in \Cref{tab:cifar_d512}. Compared to other estimators, however, MINDE-C does not appear to be the best, while it still obtain a lower estimation error in high \gls{MI} case in \Cref{tab:mnist_brightness_multiplication}.

% \vspace{0.3cm}

% In \Cref{subsubsec:mnist_class} we employed MNIST to construct datasets with known ground truth \gls{MI}; in this section we attempt the same task using the CIFAR10 dataset. Preliminary results {\color{red}ALBERTO: maybe we should include an appendix where we show just a few numbers where all estimators fail}, however, show that all \gls{MI} estimators struggle at providing meaningful \gls{MI} estimates. For this reason, we train a ResNet feature extractor to construct 512-dimensional embedding datasets. Results in \Cref{tab:mi_estimates_cifar10} show that feature extractors are extremely beneficial to variational estimators, while generative estimators do not benefit from them. {\color{red} ALBERTO: for now MINDE does not work on the embeddings, maybe because of network size; but at this point a discussion on resource demands for generative estimators should be done.}

% We employ the same techniques used in \Cref{subsubsec:mnist_class} to construct 

\begin{table}[t]
    \centering
    \scriptsize
    \begin{subtable}[t]{0.49\textwidth}
        \centering
        \resizebox{\textwidth}{!}{
            \begin{tabular}{lcccccc}
                \toprule
                  & \multicolumn{6}{c}{Ground-truth Mutual Information} \\
                Estimator & 0.0 & 0.5 & 1.0 & 1.5 & 2.0 & 2.3 \\
                \midrule
                KSG        & \textbf{0.00} & \textbf{0.00} & \textbf{0.00} & \textbf{0.00} & \textbf{0.01} & \textbf{0.00} \\
                WKL        & \textbf{0.00} & 0.50 & 1.00 & 1.50 & 2.00 & 2.30 \\
                MINE-DV    & \textbf{0.00} & 0.01 & 0.01 & 0.02 & \textbf{0.01} & 0.01 \\
                MINE-NWJ   & \textbf{0.00} & 0.02 & 0.02 & 0.03 & 0.02 & 0.02 \\
                InfoNCE    & \textbf{0.00} & 0.01 & 0.02 & 0.03 & 0.02 & 0.01 \\
                MINDE-C    & 0.07 & 0.14 & 0.18 & 0.23 & 0.19 & 0.12 \\
                MINDE-J    & 4.06 & 4.10 & 4.14 & 4.85 & 6.34 & 6.50 \\
                \bottomrule
            \end{tabular}
        }
        \caption{$\texttt{embedding\_dim}=16$}
        \label{tab:cifar_d16}
    \end{subtable}
    \hfill
    \begin{subtable}[t]{0.49\textwidth}
        \centering
        \resizebox{\textwidth}{!}{
            \begin{tabular}{lcccccc}
                \toprule
                  & \multicolumn{6}{c}{Ground-truth Mutual Information} \\
                Estimator & 0.0 & 0.5 & 1.0 & 1.5 & 2.0 & 2.3 \\
                \midrule
                KSG       & \faStopwatch & \faStopwatch & \faStopwatch & \faStopwatch & \faStopwatch & \faStopwatch \\
                WKL       & \faStopwatch & \faStopwatch & \faStopwatch & \faStopwatch & \faStopwatch & \faStopwatch \\
                MINE-DV   & \textbf{0.00} & 0.05 & 0.05 & 0.05 & 0.04 & \textbf{0.02} \\
                MINE-NWJ  & \textbf{0.00} & 0.05 & 0.05 & 0.05 & 0.05 & 0.03 \\
                InfoNCE   & \textbf{0.00} & \textbf{0.03} & \textbf{0.03} & \textbf{0.03} & \textbf{0.02} & 0.03 \\
                MINDE-C   & - & - & - & - & - & - \\
                MINDE-J   & - & - & - & - & - & - \\ 
                \bottomrule
            \end{tabular}
        }
        \caption{$\texttt{embedding\_dim}=512$}
        \label{tab:cifar_d512}
    \end{subtable}

    \caption{
        Mutual information estimation error in Labels Pairing experiment on CIFAR10 embeddings.
        The ``\faStopwatch'' denotes a missing result due to time limit, and ``-''~--- due to numerical instability.    }
    \label{tab:cifar_combined}
\end{table}

\section{Discussion}
\label{section:discussion}

Our extensive benchmarking of various \gls{MI} estimators uncovers new nuances
and provides additional experimental grounding for known challenges and problems.
Below, we carefully investigate our results, connect them to specific phenomena in \gls{MI} estimation,
and offer a thorough discussion of their corresponding preconditions and impact.

\subsection{Fundamental Problems}
\label{subsection:fundamental_problems}

Our results indicate consistent saturation of \gls{MI} estimates for high ground-truth values
that is especially pronounced in low-dimensional regimes.
This behavior is observed universally (for any estimator and distribution), even in the simplest setups,
such as the Gaussian or quantized uniform distributions,
and cannot always be tied to the difficulties identified in previous works.

\textbf{High sample complexity.}
\gls{MI} is notoriously hard to estimate due to its high sample complexity,
which typically grows exponentially with the dimensionality.
Existing literature establishes the following results:
\begin{enumerate}
    \item Any distribution-free, high-confidence \emph{lower bound} on \gls{MI} estimated from $N$ samples cannot be larger than $O(\log N)$~\citep[Theorem 4.1]{mcallester2020limitations_MI};
    \item Theorem 1 by~\citep{goldfeld2020convergence_of_SEM_entropy_estimation} derives exponential (in dimensionality) sample complexity of differential entropy estimation for a certain class of distributions.
\end{enumerate}
In contrast to previous findings, our results suggest that \gls{MI} estimation sample complexity is also exponential \emph{in \gls{MI}} for all distribution-free estimators.
The following example serves as a simple formal illustration of this limitation.
\begin{simplebox}
    \begin{example}[Sample complexity, quantized uniform distribution]
    Let $X \sim \uniform[0;1]$ and $Y = \left\lfloor n \cdot X \right\rfloor$ for some $n \in \naturals$.
    Then the \gls{MI} satisfies $\MI(X;Y) = \sent(Y) = \log n$,
    implying that at least $\exp(\MI(X;Y))$ samples are necessary to fully recover the support of $Y$.
    Conversely, for any estimator of \gls{MI} that makes no assumptions about the underlying distributions,
    given $N$ i.i.d. samples from $\jproba{X}{Y}$, the estimate is bounded above by $\log N$.
\end{example}
\end{simplebox}

For what concerns generative estimators specifically, we highlight some of their fundamental limitations by deriving the sample complexity of MINDE-C. Let $X$ and $Y$ have joint law $\mu^{C}$ with values in \(\mathbb{R}^{d_X}\) and \(\mathbb{R}^{d_Y}\) respectively, marginals $\mu^{X},\mu^{Y}$, and conditional laws $\mu^{X_y}$ of $X\mid Y=y$.
Following \citep{zhu2023sample}, we diffuse $X$ with the Ornstein--Uhlenbeck process $\mathrm{d}X_t=-\tfrac12X_t\,\mathrm{d}t+\mathrm{d}W_t$, $X_0=A$, so that $X_t=\alpha_t X+\sqrt{h_t}\,Z$ with $\alpha_t=e^{-t/2}$, $h_t=1-e^{-t}$ and $Z\sim\mathcal{N}(0,I_{d_X})$ independent of $(X,Y)$.
We write $s_u(x,t)=\nabla_x\log p_t(x)$ and $s_c(x,t,y)=\nabla_x\log p_t(x\mid y)$ for the unconditional and conditional scores of $X_t$, and $\kappa_t=-Z/\sqrt{h_t}$ for the denoising target.
The score network $\hat s(x,t,y,c)$ belongs to the class $\mathcal{S}$ of depth-$L$, width-$m$ ReLU networks with Gaussian initialisation, whose covering number we denote by $\mathcal{N}_c(\cdot,\mathcal{S})$.
As in MINDE-C, a binary input $c$ selects the conditional score $\hat s_c(x,t,y)=\hat s(x,t,y,1)$ or the unconditional score $\hat s_u(x,t)=\hat s(x,t,0,0)$.
The score errors are $e_c=\hat s_c-s_c$ and $e_u=\hat s_u-s_u$.
The sample complexity of denoising score matching is controlled by the following result.

\begin{theorem}[\citealp{zhu2023sample}, Thm.~3]\label{thm:zlc}
Let $x\sim p_0$ on $\mathbb{R}^{d}$ with $\|x\|_2\le C_d$, let $\nabla\log p_t$ be $1$-Lipschitz, and let $\hat s\in\mathcal{S}$ minimise the empirical denoising score-matching loss
\[
\hat{\mathcal{L}}(s)=\frac1n\sum_{i=1}^n\frac{1}{T-t_0}\int_{t_0}^{T}\mathbb{E}_{x_t\sim p_{0t}(\cdot\mid x_{(i)})}\big\|\nabla_{x_t}\log p_{0t}(x_t\mid x_{(i)})-s(x_t,t)\big\|_2^2\,\mathrm{d}t
\]
over $n$ i.i.d.\ samples $x_{(i)}\sim p_0$.
Then, for any $\varepsilon,\delta\in(0,1)$, with probability at least $1-2\delta-2L\exp(-\Omega(m))$ over the initialisation and the denoising noise,
\[
\frac{1}{T-t_0}\int_{t_0}^{T}\big\|\nabla\log p_t-\hat s(\cdot,t)\big\|^2_{\ell^2(p_t)}\,\mathrm{d}t
\lesssim\frac{1}{n\varepsilon^2}\Big(\frac{d\,(T-\log t_0)}{T-t_0}+C_d^2\Big)\log\frac{\mathcal{N}_c(1/n,\mathcal{S})}{\delta}+\frac1n+d\,\varepsilon^2 ,
\]
where $\varepsilon$ is the sup-norm accuracy at which $\mathcal{S}$ approximates the score.
\end{theorem}

The estimation error of MINDE is characterised as follows.

\begin{proposition}[\citealp{franzese2024minde}, Eqs.~(10)--(12)]\label{prop:minde}
Let $\mu^{P},\mu^{Q}$ be diffused by the same forward SDE with diffusion coefficient $g_t$, time marginals $\nu_t^{P},\nu_t^{Q}$ and scores $s_t^{P},s_t^{Q}$. Then
\[
\mathrm{KL}\big[\mu^{P}\,\|\,\mu^{Q}\big]=\int_0^T\frac{g_t^2}{2}\,\mathbb{E}_{\nu_t^{P}}\big\|s_t^{P}-s_t^{Q}\big\|^2\mathrm{d}t+\mathrm{KL}\big[\nu_T^{P}\,\|\,\nu_T^{Q}\big].
\]
If the scores are replaced by approximations $\tilde s_t^{\bullet}=s_t^{\bullet}+e_t^{\bullet}$, the estimator $\int_0^T\frac{g_t^2}{2}\mathbb{E}_{\nu_t^{P}}\|\tilde s_t^{P}-\tilde s_t^{Q}\|^2\mathrm{d}t$ differs from the integral above by
\[
\mathcal{D}=\int_0^T\frac{g_t^2}{2}\,\mathbb{E}_{\nu_t^{P}}\Big[\big\|e_t^{P}-e_t^{Q}\big\|^2+2\big\langle s_t^{P}-s_t^{Q},\,e_t^{P}-e_t^{Q}\big\rangle\Big]\mathrm{d}t .
\]
\end{proposition}

Since $\MI(X;Y)=\int\mathrm{KL}[\mu^{X_y}\,\|\,\mu^{X}]\,\mathrm{d}\mu^{Y}(y)$, MINDE-C applies Proposition~\ref{prop:minde} with $\mu^{P}=\mu^{X_y}$ and $\mu^{Q}=\mu^{X}$, averaged over $y\sim\mu^{Y}$.
Because Theorem~\ref{thm:zlc} is vacuous at $t_0=0$, we restrict the time integral to $[t_0,T]$; with $g_t=1$ and expectations taken over $(X,Y,Z)$, this defines
\[
\hat I=\tfrac12\int_{t_0}^{T}\mathbb{E}\big\|\hat s_c(X_t,t,B)-\hat s_u(X_t,t)\big\|^2\mathrm{d}t,
\qquad
I_{[t_0,T]}=\tfrac12\int_{t_0}^{T}\mathbb{E}\big\|s_c(X_t,t,B)-s_u(X_t,t)\big\|^2\mathrm{d}t .
\]
By Proposition~\ref{prop:minde}, $I_{[t_0,T]}=I(X_{t_0};B)-I(X_T;B)\le \MI(X;Y)$, and
\[
\mathcal{D}=\hat I-I_{[t_0,T]}=\Delta+\int_{t_0}^{T}\mathbb{E}\big\langle s_c-s_u,\,e_c-e_u\big\rangle\,\mathrm{d}t,
\qquad
\Delta=\tfrac12\int_{t_0}^{T}\mathbb{E}\big\|e_c-e_u\big\|^2\mathrm{d}t .
\]
Combining Theorem~\ref{thm:zlc} and Proposition~\ref{prop:minde} yields the following bound.

\begin{proposition}\label{prop:minde-sc}
Assume $\|(X,Y)\|_2\le C_d$ almost surely and that, for $t\in[t_0,T]$, the maps $x\mapsto s_u(x,t)$ and $(x,y)\mapsto s_c(x,t,y)$ are $1$-Lipschitz.
Let $\hat s\in\mathcal{S}$ minimise, over $n$ i.i.d.\ pairs $(a_i,b_i)\sim\mu^{C}$, the empirical loss
\[
\hat{\mathcal{L}}(s)=\frac1n\sum_{i=1}^n\frac{1}{T-t_0}\int_{t_0}^{T}\mathbb{E}_{Z}\Big[\tfrac12\big\|s(X_t^{(i)},t,b_i,1)-\kappa_t\big\|^2+\tfrac12\big\|s(X_t^{(i)},t,0,0)-\kappa_t\big\|^2\Big]\mathrm{d}t,
\qquad X_t^{(i)}=\alpha_ta_i+\sqrt{h_t}\,Z .
\]

Then, for any $\varepsilon,\delta\in(0,1)$, with probability at least $1-2\delta-2L\exp(-\Omega(m))$,
\[
\big|\hat I-I_{[t_0,T]}\big|\le\Delta+2\sqrt{I_{[t_0,T]}\,\Delta},
\qquad
\Delta\lesssim2\,B_n(\varepsilon,T,t_0),
\]
where
\[
B_n(\varepsilon,T,t_0)=\frac{1}{n\varepsilon^2}\Big(d_X\,(T-\log t_0)+C_d^2\,(T-t_0)\Big)\log\frac{\mathcal{N}_c(1/n,\mathcal{S})}{\delta}+\frac{T-t_0}{n}+d_X\,(T-t_0)\,\varepsilon^2 .
\]
Consequently, $\big|\hat I-I_{[t_0,T]}\big|\lesssim2B_n+2\sqrt{2\,\MI(X;Y)\,B_n}$.
\end{proposition}

\Cref{prop:minde-sc}, with its proof in \cref{app:minde_sc_proof}, shows that the sample complexity of estimating \gls{MI} is essentially the same as the sample complexity of learning the score. As a consequence, for a score with a fixed Lipschitz constant, the number of samples required to accurately estimate \gls{MI} with MINDE-C scales linearly with \(d_X\) and sublinearly with the \gls{MI}. Similar considerations hold for MINDE-J, where instead of \(d_X\) we have \(d_{X}+d_Y\), highlighting a design advantage of the conditional variant compared to the joint variant: in MINDE-C, \(d_Y\) enters only through \(C_d\) and the covering number, not through the \(d_X(T-\log t_0)\) term that dominates for small \(t_0\), so the larger representation can be used as the conditioning variable, while MINDE-J must account for both. Crucially, the Lipschitz assumption is not a technicality. First, it hides a dependence on the \gls{MI} per component: as the \gls{MI} grows, the conditional distribution concentrates in a neighbourhood of width \(\sim e^{-I(X;Y)}\) around a lower-dimensional set, and the conditional score becomes correspondingly steep. For a bivariate Gaussian with correlation \(\rho\), the conditional score of \(X_t\) has Lipschitz constant \(1/(\alpha_t^2(1-\rho^2)+h_t)\approx 1/(e^{-2I(X;Y)}+t_0)\) near \(t_0\). 
\begin{wrapfigure}[13]{r}{0.3\columnwidth}
    \centering
    \includegraphics[width=\linewidth]{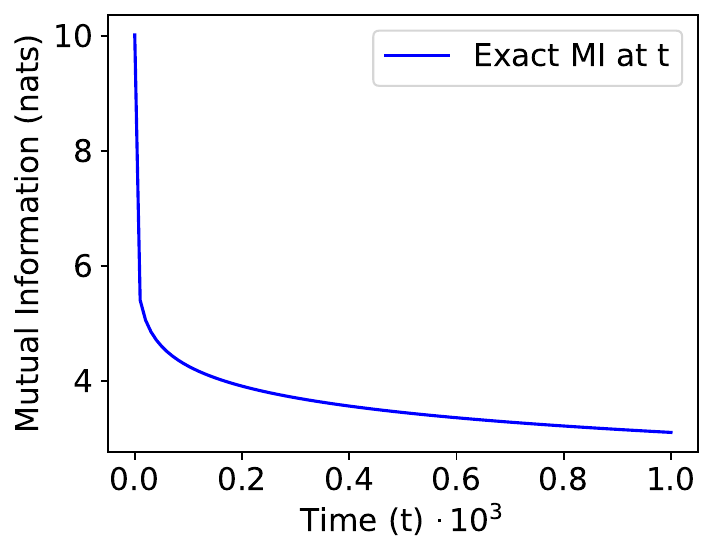}
    \caption{\gls{MI} vs Diffusion Time}
    \label{fig:mi_v_t}
\end{wrapfigure}
The bound is therefore sublinear in the total \gls{MI} only when the \gls{MI} per component stays bounded, as in our stacked high-dimensional tests. Second, it depends on the regularity of the marginals: the scores of the correlated uniform and smoothed uniform distributions are non-Lipschitz at the boundary of their support, and that of the log-gamma-exponential distribution in its tail, whereas the Gaussian score is linear. Diffusion makes these scores Lipschitz on \([t_0,T]\), but with a constant that grows as \(t_0\to 0\), which a finite network can only match with more parameters or a larger approximation error. Consistently, at \(d=64\) in \cref{fig:cont_synth_dim_all}, MINDE-C recovers about 100 of 128 nats on the correlated normal, but only 50, 82 and 72 on the correlated uniform, smoothed uniform and log-gamma-exponential distributions; since the correlated uniform shares the Gaussian copula, this gap is due entirely to the marginals. Finally, note that the term in \cref{prop:minde-sc} does not take into account the \gls{MI} lost in \([0,t_0]\): in practice, the diffusion model never sees real clean data and cannot distinguish the true data from its slightly noisy version. Statistically, this introduces a previously overlooked bias in MINDE. Consequently, the estimator computes a perturbed \gls{MI} that severely penalizes highly concentrated distributions. \Cref{example:diffusion_pitfall} shows that this loss of information can be catastrophic in high \gls{MI} cases.

\WFclear

\begin{simplebox}
    \begin{example}[Error inflation in concentrated distributions]
    \label{example:diffusion_pitfall}
    Consider the noisy version of the random variables under time evolution:
    \[\label{eq:ve_sde}
        X_t = X_0 + \sqrt{t}Z_X \qquad
        Y_t = Y_0 + \sqrt{t}Z_Y \nonumber
    \]
    where \(Z_X\) and \(Z_Y\) are standard normal distributions.
    By the De Bruijn identity, we have:
    \[
        \frac{\dif}{\dif t}\MI(X_t; Y_t)=-\frac{1}{2}\expect_{x,y}\left[\|\nabla_{xy}\PMI(X_t,Y_t)\|^2\right]
    \]
    Expanding \gls{MI} near zero yields the value effectively computed by MINDE:
    \[
        \MI(X_{t_0};Y_{t_0}) = \MI(X;Y) - \frac{t_0}{2}\mathbb{E}_{x,y}\left[\| \nabla_{xy}\PMI(X_{t_0},Y_{t_0})\|^2\right].
    \]
    Statistically, this adds a bias term to MINDE. For high \gls{MI} distributions, the joint distribution is highly concentrated. This causes the gradient norm (and the resulting error term) to become exceptionally large. We demonstrate this using a bivariate Gaussian distribution with zero mean and correlation coefficient \(\rho=1-10^{-9}\). As shown in \Cref{fig:mi_v_t}, \gls{MI} drops dramatically in the very first instants of the diffusion process. In order to recover the lost \gls{MI} one can lower \(t_0\); however \cref{prop:minde-sc} shows that this impacts the sample complexity as \(-\log t_0\) for \(t_0\to 0\).
    \end{example}
\end{simplebox}

% --- THE KILL SWITCH ---
% This forces LaTeX to stop wrapping and resets the margins for the rest of the document.
\WFclear

\textbf{Numerical instability.}
Given that \gls{MI} estimates saturate even in the Gaussian case,
where \gls{MI} is a function of the covariance matrix~--
a simple statistic that is \emph{not} exponentially hard to estimate in high dimensions~\citep{cai2015logdet}~--
one might conclude that sample complexity is not the only limiting factor in measuring \gls{MI}.
We interpret this as a distinct manifestation of a problem not addressed in previous work:
even when \gls{MI} can be expressed through well-behaved statistics,
the relation may not be numerically stable enough to allow precise estimation.
The following example illustrates the phenomenon.

\begin{simplebox}
    \begin{example}[Numerical instability, Gaussian distribution]
    \label{example:instability}
    Let $(X,Y) \sim \normal(0, \Sigma)$, where
    \[
        \Sigma =
        \begin{pmatrix}
            1 & \rho \\
            \rho & 1
        \end{pmatrix},
        \qquad
        \rho \in (-1;1)
    \]
    Then $\MI(X;Y) = -\frac{1}{2} \log(1 - \rho^2)$ and $\partial \MI(X;Y) / \partial \rho = \frac{\rho}{1 - \rho^2}$.
    Consequently, given a correlation coefficient estimation error $\Delta \rho$,
    the \gls{MI} estimation error satisfies
    \begin{align*}
        \Delta \MI(X;Y)
        &= \frac{\rho}{1 - \rho^2} \cdot \Delta \rho + O((\Delta \rho)^2) \\
        &= e^{2 \MI(X;Y)} \sqrt{1 - e^{-2 \MI(X;Y)}} \cdot \Delta \rho + O((\Delta \rho)^2)
    \end{align*}
    Here $\sqrt{1 - e^{-2 \MI(X;Y)}} \to 1$ for high \gls{MI}, and $e^{2 \MI(X;Y)}$ grows exponentially with $\MI(X;Y)$. Moreover, setting $\Delta \rho = 10^{-6}$ (approximately single-precision floating-point machine epsilon) and $\rho = 1 - \Delta \rho$ yields $\MI(X;Y) \approx 6.56 \; \textnormal{nats}$~-- a figure that is close to the saturation value we observe in our setups.
\end{example}
\end{simplebox}

\textbf{High variance.}
The problem of the high variance of the \gls{MI} estimates is widely recognized in the literature~\citep{poole2019variational_bounds,song2020understanding_limitations,letizia2024fDIME}.
However, this issue is typically tied to a particular estimator (e.g., to MINE without clipping~\citep{song2020understanding_limitations}),
and not to the fundamentally complex nature of copulas and \gls{MI}.
In contrast, we suggest shifting the focus from methods to distributions:
we argue that some pairs of $X$ and $Y$ inherently yield high-variance estimates for a wide range of estimators.
Recall that \gls{MI} can be viewed as an expected value of \gls{PMI} (see \eqref{eq:mutual_information})~---
$\PMI(x,y) = \log \frac{\partial \jproba{X}{Y}}{\partial \mproba{X}{Y}}(x,y)$.
Note that $\PMI(X,Y)$ is a random variable that may be of high variance on its own.
Given that majority of the \gls{MI} estimators approximate \gls{PMI}
either explicitly (via the density ratio, as in MIENF~\citep{butakov2024normflows} and VCE~\citep{chen2025vector_copula})
or implicitly (via the classifier, as in virtually any variational method),
the estimate may still be unreliable even if the \gls{PMI} is approximated precisely.
The following example illustrates this claim.
\begin{simplebox}
    \begin{example}[High underlying variance]
    %Let $(X,Y) \sim \uniform \left([0;p]^2 \cup \bigcup_{k=1}^n \left[p+\frac{(k-1)/n}{1-p};p+\frac{k/n}{1-p}\right] \right)$.
    %Then $\MI(X;Y) = h(p) + (1-p) \cdot \log n$,
    Let $X \sim \uniform[0;1]$ and
    \[
        Y \mid X \sim
        \begin{cases}
            \uniform[0;p], &\quad X \in [0;p] \\
            \uniform\!\left[p + (1-p)\frac{\lfloor X \cdot n \rfloor}{n};p + (1-p)\frac{\lfloor X \cdot n \rfloor + 1}{n}\right], &\quad X \in [p;1]
        \end{cases}
    \]
    Then
    \[
        \PMI(X,Y) =
        \begin{cases}
            -\log p, &\quad \textnormal{with probability $p$} \\
            \log n -\log (1 - p), &\quad \textnormal{with probability $1 - p$} \\
        \end{cases}
    \]
    Consequently, $\MI(X;Y) = \expect \PMI(X,Y) = h(p) + (1-p) \cdot \log n$
    and $\variance [\PMI(X,Y)] = p (1-p) (\log n - \log (1/p-1))^2$ %= p \log^2 p + (1 - p) (\log n - \log (1-p))^2$,
    where $h(\blankarg)$ is the binary cross-entropy function.
\end{example}
\end{simplebox}

In practice, however, this intrinsic variance may manifest itself as increased bias of the estimate,
as estimators are typically unable to learn the highly variable parts of the distribution.
We indeed observe such behavior in our experiments with the Rare Event Channel.

% ----------------------
\subsection{Practical Observations}
\label{subsection:practical_observations}

\textbf{Non-parametric methods are best for low-dimensional data.}
Our low-dimensional benchmarks show that non-parametric methods outperform competitors in low-dimensional scenarios, especially at high \gls{MI}.
Remarkably, this estimator class is also the most lightweight to run,
whereas other estimators are much more compute intensive since they require training a variational network or a generative model.
%On the other hand, non-parametric estimators typically struggle in high-dimensional regimes.

\textbf{Variational estimators struggle to capture intricate structure.}
Recently, new compression-based \gls{MI} estimation techniques have been proposed to enable non-parametric methods (such as KSG or WKL) in high-dimensional settings~\citep{butakov2024lossy_compression,gowri2024LMI}.
However, prior to our work, applying such techniques to neural estimators seemed unnecessary,
since manifold learning is expected to be performed implicitly, as a side-product.
Indeed, variational (variational) estimators essentially learn the \gls{PMI}~---
$\partial \jproba{X}{Y} / \partial \mproba{X}{Y}$~-- which remains well-defined even for data lying on a low-dimensional manifold.
Nevertheless, despite this theoretical robustness, we observe severe limitations of these estimators on high-dimensional real data, including a complete failure to accurately estimate \gls{MI} on the uncompressed CIFAR-10 dataset.
In stark contrast, the same estimators achieve nearly perfect performance when the data is compressed beforehand (see \Cref{tab:cifar_combined}).
These results lead us to the following conclusion:
preparatory compression of the data is a viable and effective technique for improving estimation quality,
\emph{even for neural methods}.

\section{Conclusion}
\label{section:conclusion}

While recent years have witnessed a proliferation of neural \gls{MI} estimators, no single approach has emerged as universally superior. Progress has been largely incremental, and despite the adoption of cutting-edge models and concepts, the results remain limited: in certain regimes, modern neural estimators offer little to no improvement over simpler methods, while demanding substantially more computational resources. To highlight these challenges, we provide a comprehensive overview of the task and existing approaches, followed by an extensive benchmarking of the most popular \gls{MI} estimation paradigms.
Beyond corroborating previously documented limitations, % corroborating? wow, what a word! XD
our results reveal the following fundamental and practical challenges:
\begin{itemize}
    \item \emph{Numerical instability}, especially in high-\gls{MI} settings;
    \item \emph{High variance of PMI}, which propagates into \gls{MI} estimates and renders them unreliable;
    \item \emph{Inability of neural estimators to learn useful representations}: in many cases achieving better results requires explicit compression before estimation;
    \item \emph{An abundance and variability of estimator-specific practical considerations}:
        not only is there no single winner in terms of estimation accuracy,
        but there is also no method that \emph{just works} across all tested regimes;
    \item \emph{High sample complexity} for both generative and variational estimators.
\end{itemize}

Our results highlight a distinct trade-off.
Generative \gls{MI} estimators offer a significant advantage in high-dimensional and high-\gls{MI} tasks,
but they require complex modeling even for straightforward dependencies.
Conversely, variational estimators perform optimally when paired with expressive neural embeddings,
yet they are fundamentally constrained by a sample complexity that scales exponentially with the true \gls{MI}
and may fail without expressive compressed representations.
Finally, inexpensive $k$-NN-based nonparametric methods remain unchallenged in low-dimensional, low-\gls{MI}, and small-sample-size regimes,
while failing on raw high-dimensional and complex data. Our work establishes a comprehensive evaluation framework that exposes the inherent pitfalls of current state-of-the-art estimators,
providing a rigorous baseline for future developments.
By identifying these novel properties and specific failure modes,
we offer practitioners a roadmap for estimator selection and present a clear set of challenges
that must be addressed to advance the next generation of neural \gls{MI} estimation.

%\paragraph{Alberto:} While recent years have witnessed a proliferation of neural \gls{MI} estimators, no single approach has emerged as universally superior. In this work, we demonstrate through extensive benchmarking that classical estimators remain competitive in low-dimensional regimes across diverse distributions. Our results highlight a distinct trade-off: while generative \gls{MI} estimators offer a significant advantage in high-dimensional and high-\gls{MI} tasks, they necessitate complex modeling even for straightforward dependencies. Conversely, variational estimators perform optimally when paired with expressive neural embeddings but are fundamentally constrained by sample complexity that scales exponentially with the underlying \gls{MI} and fail without compressed expressive representations. Our work establishes a comprehensive evaluation framework that exposes the inherent pitfalls of current state-of-the-art estimators, providing a rigorous baseline for future developments in the field. By identifying these novel properties and specific failure modes, we provide practitioners with a roadmap for estimator selection and offer a clear set of challenges that must be addressed to advance the next generation of neural \gls{MI} estimation.

\bibliography{main.bib}

@article{kozachenko1987entropy_of_random_vector,
    title = {Sample Estimate of the Entropy of a Random Vector},
    author = {L. F. Kozachenko and N. N. Leonenko},
    journal = {Problems Inform. Transmission},
    pages = {95--101},
    year = {1987},
    volume = {23},
    issue = {2},
}

@article{
    moddemeijer1989binning_MI,
    title = {On estimation of entropy and mutual information of continuous distributions},
    journal = {Signal Processing},
    volume = {16},
    number = {3},
    pages = {233-248},
    year = {1989},
    issn = {0165-1684},
    doi = {https://doi.org/10.1016/0165-1684(89)90132-1},
    url = {https://www.sciencedirect.com/science/article/pii/0165168489901321},
    author = {R. Moddemeijer}
}

@article{mccann1995measure_preserving_maps,
    author = {McCann, Robert},
    year = {1995},
    month = {11},
    pages = {},
    title = {Existence and uniqueness of monotone measure-preserving maps},
    volume = {80},
    journal = {Duke Mathematical Journal},
    doi = {10.1215/S0012-7094-95-08013-2}
}

@article{moon1995KDE_MI,
    title = {Estimation of mutual information using kernel density estimators},
    author = {Moon, Young-Il and Rajagopalan, Balaji and Lall, Upmanu},
    journal = {Phys. Rev. E},
    volume = {52},
    issue = {3},
    pages = {2318--2321},
    numpages = {0},
    year = {1995},
    month = {Sep},
    publisher = {American Physical Society},
    doi = {10.1103/PhysRevE.52.2318},
    url = {https://link.aps.org/doi/10.1103/PhysRevE.52.2318}
}

@article{kraskov2004KSG,
    title = {Estimating mutual information},
    author = {Kraskov, Alexander and St\"ogbauer, Harald and Grassberger, Peter},
    journal = {Phys. Rev. E},
    volume = {69},
    issue = {6},
    pages = {066138},
    numpages = {16},
    year = {2004},
    month = {Jun},
    publisher = {American Physical Society},
    doi = {10.1103/PhysRevE.69.066138},
    url = {https://link.aps.org/doi/10.1103/PhysRevE.69.066138}
}

@book{cover2006information_theory,
    author = {Cover, Thomas M. and Thomas, Joy A.},
    title = {Elements of Information Theory (Wiley Series in Telecommunications and Signal Processing)},
    year = {2006},
    publisher = {Wiley-Interscience},
    address = {USA}
}

@article{nguyen2010NWJ,
    title={Estimating divergence functionals and the likelihood ratio by convex risk minimization},
    author={Nguyen, XuanLong and Wainwright, Martin J and Jordan, Michael I},
    journal={IEEE Transactions on Information Theory},
    volume={56},
    number={11},
    pages={5847--5861},
    year={2010},
    publisher={IEEE}
}

@article{ma2011MI_is_copula_entropy,
    title = {Mutual Information Is Copula Entropy},
    journal = {Tsinghua Science \& Technology},
    volume = {16},
    number = {1},
    pages = {51-54},
    year = {2011},
    issn = {1007-0214},
    doi = {https://doi.org/10.1016/S1007-0214(11)70008-6},
    url = {https://www.sciencedirect.com/science/article/pii/S1007021411700086},
    author = {Jian Ma and Zengqi Sun}
}

@article{arellano_valle2013MI_for_skew_distributions,
    author = {Arellano-Valle, R. and Contreras-Reyes, Javier and Genton, Marc},
    year = {2013},
    month = {03},
    pages = {42-62},
    title = {Shannon Entropy and Mutual Information for Multivariate Skew-Elliptical Distributions},
    volume = {40},
    journal = {Scandinavian Journal of Statistics},
    doi = {10.1111/j.1467-9469.2011.}
}

@misc{kingma2013VAE,
    title={Auto-Encoding Variational Bayes}, 
    author={Diederik P Kingma and Max Welling},
    year={2013},
    eprint={1312.6114},
    archivePrefix={arXiv},
    primaryClass={stat.ML},
    url={https://arxiv.org/abs/1312.6114}, 
}

@article{tabak2013nonparametric_PDF_estimators,
    author = {Tabak, E. G. and Turner, Cristina V.},
    title = {A Family of Nonparametric Density Estimation Algorithms},
    journal = {Communications on Pure and Applied Mathematics},
    volume = {66},
    number = {2},
    pages = {145-164},
    doi = {https://doi.org/10.1002/cpa.21423},
    url = {https://onlinelibrary.wiley.com/doi/abs/10.1002/cpa.21423},
    eprint = {https://onlinelibrary.wiley.com/doi/pdf/10.1002/cpa.21423},
    year = {2013}
}

@INPROCEEDINGS{sulaiman2015MI_feature_selection,
  author={Sulaiman, Muhammad Aliyu and Labadin, Jane},
  booktitle={2015 9th International Conference on IT in Asia (CITA)}, 
  title={Feature selection based on mutual information}, 
  year={2015},
  volume={},
  number={},
  pages={1-6},
  doi={10.1109/CITA.2015.7349827}
}

@article{cai2015logdet,
    title = {Law of log determinant of sample covariance matrix and optimal estimation of differential entropy for high-dimensional Gaussian distributions},
    journal = {Journal of Multivariate Analysis},
    volume = {137},
    pages = {161-172},
    year = {2015},
    issn = {0047-259X},
    doi = {https://doi.org/10.1016/j.jmva.2015.02.003},
    url = {https://www.sciencedirect.com/science/article/pii/S0047259X1500038X},
    author = {T. Tony Cai and Tengyuan Liang and Harrison H. Zhou}
}

@inproceedings{chen2016InfoGAN,
    author = {Chen, Xi and Duan, Yan and Houthooft, Rein and Schulman, John and Sutskever, Ilya and Abbeel, Pieter},
    booktitle = {Advances in Neural Information Processing Systems},
    editor = {D. Lee and M. Sugiyama and U. Luxburg and I. Guyon and R. Garnett},
    pages = {},
    publisher = {Curran Associates, Inc.},
    title = {InfoGAN: Interpretable Representation Learning by Information Maximizing Generative Adversarial Nets},
    url = {https://proceedings.neurips.cc/paper_files/paper/2016/file/7c9d0b1f96aebd7b5eca8c3edaa19ebb-Paper.pdf},
    volume = {29},
    year = {2016}
}

@inproceedings{gao2017discrete_KSG,
    author = {Gao, Weihao and Kannan, Sreeram and Oh, Sewoong and Viswanath, Pramod},
    booktitle = {Advances in Neural Information Processing Systems},
    editor = {I. Guyon and U. Von Luxburg and S. Bengio and H. Wallach and R. Fergus and S. Vishwanathan and R. Garnett},
    pages = {},
    publisher = {Curran Associates, Inc.},
    title = {Estimating Mutual Information for Discrete-Continuous Mixtures},
    url = {https://proceedings.neurips.cc/paper_files/paper/2017/file/ef72d53990bc4805684c9b61fa64a102-Paper.pdf},
    volume = {30},
    year = {2017}
}

@inproceedings{foresti2026information,
  title={Information Estimation with Discrete Diffusion},
  author={Foresti, Alberto and Franzese, Giulio and Michiardi, Pietro},
  booktitle={International Conference on Learning Representations},
  volume={2026},
  pages={21001--21043},
  year={2026}
}

@article{vincent2011connection,
  title={A connection between score matching and denoising autoencoders},
  author={Vincent, Pascal},
  journal={Neural computation},
  volume={23},
  number={7},
  pages={1661--1674},
  year={2011},
  publisher={MIT Press}
}

@article{zhu2023sample,
  title={Sample complexity bounds for score-matching: Causal discovery and generative modeling},
  author={Zhu, Zhenyu and Locatello, Francesco and Cevher, Volkan},
  journal={Advances in Neural Information Processing Systems},
  volume={36},
  pages={3325--3337},
  year={2023}
}

@InProceedings{runge2018conditional_independence_testing_MI,
  title = 	 {Conditional independence testing based on a nearest-neighbor estimator of conditional mutual information},
  author = 	 {Runge, Jakob},
  booktitle = 	 {Proceedings of the Twenty-First International Conference on Artificial Intelligence and Statistics},
  pages = 	 {938--947},
  year = 	 {2018},
  editor = 	 {Storkey, Amos and Perez-Cruz, Fernando},
  volume = 	 {84},
  series = 	 {Proceedings of Machine Learning Research},
  month = 	 {09--11 Apr},
  publisher =    {PMLR},
  url = 	 {https://proceedings.mlr.press/v84/runge18a.html}
}

@article{goldfeld2018estimating,
    title={Estimating information flow in deep neural networks},
    author={Goldfeld, Ziv and Berg, Ewout van den and Greenewald, Kristjan and Melnyk, Igor and Nguyen, Nam and Kingsbury, Brian and Polyanskiy, Yury},
    journal={arXiv preprint arXiv:1810.05728},
    year={2018}
}

@inproceedings{hjelm2018deep_infomax,
    title={Learning deep representations by mutual information estimation and maximization},
    author={R Devon Hjelm and Alex Fedorov and Samuel Lavoie-Marchildon and Karan Grewal and Phil Bachman and Adam Trischler and Yoshua Bengio},
    booktitle={International Conference on Learning Representations},
    year={2019},
    url={https://openreview.net/forum?id=Bklr3j0cKX},
}

@inproceedings{belghazi2018mine,
    title = {Mutual Information Neural Estimation},
    author = {Belghazi, Mohamed Ishmael and Baratin, Aristide and Rajeshwar, Sai and Ozair, Sherjil and Bengio, Yoshua and Courville, Aaron and Hjelm, Devon},
    booktitle = {Proceedings of the 35th International Conference on Machine Learning},
    pages = {531--540},
    year = {2018},
    editor = {Dy, Jennifer and Krause, Andreas},
    volume = {80},
    series = {Proceedings of Machine Learning Research},
    month = {07},
    publisher = {PMLR},  url = {https://proceedings.mlr.press/v80/belghazi18a.html}
}

@misc{oord2019infoNCE,
      title={Representation Learning with Contrastive Predictive Coding}, 
      author={Aaron van den Oord and Yazhe Li and Oriol Vinyals},
      year={2019},
      eprint={1807.03748},
      archivePrefix={arXiv},
      primaryClass={cs.LG},
      url={https://arxiv.org/abs/1807.03748}, 
}

@article{berrett2019efficient_knn_entropy_estimation,
    author = {Berrett, Thomas B. and Samworth, Richard J. and Yuan, Ming},
    doi = {10.1214/18-AOS1688},
    fjournal = {Annals of Statistics},
    journal = {Ann. Statist.},
    month = {02},
    number = {1},
    pages = {288--318},
    publisher = {The Institute of Mathematical Statistics},
    title = {Efficient multivariate entropy estimation via $k$-nearest neighbour distances},
    url = {https://doi.org/10.1214/18-AOS1688},
    volume = {47},
    year = {2019}
}

@InProceedings{poole2019variational_bounds,
    title = 	 {On Variational Bounds of Mutual Information},
    author =       {Poole, Ben and Ozair, Sherjil and Van Den Oord, Aaron and Alemi, Alex and Tucker, George},
    booktitle = 	 {Proceedings of the 36th International Conference on Machine Learning},
    pages = 	 {5171--5180},
    year = 	 {2019},
    editor = 	 {Chaudhuri, Kamalika and Salakhutdinov, Ruslan},
    volume = 	 {97},
    series = 	 {Proceedings of Machine Learning Research},
    month = 	 {09--15 Jun},
    publisher =    {PMLR},
    url = 	 {https://proceedings.mlr.press/v97/poole19a.html}
}

@inproceedings{tschannen2020on_DIM,
    title={On Mutual Information Maximization for Representation Learning},
    author={Michael Tschannen and Josip Djolonga and Paul K. Rubenstein and Sylvain Gelly and Mario Lucic},
    booktitle={International Conference on Learning Representations},
    year={2020},
    url={https://openreview.net/forum?id=rkxoh24FPH}
}

@article{geenens2020discrete_copula,
    url = {https://doi.org/10.1515/demo-2020-0022},
    title = {Copula modeling for discrete random vectors},
    title = {},
    author = {Gery Geenens},
    pages = {417--440},
    volume = {8},
    number = {1},
    journal = {Dependence Modeling},
    doi = {doi:10.1515/demo-2020-0022},
    year = {2020},
    lastchecked = {2026-06-16}
}

@article{goldfeld2020convergence_of_SEM_entropy_estimation,
    author={Z. {Goldfeld} and K. {Greenewald} and J. {Niles-Weed} and Y. {Polyanskiy}},
    journal={IEEE Transactions on Information Theory}, 
    title={Convergence of Smoothed Empirical Measures With Applications to Entropy Estimation}, 
    year={2020},
    volume={66},
    number={7},
    pages={4368-4391},
    doi={10.1109/TIT.2020.2975480}
}

@InProceedings{mcallester2020limitations_MI,
    title = 	 {Formal Limitations on the Measurement of Mutual Information},
    author =       {McAllester, David and Stratos, Karl},
    booktitle = 	 {Proceedings of the Twenty Third International Conference on Artificial Intelligence and Statistics},
    pages = 	 {875--884},
    year = 	 {2020},
    editor = 	 {Chiappa, Silvia and Calandra, Roberto},
    volume = 	 {108},
    series = 	 {Proceedings of Machine Learning Research},
    month = 	 {08},
    publisher =    {PMLR},
    url = 	 {https://proceedings.mlr.press/v108/mcallester20a.html}
}

@inproceedings{song2020understanding_limitations,
    title={Understanding the Limitations of Variational Mutual Information Estimators},
    author={Jiaming Song and Stefano Ermon},
    booktitle={International Conference on Learning Representations},
    year={2020},
    url={https://openreview.net/forum?id=B1x62TNtDS}
}

@InProceedings{rhodes2020TRE,
    title={Telescoping Density-Ratio Estimation}, 
    author={Benjamin Rhodes and Kai Xu and Michael U. Gutmann},
    year={2020},
    eprint={2006.12204},
    archivePrefix={arXiv},
    primaryClass={stat.ML},
    booktitle={The Thirty-fourth Annual Conference on Neural Information Processing Systems},
    url={https://arxiv.org/abs/2006.12204}, 
}

@inproceedings{arbel2021generalized_EBMs,
    title={Generalized Energy Based Models},
    author={Michael Arbel and Liang Zhou and Arthur Gretton},
    booktitle={International Conference on Learning Representations},
    year={2021},
    url={https://openreview.net/forum?id=0PtUPB9z6qK}
}

@misc{fan2021vector_copulas,
    title={Vector copulas}, 
    author={Yanqin Fan and Marc Henry},
    year={2021},
    eprint={2009.06558},
    archivePrefix={arXiv},
    primaryClass={econ.EM},
    url={https://arxiv.org/abs/2009.06558}, 
}

@InProceedings{choi2022DRE-infinity,
    title = 	 { Density Ratio Estimation via Infinitesimal Classification },
    author =       {Choi, Kristy and Meng, Chenlin and Song, Yang and Ermon, Stefano},
    booktitle = 	 {Proceedings of The 25th International Conference on Artificial Intelligence and Statistics},
    pages = 	 {2552--2573},
    year = 	 {2022},
    editor = 	 {Camps-Valls, Gustau and Ruiz, Francisco J. R. and Valera, Isabel},
    volume = 	 {151},
    series = 	 {Proceedings of Machine Learning Research},
    month = 	 {28--30 Mar},
    publisher =    {PMLR},
    url = 	 {https://proceedings.mlr.press/v151/choi22a.html}
}

@article{duong2023normflows_for_conditional_independence_testing,
    author = {Duong, Bao and Nguyen, Thin},
    year = {2023},
    month = {08},
    pages = {},
    title = {Normalizing flows for conditional independence testing},
    volume = {66},
    journal = {Knowledge and Information Systems},
    doi = {10.1007/s10115-023-01964-w}
}

@inproceedings{czyz2023beyond_normal,
    title={Beyond Normal: On the Evaluation of Mutual Information Estimators},
    author={Pawe{\l} Czy{\.z} and Frederic Grabowski and Julia E Vogt and Niko Beerenwinkel and Alexander Marx},
    booktitle={Thirty-seventh Conference on Neural Information Processing Systems},
    year={2023},
    url={https://openreview.net/forum?id=25vRtG56YH}
}

@article{duong2023dine,
    title={Diffeomorphic Information Neural Estimation},
    volume={37},
    url={https://ojs.aaai.org/index.php/AAAI/article/view/25908},
    DOI={10.1609/aaai.v37i6.25908},
    abstractNote={Mutual Information (MI) and Conditional Mutual Information (CMI) are multi-purpose tools from information theory that are able to naturally measure the statistical dependencies between random variables, thus they are usually of central interest in several statistical and machine learning tasks, such as conditional independence testing and representation learning. However, estimating CMI, or even MI, is infamously challenging due the intractable formulation. In this study, we introduce DINE (Diffeomorphic Information Neural Estimator)–a novel approach for estimating CMI of continuous random variables, inspired by the invariance of CMI over diffeomorphic maps. We show that the variables of interest can be replaced with appropriate surrogates that follow simpler distributions, allowing the CMI to be efficiently evaluated via analytical solutions. Additionally, we demonstrate the quality of the proposed estimator in comparison with state-of-the-arts in three important tasks, including estimating MI, CMI, as well as its application in conditional independence testing. The empirical evaluations show that DINE consistently outperforms competitors in all tasks and is able to adapt very well to complex and high-dimensional relationships.},
    number={6},
    journal={Proceedings of the AAAI Conference on Artificial Intelligence},
    author={Duong, Bao and Nguyen, Thin},
    year={2023},
    month={Jun.},
    pages={7468-7475}
}

@article{lee2024benchmark,
  title={A Benchmark Suite for Evaluating Neural Mutual Information Estimators on Unstructured Datasets},
  author={Lee, Kyungeun and Rhee, Wonjong},
  journal={arXiv preprint arXiv:2410.10924},
  year={2024}
}

@inproceedings{franzese2024minde,
    title={{MINDE}: Mutual Information Neural Diffusion Estimation},
    author={Giulio Franzese and Mustapha BOUNOUA and Pietro Michiardi},
    booktitle={The Twelfth International Conference on Learning Representations},
    year={2024},
    url={https://openreview.net/forum?id=0kWd8SJq8d}
}

@inproceedings{letizia2024fDIME,
    title={Mutual Information Estimation via \$f\$-Divergence and Data Derangements},
    author={Nunzio Alexandro Letizia and Nicola Novello and Andrea M Tonello},
    booktitle={The Thirty-eighth Annual Conference on Neural Information Processing Systems},
    year={2024},
    url={https://openreview.net/forum?id=PThi9hf9UT}
}

@inproceedings{butakov2024lossy_compression,
    title={Information Bottleneck Analysis of Deep Neural Networks via Lossy Compression},
    author={Ivan Butakov and Alexander Tolmachev and Sofia Malanchuk and Anna Neopryatnaya and Alexey Frolov and Kirill Andreev},
    booktitle={The Twelfth International Conference on Learning Representations},
    year={2024},
    url={https://openreview.net/forum?id=huGECz8dPp}
}

@inproceedings{butakov2024normflows,
    title={Mutual Information Estimation via Normalizing Flows},
    author={Ivan Butakov and Alexander Tolmachev and Sofia Malanchuk and Anna Neopryatnaya and Alexey Frolov},
    booktitle={The Thirty-eighth Annual Conference on Neural Information Processing Systems},
    year={2024},
    url={https://openreview.net/forum?id=JiQXsLvDls}
}

@book{polyanskiy2024information_theory,
    title={Information Theory: From Coding to Learning},
    author={Polyanskiy, Y. and Wu, Y.},
    isbn={9781108832908},
    url={https://books.google.ru/books?id=CySo0AEACAAJ},
    year={2024},
    publisher={Cambridge University Press}
}

@Article{huang2024MI_feature_selection_TS,
    AUTHOR = {Huang, Lin and Zhou, Xingqiang and Shi, Lianhui and Gong, Li},
    TITLE = {Time Series Feature Selection Method Based on Mutual Information},
    JOURNAL = {Applied Sciences},
    VOLUME = {14},
    YEAR = {2024},
    NUMBER = {5},
    ARTICLE-NUMBER = {1960},
    URL = {https://www.mdpi.com/2076-3417/14/5/1960},
    ISSN = {2076-3417},
    DOI = {10.3390/app14051960}
}

@inproceedings{gowri2024LMI,
    author = {Gowri, Gokul and Lun, Xiao-Kang and Klein, Allon M. and Yin, Peng},
    booktitle = {Advances in Neural Information Processing Systems},
    editor = {A. Globerson and L. Mackey and D. Belgrave and A. Fan and U. Paquet and J. Tomczak and C. Zhang},
    pages = {132843--132875},
    publisher = {Curran Associates, Inc.},
    title = {Approximating mutual information of high-dimensional variables using learned representations},
    url = {https://proceedings.neurips.cc/paper_files/paper/2024/file/efa9e3ed9c411aeb53d7bfe38a19d884-Paper-Conference.pdf},
    volume = {37},
    year = {2024}
}

@misc{ni2025neural_DoE,
    title={A Neural Difference-of-Entropies Estimator for Mutual Information}, 
    author={Haoran Ni and Martin Lotz},
    year={2025},
    eprint={2502.13085},
    archivePrefix={arXiv},
    primaryClass={stat.ML},
    url={https://arxiv.org/abs/2502.13085}, 
}

@article{czyz2025PMI,
    title={On the Properties and Estimation of Pointwise Mutual Information Profiles},
    author={Pawe{\l} Czy{\.z} and Frederic Grabowski and Julia E Vogt and Niko Beerenwinkel and Alexander Marx},
    journal={Transactions on Machine Learning Research},
    issn={2835-8856},
    year={2025},
    url={https://openreview.net/forum?id=LdflD41Gn8},
    note={}
}

@inproceedings{wang2025PMI_alignment,
    title={Information Theoretic Text-to-Image Alignment},
    author={Chao Wang and Giulio Franzese and Alessandro Finamore and Massimo Gallo and Pietro Michiardi},
    booktitle={The Thirteenth International Conference on Learning Representations},
    year={2025},
    url={https://openreview.net/forum?id=Ugs2W5XFFo}
}

@inproceedings{wang2025RFMI,
    title={{RFMI}: Estimating Mutual Information on Rectified Flow for Text-to-Image Alignment},
    author={Chao Wang and Giulio Franzese and Alessandro Finamore and Pietro Michiardi},
    booktitle={ICLR 2025 Workshop on Deep Generative Model in Machine Learning: Theory, Principle and Efficacy},
    year={2025},
    url={https://openreview.net/forum?id=8CogNH4vBp}
}

@inproceedings{dahlke2025flowbased,
    title={Flow-based Variational Mutual Information: Fast and Flexible Approximations},
    author={Caleb Dahlke and Jason Pacheco},
    booktitle={The Thirteenth International Conference on Learning Representations},
    year={2025},
    url={https://openreview.net/forum?id=spDUv05cEq}
}

@misc{kholkin2025InfoBridge,
    title={InfoBridge: Mutual Information estimation via Bridge Matching}, 
    author={Sergei Kholkin and Ivan Butakov and Evgeny Burnaev and Nikita Gushchin and Alexander Korotin},
    year={2025},
    eprint={2502.01383},
    archivePrefix={arXiv},
    primaryClass={cs.LG},
    url={https://arxiv.org/abs/2502.01383}, 
}

@Article{franzese2025latent_abstractions,
    AUTHOR = {Franzese, Giulio and Martini, Mattia and Corallo, Giulio and Papotti, Paolo and Michiardi, Pietro},
    TITLE = {Latent Abstractions in Generative Diffusion Models},
    JOURNAL = {Entropy},
    VOLUME = {27},
    YEAR = {2025},
    NUMBER = {4},
    ARTICLE-NUMBER = {371},
    URL = {https://www.mdpi.com/1099-4300/27/4/371},
    PubMedID = {40282606},
    ISSN = {1099-4300},
    DOI = {10.3390/e27040371}
}

@inproceedings{chen2025vector_copula,
    title={Neural Mutual Information Estimation with Vector Copulas},
    author={Yanzhi Chen and Zijing Ou and Adrian Weller and Michael Gutmann},
    booktitle={The Thirty-ningth Annual Conference on Neural Information Processing Systems},
    year={2025}
}

@misc{ronneberger2015unetconvolutionalnetworksbiomedical,
      title={U-Net: Convolutional Networks for Biomedical Image Segmentation}, 
      author={Olaf Ronneberger and Philipp Fischer and Thomas Brox},
      year={2015},
      eprint={1505.04597},
      archivePrefix={arXiv},
      primaryClass={cs.CV},
      url={https://arxiv.org/abs/1505.04597}, 
}

@misc{von-platen-etal-2022-diffusers,
  author = {Patrick von Platen and Suraj Patil and Anton Lozhkov and Pedro Cuenca and Nathan Lambert and Kashif Rasul and Mishig Davaadorj and Dhruv Nair and Sayak Paul and William Berman and Yiyi Xu and Steven Liu and Thomas Wolf},
  title = {Diffusers: State-of-the-art diffusion models},
  year = {2022},
  publisher = {GitHub},
  journal = {GitHub repository},
  howpublished = {\url{https://github.com/huggingface/diffusers}}
}

@article{lipman2022flow,
  title={Flow matching for generative modeling},
  author={Lipman, Yaron and Chen, Ricky TQ and Ben-Hamu, Heli and Nickel, Maximilian and Le, Matt},
  journal={arXiv preprint arXiv:2210.02747},
  year={2022}
}

@article{chen2018neural,
  title={Neural ordinary differential equations},
  author={Chen, Ricky TQ and Rubanova, Yulia and Bettencourt, Jesse and Duvenaud, David K},
  journal={Advances in neural information processing systems},
  volume={31},
  year={2018}
}

@article{scikit-learn,
  title={Scikit-learn: Machine Learning in {P}ython},
  author={Pedregosa, F. and Varoquaux, G. and Gramfort, A. and Michel, V.
          and Thirion, B. and Grisel, O. and Blondel, M. and Prettenhofer, P.
          and Weiss, R. and Dubourg, V. and Vanderplas, J. and Passos, A. and
          Cournapeau, D. and Brucher, M. and Perrot, M. and Duchesnay, E.},
  journal={Journal of Machine Learning Research},
  volume={12},
  pages={2825--2830},
  year={2011}
}
\bibliographystyle{IEEEtran}
\newpage
\appendix
% 

%\section{Complete proofs}
%\label{appendix:proofs}

% 1. Optional: Add a title for your Appendix ToC
\section*{Table of Contents} 

% 2. Start tracking the sections for this specific ToC
\startcontents[appendices]

% 3. Print the ToC (The '1' means start tracking at the Section level)
\printcontents[appendices]{l}{1}{\setcounter{tocdepth}{2}}

\section{Distributions}
\label{appendix:distributions}

\subsection{Technical details for discrete distributions}

\textbf{Noiseless Channel (Entropy-Calibrated Construction).}
Let $X$ be a discrete random variable supported on a finite alphabet
$\mathcal A=\{1,\dots,K\}$ with probability mass function $p_X$.
The channel is deterministic:
\[
    Y = X.
\]
Hence
\[
    p_{X,Y}(i,j)=p_X(i) \cdot \mathbf I\{i=j\}, \qquad \sent(Y|X)=0,
\]
and the \gls{MI} satisfies
\[
    \MI(X;Y)=\sent(X).
\]

%\medskip
%\noindent
%\textbf{Information calibration.}
To construct a benchmark instance with prescribed
\(
    \MI(X;Y)=h \ge 0,
\)
it suffices to choose $p_X$ such that $\sent(X)=\varkappa$.
We employ a maximum-symmetry parameterization:
let $K=\lceil e^{\varkappa}\rceil$ and define
\[
    p_X(i)=\frac{1-r}{K-1},
    \quad i=1,\dots,K-1,
    \qquad
    p_X(K)=r,
\]
where $r\in[0,1/K]$ solves the following equation:
\[
    - r \log r - (1-r)\log\!\left(\frac{1-r}{K-1}\right) = \varkappa.
\]
The scalar equation admits a unique solution in $[0,1/K]$
and is solved numerically.
Thus, in the noiseless channel, the target \gls{MI}
is achieved exactly by entropy calibration of the input distribution.

\textbf{Symmetric Noisy Channel (Noise-Calibrated Construction).}
Let $X$ be a discrete random variable over
$\mathcal A=\{1,\dots,K\}$ with probability mass function $p_X$.
Fix a reroll probability $\varepsilon\in[0,1]$.
The channel output $Y$ is generated according to
\[
    p_{Y|X}(j|i) =
    (1-\varepsilon) \cdot \mathbf I\{j=i\} + \frac{\varepsilon}{K}.
\]

An actual symbol error occurs only if rerolling produces
a symbol different from $X$, hence
\[
    p_{\mathrm{err}} = \varepsilon\!\left(1-\frac{1}{K}\right).
\]

The output marginal equals
\[
    p_Y(j) =
    (1-\varepsilon)p_X(j)  + \frac{\varepsilon}{K},
\]
and the \gls{MI} can be expressed explicitly as
\[
    \MI(X;Y;\varepsilon)
    =
    \sum_{i=1}^{K}\sum_{j=1}^{K}
    p_X(i)\,p_{Y|X}(j|i)\,
    \log \frac{p_{Y|X}(j|i)}{p_Y(j)}.
\]

\medskip
\noindent
\textbf{Information calibration.}
For fixed alphabet size $K$ and fixed input distribution $p_X$,
the mapping $\varepsilon \mapsto \MI(X;Y;\varepsilon)$ is continuous and strictly decreasing on $[0,1]$, with
\[
    \MI(X;Y;0)=\sent(X),
    \qquad
    \MI(X;Y;1)=0.
\]
Therefore, for any target value
\(
    0 \le \varkappa \le \sent(X),
\)
there exists a unique $\varepsilon$ such that
\[
    \MI(X;Y;\varepsilon)=\varkappa.
\]
In practice, as implemented in the code, $\varepsilon$
is found numerically via one-dimensional root-finding, which ensures that generated
samples $(X,Y)$ achieve the prescribed \gls{MI} $\varkappa$.

\medskip
\noindent
\textbf{Uniform-input case.}
When $p_X(i)=1/K \quad \forall i \in \mathcal{A}$, the output is also uniform and
\[
    \MI(X;Y) =  \log K + s\log s +
    (1-s)\log\!\left(\frac{\varepsilon}{K}\right),
    \qquad
    s=1-\varepsilon\!\left(1-\frac{1}{K}\right).
\]

\subsection{Technical details for continuous distributions}

\textbf{Correlated Normal \& Uniform.}
For jointly Gaussian $(X,Y) \sim \normal(m,\Sigma)$, the following holds~\citep{polyanskiy2024information_theory}:
\[
    \MI(X;Y) = \frac{1}{2} \log \det \Sigma_X + \frac{1}{2} \log \det \Sigma_Y - \frac{1}{2} \log \det \Sigma,
    \qquad
    \textnormal{where} \quad \Sigma = \begin{bmatrix} \Sigma_X & \Sigma_{XY} \\ \Sigma_{XY} & \Sigma_Y \end{bmatrix}
\]
Given $\{\rho_i\}$ are singular values of $\Sigma_X^{-1/2} \Sigma_{XY} \Sigma_Y^{-1/2}$,
this expression further simplifies to
\[
    \MI(X;Y) = -\frac{1}{2} \sum_{i} \log (1 - \rho_i^2)
\]
Finally, if $\Sigma_X$ and $\Sigma_Y$ are unit, applying Gaussian CDF to $X$ and $Y$ component-wise yields
a Gaussian copula distribution, which we refer to as the \emph{correlated uniform} distribution.

\textbf{Correlated Student.}
Consider an $(n+m)$-dimensional pair $(\tilde X; \tilde Y) \sim \normal(0, \Sigma)$.
Define $X = \tilde X \sqrt{k / Z}$ and $Y = \tilde Y \sqrt{k / Z}$, where $Z \sim \chi^2(k)$.
Then $X$ and $Y$ jointly follow a multivariate Student-$t$ distribution with $k$ degrees of freedom.
\citep{arellano_valle2013MI_for_skew_distributions} showed that
\[
    \MI(X;Y) = \MI(\tilde X; \tilde Y) + c(k,n,m),
\]
where
\[
    c(k, n, m) = f(k) + f(k + n + m) - f(k + n) - f(k + m),
    \qquad
    f(x) = \log \Gamma \left( \frac{x}{2} \right) - \frac{x}{2} \psi \left( \frac{x}{2} \right),
\]
and $\psi$ is the digamma function.

Note that even when $\tilde X$ and $\tilde Y$ are independent, $\MI(X;Y) \neq 0$,
because $X$ and $Y$ share information through the common scaling factor involving $U$.

\textbf{Smoothed Uniform.}

\begin{lemma}
    \label{lemma:closed_form_MI_expression_smoothed_uniform}
    Consider independent $X \sim \uniform[0;1]$, $Z \sim \uniform[-\varepsilon; \varepsilon]$ and $Y = X + Z$.
    Then
    \begin{equation}
        \label{eq:smoothed_uniform_mutual_information}
        \MI(X; Y) =
        \begin{dcases}
            \varepsilon - \log (2 \varepsilon), &\;\; \varepsilon < 1/2 \\
            (4 \varepsilon)^{-1}, &\;\; \varepsilon \geq 1/2
        \end{dcases}
    \end{equation}
\end{lemma}

\begin{proof}
    Probability density function of $ Y $ (two cases):
    \[
        (\varepsilon < 1/2): \qquad
        \PDF_Y(y) = (\PDF_X * \PDF_Z)(y) =
        \begin{cases}
            0, &\; y < -\varepsilon \vee y \geq 1 + \varepsilon \\
            \frac{y + \varepsilon}{2\varepsilon}, & \; -\varepsilon \leq y < \varepsilon \\
            1, & \; \varepsilon \leq y < 1 - \varepsilon \\
            \frac{1 + \varepsilon - y}{2\varepsilon}, & \; 1 - \varepsilon \leq y < 1 + \varepsilon \\
        \end{cases}
    \]
    \[
        (\varepsilon \geq 1/2): \qquad
        \PDF_Y(y) = (\PDF_X * \PDF_Z)(y) =
        \begin{cases}
            0, &\; y < -\varepsilon \vee y \geq 1 + \varepsilon \\
            \frac{y + \varepsilon}{2\varepsilon}, & \; -\varepsilon \leq y < 1 - \varepsilon \\
            \frac{1}{2\varepsilon}, & \; 1 - \varepsilon \leq y < \varepsilon \\
            \frac{1 + \varepsilon - y}{2\varepsilon}, & \; \varepsilon \leq y < 1 + \varepsilon \\
        \end{cases}
    \]
    Differential entropy of a uniformly distributed random variable:
    \[
        \dent(\uniform[a;b]) = \log(b - a)
    \]
    Conditional differential entropy of $ Y $ with respect to $ X $:
    \[
        \dent(Y \mid X) = \expect_{x \sim X} \dent(Y \mid X = x) = \expect_{x \sim X} \dent(Z + x \mid X = x)
    \]
    As $ X $ and $ Z $ are independent,
    \begin{equation}
        \label{eq:smoothed_uniform_Y_conditional_entropy}
        \expect_{x \sim X} \dent(Z + x \mid X = x)
        = \expect_{x \sim X} \dent(Z + x)
        = \int\limits_0^1 \log (2 \varepsilon) \, \dif x
        = \log (2 \varepsilon)
    \end{equation}
    Differential entropy of $ Y $:
    \begin{equation}
        \label{eq:smoothed_uniform_Y_entropy}
        \dent(Y) = - \int\limits_{-\infty}^\infty \PDF_Y(y) \, \dif y =
        \begin{cases}
            \varepsilon, &\;\; \varepsilon < 1/2 \\
            (4\varepsilon)^{-1} + \log(2 \varepsilon), &\;\; \varepsilon \geq 1/2
        \end{cases}
    \end{equation}
    The final result is acquired via substituting~\eqref{eq:smoothed_uniform_Y_conditional_entropy} and~\eqref{eq:smoothed_uniform_Y_entropy} into $\MI(X;Y) = \dent(Y) - \dent(Y \mid X)$.
\end{proof}

Equation~\eqref{eq:smoothed_uniform_mutual_information} can be inverted:
\begin{equation}
    \label{eq:smoothed_uniform_mutual_information_inverted}
    \varepsilon =
    \begin{cases}
        (4 \cdot \MI(X;Y))^{-1}, &\;\; \MI(X;Y) < 1/2 \\
        -W\left[ -\frac{1}{2}\exp(-\MI(X;Y))\right], &\;\; \MI(X;Y) \geq 1/2
    \end{cases},
\end{equation}
where $ W $ is the product logarithm function.

\textbf{Log-Gamma-Exponential.}
\begin{lemma}
\label{lemma:closed_form_MI_expression_log_gamma}
Let \(\digamma\) be the digamma function, and let \(\loggamma(\theta)\) be the log-gamma distribution with probability density function

\[
f(\theta;x) = \frac{\textnormal{exp}(\theta x-\textnormal{exp}(x))}{\Gamma(\theta)}
\]

with shape parameter \(\theta\). Consider independent \(X \sim \loggamma(\theta)+\digamma(\theta)\), \(Z \sim \loggamma(1)\) and \(Y = Z - X\). The exponentiated variables \(e^X\) and \(e^Y\) follow the joint probability density function \(\PDF(t,s) = \frac{1}{\Gamma(\theta)} t^\theta e^{-t-ts}\). The \gls{MI} is given by:
\[
\MI(X;Y) = \digamma(\theta + 1) - \log(\theta).
\]
    
\end{lemma}
\begin{proof}
Recall the differential entropy of the log-gamma distribution is

\[
\dent(X) = \expect[-\log f(\theta;x)] = -\expect[\theta x - \textnormal{exp}(x)-\log(\Gamma(\theta))]
\]

By solving the expectations we get

\begin{equation}\label{eq:loggamma_ent}
    \dent(X) = -[\theta\digamma(\theta)-\theta-\log(\Gamma(\theta))] = \theta+\log(\Gamma(\theta))-\theta\digamma(\theta).
\end{equation}

Now, by standard \gls{MI} equalities we get

\[
\MI(X;Y) = \dent(Y)-\dent(Y|X) = \dent(Y) - \dent(Z-X|X) = \dent(Y) - \dent(Z).
\]

To find \(\dent(Y)\), we note that differential entropy is invariant to translations. Let \(X_0 \sim \loggamma(\theta)\) such that \(X = X_0 + \digamma(\theta)\). We define \(Y_0 = Z - X_0\), meaning \(\dent(Y) = \dent(Y_0)\). Exponentiating the variables maps them to the standard Gamma family. Let \(U = \textnormal{exp}(Z) \sim \textnormal{Gamma}(1, 1)\) and \(V = \textnormal{exp}(X_0) \sim \textnormal{Gamma}(\theta, 1)\). The ratio \(W = \textnormal{exp}(Y_0) = \nicefrac{U}{V}\) follows a Beta Prime distribution with parameters \((1, \theta)\). By applying the change of variables \(y_0 = \log(w)\), the \textnormal{PDF} of \(Y_0\) is the generalized logistic distribution:

\[
f_{Y_0}(y_0) = \theta\textnormal{exp}(y_0)(1+\textnormal{exp}(y_0))^{-(\theta+1)}
\]

The entropy \(\dent(Y) = \dent(Y_0) = \expect[-\log f_{Y_0}(Y_0)]\) expands to:

\[
\dent(Y) = -\log(\theta) - \expect[Y_0] + (\theta+1)\expect[\log(1+\textnormal{exp}(Y_0))]
\]

By linearity, \(\expect[Y_0] = \expect[Z] - \expect[X_0] = \digamma(1) - \digamma(\theta)\). Using the substitution \(u = 1+\textnormal{exp}(y_0)\), the integral for the second expectation evaluates exactly to \(\nicefrac{1}{\theta^2}\). Thus:

\[
\dent(Y) = -\log(\theta) - (\digamma(1) - \digamma(\theta)) + (\theta+1)\left(\frac{1}{\theta}\right) = -\log(\theta) - \digamma(1) + \digamma(\theta) + 1 + \frac{1}{\theta}
\]

Finally, we substitute \(\dent(Y)\) and \(\dent(Z) = 1 - \digamma(1)\) back into the \gls{MI} equation:

\[
\MI(X;Y) = \left(-\log(\theta) - \digamma(1) + \digamma(\theta) + 1 + \frac{1}{\theta}\right) - (1 - \digamma(1))
\]

\[
\MI(X;Y) = \digamma(\theta) + \frac{1}{\theta} - \log(\theta)
\]

Using the fundamental recurrence relation of the digamma function, \(\digamma(\theta + 1) = \digamma(\theta) + \frac{1}{\theta}\), this simplifies perfectly to our final result:

\[
\MI(X;Y) = \digamma(\theta + 1) - \log(\theta).
\]

\end{proof}

\textbf{Rare Event Channel.}
Let \(Z\) be a discrete source random variable with entropy \(\sent(Z) = \varkappa\) and probability mass function

\begin{align*}
    q_k =
    \begin{cases}
        q & \text{ if } k=1\\
        \alpha & \text{ if } 1 < k \leq N
    \end{cases}
\end{align*}

where \(N=\lceil \text{exp}(h)\rceil\), \(\alpha = \frac{1-q}{N-1}\), and \(q\) solves the entropy equation:

\[
    q\log q + (1-q)\log \frac{1-q}{N-1}=-\varkappa.
\]

The rare event channel models transmission errors by passing the source \(Z\) through a state-dependent channel. The state is governed by a Bernoulli random variable \(S \sim \textnormal{Bernoulli}(p)\), where \(p\) represents the failure probability. When the channel succeeds (\(S=1\)), the intermediate discrete variables perfectly copy the source, meaning \(X' = Y' = Z\). When the channel fails (\(S=0\)), both collapse to a predetermined, fixed symbol.

The final continuous random variables \(X\) and \(Y\) are obtained by dequantizing \(X'\) and \(Y'\). Because dequantization adds independent uniform noise inside the discrete bins, \(X\) and \(Y\) share the exact same hypercube during a success, but become completely independent uniform noise within the fixed failure bin during a failure. More formally, the generative process obeys the following Bayesian network:

\begin{center}
\begin{tikzpicture}[
    rv/.style={circle, draw=black, thick, minimum size=8mm, inner sep=2pt},
    ->, >=stealth, thick
]
    % Nodes
    \node[rv] (Z)  at (1.5, 2.5) {\(Z\)};
    \node[rv] (S)  at (1.5, 1)   {\(S\)};
    \node[rv] (Xp) at (0, 1)     {\(X'\)};
    \node[rv] (Yp) at (3, 1)     {\(Y'\)};
    \node[rv] (X)  at (0, -0.5)  {\(X\)};
    \node[rv] (Y)  at (3, -0.5)  {\(Y\)};

    % Edges
    \draw (Z) -- (Xp);
    \draw (Z) -- (Yp);
    \draw (S) -- (Xp);
    \draw (S) -- (Yp);
    \draw (Xp) -- (X);
    \draw (Yp) -- (Y);
\end{tikzpicture}
\end{center}

We can compute \(\MI(X';Y')\) directly by observing the relationship between the discrete variables. Because the channel forces \(X' = Y'\) in all states (either both output the fixed failure symbol or both copy \(Z\)), they are perfectly identical random variables. The \gls{MI} of a variable with itself is simply its entropy:

\[
    \MI(X';Y') = \sent(X').
\]

To compute \(\sent(X')\), we use the chain rule for entropy with the state variable \(S\):

\[
    \sent(X', S) = \sent(X') + \sent(S \mid X') = \sent(S) + \sent(X' \mid S).
\]

Assuming the fixed failure symbol is an out-of-band symbol distinct from the support of \(Z\), knowing the value of \(X'\) perfectly identifies whether the channel failed or succeeded. Thus, \(\sent(S \mid X') = 0\), which simplifies the equation to:

\begin{align*}
    \sent(X') &= \sent(S) + \sent(X' \mid S) \\
              &= \sent(S) + P(S=0)\sent(X' \mid S=0) + P(S=1)\sent(X' \mid S=1).
\end{align*}

We know \(S \sim \bernoulli(p)\), so \(\sent(S) = h(p)\). When the channel fails (\(S=0\) with probability \(p\)), \(X'\) becomes a deterministic value. The entropy of a constant is zero, meaning \(\sent(X' \mid S=0) = 0\). When the channel succeeds (\(S=1\) with probability \(1-p\)), \(X'\) perfectly copies the source \(Z\), meaning \(\sent(X' \mid S=1) = \sent(Z)\). Substituting these values yields:

\[
    \sent(X') = h(p) + p \cdot 0 + (1-p)\sent(Z).
\]

Finally, because the dequantized continuous variables \(X\) and \(Y\) share no additional information beyond their discrete bins, their \gls{MI} is strictly bounded by the discrete variables. This yields our final expression:

\[
    \MI(X;Y) = \MI(X';Y') = h(p) + (1 - p)\sent(Z).
\]

\subsection{Technical details for mixtures distributions}

% \textbf{Quantized.}

% \textbf{Disjoint supports.}

% \begin{lemma}
%     Consider a mixture of distributions $\jproba{X}{Y} = \sum_{k=1}^n \PMF_k \cdot \jproba{X_k}{Y_k}$ such that supports of $\{X_k\}_{k=1}^n$ and $\{Y_k\}_{k=1}^n$ are disjoint and $\MI(X_k\;Y_k)$ exists for every $k$.
%     Then
%     \[
%         \MI(X;Y) = \sum_{k=1}^n \PMF_k \cdot \MI(X_k;Y_k) + \sent(\PMF)
%     \]
% \end{lemma}

Despite extensive benchmarking of \gls{MI} estimators on synthetic distributions
(e.g., multivariate normal, uniform, Student's t) and 2D image datasets,
their performance on discrete-continuous mixture distributions remains less explored.
To address this, we extend our benchmark to include such mixtures.
First, we incorporate the canonical examples proposed by the authors of a modified KSG estimator~\citep{gao2017discrete_KSG}.
Additionally, we \textbf{generalize these cases} to form a more comprehensive test suite.

\begin{lemma}
Let $(X,Y)$ be a random pair with joint distribution
\[
p_{X,Y}(x,y) = \sum_{k=1}^n \pi_k \, p_{X_k,Y_k}(x,y),
\]
where $\pi_k \ge 0$ and $\sum_{k=1}^n \pi_k = 1$. Assume that:
(i) the supports of $\{X_k\}_{k=1}^n$ are pairwise disjoint;
(ii) the supports of $\{Y_k\}_{k=1}^n$ are pairwise disjoint; and
(iii) $\MI(X_k;Y_k)$ is finite for all $k$.
Then,
\[
\MI(X;Y)
=
\sum_{k=1}^n \pi_k \, \MI(X_k;Y_k)
+
\sent(\pi),
\]
where $\sent(\pi) = -\sum_{k=1}^n \pi_k \log \pi_k$.
\end{lemma}

\begin{proof}
Let $K$ be a discrete random variable taking values in $\{1,\dots,n\}$ with $\Pr(K=k)=\pi_k$, such that
\[
(X,Y)\mid (K=k) \sim (X_k,Y_k).
\]
By construction,
\[
p_{X,Y}(x,y) = \sum_{k=1}^n \pi_k p_{X_k,Y_k}(x,y).
\]

Since the supports of $\{X_k\}$ are disjoint, $K$ is a deterministic function of $X$, and similarly, since the supports of $\{Y_k\}$ are disjoint, $K$ is a deterministic function of $Y$. Therefore,
\[
\sent(K \mid X) = \sent(K \mid Y) = 0.
\]

Using the chain rule for \gls{MI},
\[
\MI(X;Y)
= \MI(X;Y,K) - \MI(X;K \mid Y).
\]
Because $\sent(K \mid Y)=0$, we have $\MI(X;K \mid Y)=0$, and hence
\[
\MI(X;Y) = \MI(X;Y,K).
\]
Applying the chain rule again yields
\[
\MI(X;Y,K) = \MI(X;K) + \MI(X;Y \mid K).
\]
Since $\sent(K \mid X)=0$, it follows that $\MI(X;K)=\sent(K)$, and thus
\[
\MI(X;Y) = \sent(K) + \MI(X;Y \mid K).
\]

Finally, by the definition of conditional \gls{MI},
\[
\MI(X;Y \mid K) =
\sum_{k=1}^n \pi_k \, \MI(X;Y \mid K=k)
=
\sum_{k=1}^n \pi_k \, \MI(X_k;Y_k).
\]
Substituting back completes the proof.

\end{proof}

\begin{corollary}[Experiment I in~\citep{gao2017discrete_KSG}]
    Let
    \[
        X,Y \sim
        \begin{cases}
            \normal\left( 0, \begin{bmatrix} 1 & \alpha \\ \alpha & 1 \end{bmatrix} \right) &\quad \textnormal{with probability $p$} \\
            \begin{cases}
                (-1, -1), &\quad \beta/2 \\
                (1, 1),   &\quad \beta/2 \\
                (-1, 1),  &\quad (1-\beta)/2 \\
                (1, -1),  &\quad (1-\beta)/2 \\
            \end{cases} &\quad \textnormal{with probability $1-p$} \\
        \end{cases}
    \]
    Then
    \[
        \MI (X; Y) = p \cdot \left(-\frac{1}{2} \log(1 - \alpha^2) \right) + (1 - p) \cdot \left( \log 2 + q \log q + (1 - q) \log(1 - q) \right) - \left(p \log p + (1 - p) \log (1 - p) \right)
    \]
\end{corollary}

\begin{proof}
According to the Lemma, this distribution is a mixture of two distributions $p_{X_1,Y_1}$ and $p_{X_2,Y_2}$ with weights $\pi_1 = p$ and $\pi_2 = 1 - p$ respectively:

\begin{itemize}
    \item $p_{X_1,Y_1} \sim  \normal\left( 0, \begin{bmatrix} 1 & \alpha \\ \alpha & 1 \end{bmatrix} \right)$

    \item $p_{X_2,Y_2} \sim  \begin{cases}
                (-1, -1), &\quad \beta/2 \\
                (1, 1),   &\quad \beta/2 \\
                (-1, 1),  &\quad (1-\beta)/2 \\
                (1, -1),  &\quad (1-\beta)/2 \\
            \end{cases} $
\end{itemize}

Then,  $\MI(X_1, Y_1) = -\frac{1}{2} \log(1 - \alpha^2)$ according to the formula for multivariate normal case.

The next term $\MI(X_2, Y_2)$ can be calculated via three entropy terms for discrete distributions. Based on the fact that $P(X_2 = \pm 1) = P(Y_2 = \pm 1) = 1/2$, we obtain $\sent(X_2) = \sent(Y_2) = -\frac{1}{2} \log(\frac{1}{2}) \cdot 2 = \log 2$. The entropy of joint discrete distribution equals 

$$\sent(X_2, Y_2) = \left[ -\frac{\beta}{2} \log \left( \frac{\beta}{2} \right) - \frac{1-\beta}{2} \log \left( \frac{1-\beta}{2} \right) \right] \cdot 2 = -\beta \log \beta - (1 - \beta) \log (1 - \beta) + \log 2.$$

Then, 

\begin{align}
\MI(X_2, Y_2) &= \sent(X_2) + \sent(Y_2) - \sent(X_2, Y_2) = \log 2 + \log 2 - \left(-\beta \log \beta - (1 - \beta) \log (1 - \beta) + \log 2 \right) \nonumber \\
&= \log 2 + \beta \log \beta - (1 - \beta) \log (1 - \beta). \nonumber
\end{align}

Using the fact that $\sent(\pi) = -\pi_1 \log \pi_1 - \pi_2 \log \pi_2$ and the previous lemma we can finalize the proof:

\begin{align}
    \MI (X; Y) &= \pi_1 \cdot \MI(X_1, Y_1) + \pi_2 \cdot \MI(X_2, Y_2) + \dent(\pi) = p \cdot \left(-\frac{1}{2} \log(1 - \alpha^2) \right) + \nonumber \\
    &+ (1 - p) \cdot \left( \log 2 + \beta \log \beta + (1 - \beta) \log(1 - \beta) \right) - \left(p \log p + (1 - p) \log (1 - p) \right) \nonumber 
\end{align}

\end{proof}

\begin{lemma}[Experiment II in~\citep{gao2017discrete_KSG}]
    \label{lemma:smoothed-discrete-uniform}
    Let $(X,Y)$ be a random pair of a discrete random variable $X \sim \text{Uniform}\{0, 1, \dots, m-1\}$ and a continuous variable $Y \sim \text{Uniform}[X, X+n]$, where $m, n \in \mathbb{N}$ and $m \ge 2, m \ge n - 1$. Then, The \gls{MI} $\MI(X; Y)$ can be expressed as
    
    \[
        \MI(X; Y) = \sent(Y) - \sent(Y \mid X) = \frac{2}{m n} \cdot \left( \frac{n(n - 1)}{2} \log(m n) - \sum\limits_{i = 1}^{n - 1} i\log i \right) + \frac{(m - n + 1)\log m}{m} - \log n
    \]
\end{lemma}

\begin{proof}
The \gls{MI} can be calculated by $\MI(X; Y) = \sent(Y) - \sent(Y|X)$, where the latter term is computed via the entropy of the conditional uniform distribution: $\sent(Y|X) = \log n$. Due to the variable $Y$ depending on the discrete variable $X$, the distribution of $Y$ is a continuous distribution, whose probability density function follows: 

\[
p_Y(x) = \begin{cases}
\frac{\lceil x \rceil}{mn} & 0 < x < n - 1, \ \\
\frac{1}{m} & n - 1 \leq x \leq m, \\
 \frac{\lfloor m + n - x \rfloor}{mn} & m < x \leq m + n - 1, \\
 0 & \text{otherwise.}
\end{cases}
\]

Hence, $\sent(Y)$ can be computed using the definition of entropy and the symmetry of $p_Y$: 

$\sent(Y) = -\mathbb{E}_{p_Y}\log p_Y(x) = -\int\limits_0^{m + n - 1} \log p_Y(x) p_Y(x)dx = $

$-\left[\sum\limits_{i = 1}^{n - 1} \frac{i}{mn} \log \frac{i}{mn} + (m - (n-1))\cdot \frac{1}{m} \log \frac{1}{m} + \sum\limits_{i = m + 1}^{m + n - 1} \frac{m + n - i}{mn} \log \frac{m + n - i}{mn} \right] = $

$-\left[2 \sum\limits_{i = 1}^{n - 1} \frac{i}{mn} \log \frac{i}{mn} + (m - n + 1) \cdot \frac{1}{m} \log \frac{1}{m} \right] = \frac{2}{mn} \cdot \left( \frac{n(n - 1)}{2} \log(mn) - \sum\limits_{i = 1}^{n - 1} i\log i \right) + \frac{(m - n + 1)\log m}{m}$

Thus, the \gls{MI} is equal to 

\[
    \MI(X; Y) = \sent(Y) - \sent(Y \mid X) = \frac{2}{m n} \cdot \left( \frac{n(n - 1)}{2} \log(m n) - \sum\limits_{i = 1}^{n - 1} i\log i \right) + \frac{(m - n + 1)\log m}{m} - \log n
\]

\end{proof}

The provided examples compute MI between one-dimensional random variables.
To construct an $n$-dimensional benchmark where the true MI is known, we generate $n$ independent and identically distributed pairs $(X_i, Y_i)$ and define the vectors $X = (X_1, X_2, \dots, X_n)$ and $Y = (Y_1, Y_2, \dots, Y_n)$.
Due to the independence across dimensions, the \gls{MI} becomes additive: $\MI(X, Y) = n \cdot \MI(X_i, Y_i)$.
This approach allows for scalable benchmarking of MI estimators on high-dimensional data with a known ground truth.

\subsection{Technical details for mapped distributions}

One can leverage the key property of \gls{MI} to construct distributions with known ground truth \gls{MI}
and arbitrary complex marginals.
Note that, as a corollary of the \gls{DPI} (\Cref{theorem:MI_invariance}),
$\MI(f(X);g(Y) = \MI(X;Y)$, where $f$ and $g$ are measurable and injective.
While previous works also employed this idea to construct complex tests~\citep{czyz2023beyond_normal,butakov2024lossy_compression},
we propose moving beyond handcrafted $f$ and $g$ and apply learned normalizing flows,
thus constructing increasingly complex distributions with prescribed dependency structure.

In our work, we use the following distributions to train $f$ and $g$: \emph{Circles}, \emph{Moons} and \emph{SCurve}.

\textbf{Circles.}
This distribution consists of two concentric circles in $\reals^2$.
It is generated by sampling points from a large outer circle and a smaller inner circle, where the distance between the two is controlled by a scaling factor. This creates a dataset with a spherical decision boundary, making it a standard test for non-linear separability.
A small Gaussian noise is added to avoid singular support.

\textbf{Moons.}
The Moons distribution consists of two interleaving half-circles in $\reals^2$. By offsetting one semi-circle relative to the other, it creates a non-linearly separable structure that is widely used to evaluate the performance of clustering algorithms and non-linear classifiers.
A small Gaussian noise is added to avoid singular support.

\textbf{SCurve.} The SCurve is constructed by sampling a parameter $t \sim \mathcal{U}(-1.5\pi, 1.5\pi)$ and a height $h \sim \mathcal{U}(0, 2)$. These are then mapped to 3D coordinates $(x, y, z)$ as follows:
\[ 
    \begin{aligned}
    x &= \sin(t) \\
    y &= h \\
    z &= \text{sgn}(t) \cdot (\cos(t) - 1)
    \end{aligned}
\]
The use of the signum function $\text{sgn}(t)$ ensures that the curve bends in opposite directions along the $z$-axis for positive and negative values of $t$, creating the characteristic ``S'' shape. The resulting distribution is a curved 2D ribbon embedded in 3D space, providing a rigorous benchmark for algorithms intended to unfold non-linear geometries.
A small Gaussian noise is added to avoid singular support.

\subsection{Technical details for coupled distributions}

To achieve realistic marginal distributions while maintaining tractable \gls{MI},
we extend the same-class-pairing approach proposed by~\citep{lee2024benchmark} to a general invertible stochastic coupling of the data
(which also includes continuous copulas).

\textbf{Same class sampling}  
As a motivating example, consider the case $Z_1 = Z_2 = Z$, where $Z$ is a discrete random variable with finite support of size $k$. Let $\{ \pi_i \}_{1 \leq i \leq k}$ be a collection of distributions with disjoint supports, each corresponding to a distinct class label $i$. We define a Markov kernel $f$ such that $f(Z = i) \sim \pi_i$, i.e., class label $i$ is mapped to a random sample drawn from distribution $\pi_i$.

Applying $f$ independently to two instances of $Z$, we obtain random variables $W_1 = f(Z_1)$ and $W_2 = f(Z_2)$. Substituting $f_1 = f_2 = f$ in the Markov chain formulation in~\Cref{eq:mc}, the data processing inequality yields:
\[
    \MI(W_1; W_2) = \sent(Z).
\]

This construction extends the method introduced by~\citep{lee2024benchmark}, where a uniformly distributed latent variable $Z$ is used to sample images or text from distinct classes.

\textbf{Regression}  
We extend the previous approach to the continuous setting. Let $(Z_1,Z_2)$ be a continuous pair of random variable with known \gls{MI} $\MI(Z_1;Z_2)$. Instead of a finite collection of distributions, we define a \emph{family of distributions} $\{ \pi_\theta \}_{\theta \in \Theta}$, parameterized by $\theta \in \Theta\). We use this family to construct a Markov kernel $f$ such that $f(Z = \theta) \sim \pi_\theta$. As in the discrete case, we assume the supports of the $\pi_\theta$ distributions are disjoint to ensure information is preserved.

By applying this kernel independently to $Z_1$ and $Z_2$,
we obtain complex random variables $W_1 = f(Z_1)$ and $W_2 = f(Z_2)$, for which the \gls{MI} satisfies:
\[
    \MI(W_1; W_2) = \MI(Z_1;Z_2).
\]

A concrete example of this construction is as follows: let $Z$ represent the \emph{brightness level} of an image, modeled as a continuous random variable. For each brightness level $\theta$, define $\pi_\theta$ as a distribution over natural images with brightness $\theta$. Then, $W_1$ and $W_2$ correspond to two independently sampled images from the brightness levels $Z_1$ and $Z_2$,
resulting in a dataset where complex dependencies are present, but the \gls{MI} remains analytically tractable.

\subsection{Numerical results for synthetic experiments}\label{app:num_synth}

In this section we provide the numerical results for the main paper experiments. In \cref{tab:mi_estimate:base:distribution_method:GT:10k} we can find the numerical results for continuous synthetic distributions, in \cref{tab:mixture:base:distribution_method:GT:10k} we can find the numerical results for discrete synthetic distributions.
\begin{table}[t]
\centering
\scriptsize
\setlength{\tabcolsep}{4pt}
\begin{tabular}{llrrrrrrr}
\toprule
Distribution & Dim. & InfoNCE & MINE-DV & MINE-NWJ & KSG & WKL & MINDE-C & MINDE-J \\
\midrule
\multirow{7}{*}{Correlated Normal} & 1 & 1.80 & 1.98 & 1.97 & {2.00} & 2.00 & 1.97 & 1.97 \\
 & 2 & 3.58 & 3.87 & 3.79 & {4.00} & 4.02 & 3.92 & 3.93 \\
 & 4 & 5.95 & 6.66 & 5.71 & 7.38 & 8.23 & 7.81 & {7.82} \\
 & 8 & 6.24 & 9.74 & 6.13 & 8.93 & 17.79 & 15.66 & {15.74} \\
 & 16 & 6.24 & 9.81 & 6.05 & 8.14 & 24.98 & 31.00 & {31.81} \\
 & 32 & 6.24 & 9.66 & 5.85 & 6.61 & 57.00 & 58.63 & {68.59} \\
 & 64 & 6.24 & 9.51 & 5.84 & 4.95 & 60.73 & 99.58 & {103.14} \\
\midrule
\multirow{7}{*}{Correlated Uniform} & 1 & 1.78 & 1.96 & 1.95 & 2.00 & {2.00} & 1.92 & 1.94 \\
 & 2 & 3.51 & 3.69 & 3.61 & {3.98} & 3.97 & 3.75 & 3.86 \\
 & 4 & 5.86 & 5.50 & 5.14 & 6.77 & {8.02} & 7.25 & 7.53 \\
 & 8 & 6.23 & 7.26 & 5.89 & 7.83 & 23.29 & 13.01 & {14.26} \\
 & 16 & 6.24 & 8.62 & 5.99 & 7.10 & 48.08 & 22.61 & {26.90} \\
 & 32 & 6.24 & 9.21 & 5.92 & 5.92 & 15.28 & 33.53 & {39.31} \\
 & 64 & 6.24 & 9.21 & 5.87 & 4.50 & 10.06 & {50.08} & 30.91 \\
\midrule
\multirow{7}{*}{Log-Gamma-Exp.} & 1 & 1.77 & 1.97 & 1.95 & {2.00} & 2.00 & 1.94 & 1.96 \\
 & 2 & 3.48 & 3.77 & 3.68 & {4.01} & 4.01 & 3.87 & 3.93 \\
 & 4 & 5.81 & 6.23 & 5.50 & 7.22 & 8.18 & 7.64 & {7.87} \\
 & 8 & 6.23 & 8.37 & 6.00 & 9.23 & 21.04 & 15.03 & {15.87} \\
 & 16 & 6.24 & 9.37 & 5.66 & 9.77 & 34.06 & 28.57 & {32.78} \\
 & 32 & 6.24 & 9.29 & 5.13 & 9.54 & 103.76 & {47.83} & 99.38 \\
 & 64 & 6.24 & 9.40 & 5.36 & 9.31 & {138.08} & 71.67 & 287.51 \\
\midrule
\multirow{7}{*}{Smoothed Uniform} & 1 & 1.98 & 1.97 & 1.97 & {2.00} & 2.00 & 1.88 & 1.93 \\
 & 2 & {3.90} & 3.85 & 3.80 & 3.88 & 3.88 & 3.75 & 3.86 \\
 & 4 & 6.08 & 7.05 & 5.78 & 6.89 & 7.62 & 7.42 & {7.70} \\
 & 8 & 6.24 & 9.93 & 6.14 & 9.07 & 23.70 & 14.65 & {15.48} \\
 & 16 & 6.24 & 9.91 & 6.03 & 9.47 & 37.79 & 28.53 & {31.35} \\
 & 32 & 6.24 & 9.78 & 5.95 & 9.29 & 23.58 & 53.43 & {61.13} \\
 & 64 & 6.24 & 9.77 & 5.86 & 8.76 & 24.87 & {82.42} & 34.37 \\
\bottomrule
\end{tabular}
\caption{Mean estimated MI [nats] over 10 seeds versus dimensionality $d$ ($N_\mathrm{train}=10^5$, ground-truth MI $=2d$ nats).}
\label{tab:dim_mean}
\end{table}

\begin{table}[t]
\centering
\scriptsize
\setlength{\tabcolsep}{4pt}
\begin{tabular}{llrrrrrrr}
\toprule
Distribution & Dim. & InfoNCE & MINE-DV & MINE-NWJ & KSG & WKL & MINDE-C & MINDE-J \\
\midrule
\multirow{7}{*}{Correlated Normal} & 1 & $3.9\times10^{-6}$ & $7.7\times10^{-5}$ & $1.1\times10^{-4}$ & $9.4\times10^{-6}$ & $8.7\times10^{-5}$ & $1.7\times10^{-4}$ & $1.3\times10^{-4}$ \\
 & 2 & $3.7\times10^{-5}$ & $1.3\times10^{-3}$ & $1.2\times10^{-3}$ & $3.0\times10^{-5}$ & $5.0\times10^{-5}$ & $2.8\times10^{-4}$ & $2.8\times10^{-4}$ \\
 & 4 & $7.0\times10^{-6}$ & 0.04 & $6.4\times10^{-3}$ & $3.0\times10^{-5}$ & $1.5\times10^{-3}$ & $2.8\times10^{-3}$ & $3.1\times10^{-3}$ \\
 & 8 & $5.8\times10^{-9}$ & 0.43 & $7.7\times10^{-3}$ & $1.1\times10^{-4}$ & 303.0 & $3.9\times10^{-3}$ & $3.8\times10^{-3}$ \\
 & 16 & $8.7\times10^{-13}$ & 0.3 & $8.2\times10^{-3}$ & $2.1\times10^{-4}$ & 911.9 & 0.016 & 0.012 \\
 & 32 & $3.4\times10^{-14}$ & 0.062 & 0.01 & $1.4\times10^{-4}$ & $1.2\times10^{3}$ & 0.089 & 0.37 \\
 & 64 & $1.2\times10^{-14}$ & 0.038 & 0.013 & $6.1\times10^{-5}$ & 572.8 & 1.1 & 8.8 \\
\midrule
\multirow{7}{*}{Correlated Uniform} & 1 & $1.8\times10^{-5}$ & $1.0\times10^{-4}$ & $1.2\times10^{-4}$ & $2.6\times10^{-5}$ & $8.1\times10^{-5}$ & $5.3\times10^{-5}$ & $6.3\times10^{-5}$ \\
 & 2 & $1.0\times10^{-4}$ & $2.0\times10^{-3}$ & $2.2\times10^{-3}$ & $3.3\times10^{-5}$ & $3.9\times10^{-5}$ & $8.3\times10^{-4}$ & $6.0\times10^{-4}$ \\
 & 4 & $1.2\times10^{-4}$ & $6.2\times10^{-3}$ & $4.6\times10^{-3}$ & $4.2\times10^{-4}$ & $2.1\times10^{-3}$ & $4.7\times10^{-3}$ & $2.4\times10^{-3}$ \\
 & 8 & $2.3\times10^{-7}$ & 0.064 & $9.6\times10^{-3}$ & $2.1\times10^{-3}$ & 285.4 & 0.012 & 0.016 \\
 & 16 & $3.7\times10^{-9}$ & 0.22 & $9.9\times10^{-3}$ & $3.8\times10^{-4}$ & $1.2\times10^{3}$ & 0.013 & 0.17 \\
 & 32 & $1.2\times10^{-11}$ & 0.049 & $6.1\times10^{-3}$ & $1.4\times10^{-4}$ & 428.3 & 0.053 & 0.19 \\
 & 64 & $4.9\times10^{-14}$ & 0.099 & $8.6\times10^{-3}$ & $4.6\times10^{-5}$ & 334.6 & 0.88 & 4.8 \\
\midrule
\multirow{7}{*}{Log-Gamma-Exp.} & 1 & $8.0\times10^{-6}$ & $1.3\times10^{-4}$ & $5.4\times10^{-4}$ & $3.0\times10^{-5}$ & $8.4\times10^{-5}$ & $2.0\times10^{-4}$ & $8.3\times10^{-5}$ \\
 & 2 & $2.5\times10^{-5}$ & $1.5\times10^{-3}$ & $2.3\times10^{-3}$ & $2.0\times10^{-5}$ & $2.2\times10^{-5}$ & $6.7\times10^{-4}$ & $1.2\times10^{-4}$ \\
 & 4 & $2.7\times10^{-5}$ & 0.015 & 0.011 & $2.6\times10^{-5}$ & $1.9\times10^{-3}$ & $2.5\times10^{-3}$ & $6.7\times10^{-4}$ \\
 & 8 & $2.2\times10^{-7}$ & 0.12 & 0.024 & $1.8\times10^{-5}$ & 158.9 & $2.8\times10^{-3}$ & $3.4\times10^{-3}$ \\
 & 16 & $3.3\times10^{-10}$ & 0.33 & 0.097 & $1.8\times10^{-5}$ & 311.4 & 0.027 & 0.033 \\
 & 32 & $3.5\times10^{-14}$ & 0.2 & 0.24 & $2.2\times10^{-5}$ & 147.4 & 0.36 & 74.2 \\
 & 64 & $1.9\times10^{-14}$ & 0.51 & 0.26 & $3.9\times10^{-6}$ & 351.6 & 16.7 & 3.6 \\
\midrule
\multirow{7}{*}{Smoothed Uniform} & 1 & $5.2\times10^{-6}$ & $8.5\times10^{-5}$ & $8.6\times10^{-5}$ & $1.4\times10^{-5}$ & $4.8\times10^{-5}$ & $1.0\times10^{-4}$ & $8.7\times10^{-5}$ \\
 & 2 & $1.5\times10^{-5}$ & $2.2\times10^{-3}$ & $2.1\times10^{-3}$ & $3.1\times10^{-5}$ & $7.4\times10^{-5}$ & $5.3\times10^{-4}$ & $3.2\times10^{-4}$ \\
 & 4 & $5.4\times10^{-5}$ & 0.062 & $5.7\times10^{-3}$ & $3.5\times10^{-5}$ & $2.6\times10^{-3}$ & $9.9\times10^{-4}$ & $6.8\times10^{-4}$ \\
 & 8 & $3.6\times10^{-9}$ & 0.4 & 0.011 & $1.7\times10^{-5}$ & 580.2 & 0.018 & $5.8\times10^{-3}$ \\
 & 16 & $1.2\times10^{-13}$ & 0.45 & 0.011 & $1.8\times10^{-5}$ & $1.4\times10^{3}$ & 0.18 & $6.2\times10^{-3}$ \\
 & 32 & $3.4\times10^{-14}$ & 0.36 & 0.021 & $1.4\times10^{-5}$ & 347.4 & 0.14 & 0.2 \\
 & 64 & $2.0\times10^{-14}$ & 0.11 & 0.012 & $1.5\times10^{-5}$ & 338.5 & 2.1 & 10.5 \\
\bottomrule
\end{tabular}
\caption{Variance of the estimated MI over 10 seeds versus dimensionality $d$ ($N_\mathrm{train}=10^5$, ground-truth MI $=2d$ nats).}
\label{tab:dim_variance}
\end{table}

\begin{table}[t]
\centering
\scriptsize
\setlength{\tabcolsep}{4pt}
\begin{tabular}{llrrrrrrr}
\toprule
Distribution & $N_\mathrm{train}$ & InfoNCE & MINE-DV & MINE-NWJ & KSG & WKL & MINDE-C & MINDE-J \\
\midrule
\multirow{7}{*}{Correlated Normal} & 100 & 5.09 & 6.22 & 5.05 & 3.35 & {8.35} & 360.22 & 46.75 \\
 & 316 & 4.84 & 6.65 & 5.58 & 4.29 & {8.61} & 108.22 & 25.95 \\
 & 1000 & 5.80 & 6.71 & 5.71 & 5.16 & {8.62} & 28.90 & 14.27 \\
 & 3162 & 5.92 & 6.52 & 5.57 & 5.94 & {8.70} & 9.64 & 8.79 \\
 & 10000 & 5.94 & 6.76 & 5.72 & 6.55 & 8.44 & 8.12 & {8.05} \\
 & 31622 & 5.95 & 6.68 & 5.73 & 7.03 & 8.31 & 7.85 & {7.87} \\
 & 100000 & 5.95 & 6.66 & 5.71 & 7.38 & 8.23 & 7.81 & {7.82} \\
\midrule
\multirow{7}{*}{Correlated Uniform} & 100 & 1.16 & 4.57 & 4.13 & 3.11 & {9.51} & 74.92 & 21.19 \\
 & 316 & 0.66 & 5.11 & 4.67 & 3.92 & {8.45} & 27.09 & 13.98 \\
 & 1000 & 4.91 & 5.39 & 4.97 & 4.69 & {8.37} & 10.86 & 9.02 \\
 & 3162 & 5.54 & 5.26 & 4.85 & 5.38 & 8.37 & {7.78} & 7.71 \\
 & 10000 & 5.78 & 5.48 & 5.05 & 5.93 & {8.16} & 7.37 & 7.60 \\
 & 31622 & 5.84 & 5.51 & 5.15 & 6.40 & {8.06} & 7.27 & 7.56 \\
 & 100000 & 5.86 & 5.50 & 5.14 & 6.77 & {8.02} & 7.25 & 7.53 \\
\midrule
\multirow{7}{*}{Log-Gamma-Exp.} & 100 & 3.18 & 5.68 & 4.85 & 3.55 & {9.52} & 432.45 & 118.90 \\
 & 316 & 2.06 & 6.00 & 5.32 & 4.44 & {8.98} & 144.62 & 38.99 \\
 & 1000 & 5.52 & 6.12 & 5.42 & 5.21 & {8.43} & 22.08 & 13.97 \\
 & 3162 & 5.74 & 5.95 & 5.33 & 5.88 & 8.47 & {8.45} & 8.51 \\
 & 10000 & 5.78 & 6.10 & 5.46 & 6.43 & 8.32 & 7.81 & {8.01} \\
 & 31622 & 5.81 & 6.12 & 5.47 & 6.87 & 8.24 & 7.69 & {7.90} \\
 & 100000 & 5.81 & 6.23 & 5.50 & 7.22 & 8.18 & 7.64 & {7.87} \\
\midrule
\multirow{7}{*}{Smoothed Uniform} & 100 & 5.27 & 6.16 & 5.08 & 3.50 & {7.02} & 66.77 & 21.05 \\
 & 316 & 5.04 & 6.78 & 5.58 & 4.37 & {8.34} & 26.10 & 13.88 \\
 & 1000 & 5.92 & 7.05 & 5.73 & 5.15 & {7.83} & 10.79 & 9.21 \\
 & 3162 & 6.00 & 6.70 & 5.58 & 5.80 & 7.83 & 7.79 & {7.93} \\
 & 10000 & 6.05 & 6.98 & 5.72 & 6.30 & 7.57 & 7.50 & {7.75} \\
 & 31622 & 6.07 & 7.07 & 5.78 & 6.65 & 7.53 & 7.44 & {7.72} \\
 & 100000 & 6.08 & 7.05 & 5.78 & 6.89 & 7.62 & 7.42 & {7.70} \\
\bottomrule
\end{tabular}
\caption{Mean estimated MI [nats] over 10 seeds versus training samples ($d=4$, ground-truth MI $=8$ nats).}
\label{tab:samples_d4_mean}
\end{table}

\begin{table}[t]
\centering
\scriptsize
\setlength{\tabcolsep}{4pt}
\begin{tabular}{llrrrrrrr}
\toprule
Distribution & $N_\mathrm{train}$ & InfoNCE & MINE-DV & MINE-NWJ & KSG & WKL & MINDE-C & MINDE-J \\
\midrule
\multirow{7}{*}{Correlated Normal} & 100 & 0.064 & 0.019 & 0.071 & $6.8\times10^{-3}$ & 4.6 & $2.6\times10^{3}$ & 6.6 \\
 & 316 & 0.063 & 0.019 & 0.015 & $4.5\times10^{-3}$ & 1.7 & 38.4 & 1.9 \\
 & 1000 & $7.3\times10^{-4}$ & 0.017 & $3.1\times10^{-3}$ & $8.0\times10^{-4}$ & 0.35 & 1.9 & 0.14 \\
 & 3162 & $3.8\times10^{-5}$ & 0.026 & $4.7\times10^{-3}$ & $2.8\times10^{-4}$ & 0.063 & 0.04 & $5.5\times10^{-3}$ \\
 & 10000 & $1.3\times10^{-5}$ & 0.026 & $2.0\times10^{-3}$ & $3.7\times10^{-4}$ & 0.029 & $2.1\times10^{-3}$ & $1.6\times10^{-3}$ \\
 & 31622 & $6.2\times10^{-6}$ & 0.043 & $3.6\times10^{-3}$ & $3.3\times10^{-5}$ & $2.7\times10^{-3}$ & $2.2\times10^{-3}$ & $2.4\times10^{-3}$ \\
 & 100000 & $7.0\times10^{-6}$ & 0.04 & $6.4\times10^{-3}$ & $3.0\times10^{-5}$ & $1.5\times10^{-3}$ & $2.8\times10^{-3}$ & $3.1\times10^{-3}$ \\
\midrule
\multirow{7}{*}{Correlated Uniform} & 100 & 0.8 & 0.021 & 0.058 & $9.2\times10^{-3}$ & 4.8 & 35.7 & 0.77 \\
 & 316 & 0.4 & $4.1\times10^{-3}$ & $5.7\times10^{-3}$ & $7.0\times10^{-3}$ & 0.91 & 1.3 & 0.13 \\
 & 1000 & 0.014 & $4.7\times10^{-3}$ & $6.4\times10^{-3}$ & $3.4\times10^{-3}$ & 0.25 & 0.093 & $8.6\times10^{-3}$ \\
 & 3162 & $1.5\times10^{-3}$ & $6.3\times10^{-3}$ & $8.7\times10^{-3}$ & $1.7\times10^{-3}$ & 0.093 & $7.6\times10^{-3}$ & $3.5\times10^{-3}$ \\
 & 10000 & $1.3\times10^{-4}$ & 0.016 & 0.027 & $7.7\times10^{-4}$ & 0.029 & $3.1\times10^{-3}$ & $2.6\times10^{-3}$ \\
 & 31622 & $4.0\times10^{-5}$ & $7.8\times10^{-3}$ & $4.3\times10^{-3}$ & $8.5\times10^{-4}$ & $4.3\times10^{-3}$ & $4.7\times10^{-3}$ & $2.0\times10^{-3}$ \\
 & 100000 & $1.2\times10^{-4}$ & $6.2\times10^{-3}$ & $4.6\times10^{-3}$ & $4.2\times10^{-4}$ & $2.1\times10^{-3}$ & $4.7\times10^{-3}$ & $2.4\times10^{-3}$ \\
\midrule
\multirow{7}{*}{Log-Gamma-Exp.} & 100 & 0.18 & 0.037 & 0.039 & 0.017 & 5.2 & $3.4\times10^{3}$ & $3.8\times10^{3}$ \\
 & 316 & 0.26 & 0.019 & 0.01 & $1.7\times10^{-3}$ & 0.8 & 79.0 & 81.8 \\
 & 1000 & $1.7\times10^{-3}$ & 0.028 & $2.1\times10^{-3}$ & $2.0\times10^{-3}$ & 0.41 & 0.88 & 0.67 \\
 & 3162 & $7.2\times10^{-5}$ & 0.011 & 0.014 & $1.4\times10^{-3}$ & 0.091 & $7.8\times10^{-3}$ & $1.5\times10^{-3}$ \\
 & 10000 & $5.4\times10^{-5}$ & $7.4\times10^{-3}$ & $6.5\times10^{-3}$ & $1.7\times10^{-4}$ & 0.02 & $2.2\times10^{-3}$ & $9.7\times10^{-4}$ \\
 & 31622 & $4.3\times10^{-5}$ & 0.049 & $7.0\times10^{-3}$ & $3.7\times10^{-4}$ & $7.3\times10^{-3}$ & $6.4\times10^{-4}$ & $1.4\times10^{-3}$ \\
 & 100000 & $2.7\times10^{-5}$ & 0.015 & 0.011 & $2.6\times10^{-5}$ & $1.9\times10^{-3}$ & $2.5\times10^{-3}$ & $6.7\times10^{-4}$ \\
\midrule
\multirow{7}{*}{Smoothed Uniform} & 100 & 0.071 & 0.027 & 0.068 & $6.7\times10^{-3}$ & 2 & 15.2 & 0.25 \\
 & 316 & 0.03 & 0.071 & $5.5\times10^{-3}$ & $6.7\times10^{-3}$ & 0.69 & 2.5 & 0.11 \\
 & 1000 & $9.0\times10^{-4}$ & 0.012 & $9.0\times10^{-3}$ & $1.3\times10^{-3}$ & 0.34 & 0.094 & 0.026 \\
 & 3162 & $2.3\times10^{-5}$ & 0.044 & $8.6\times10^{-3}$ & $5.4\times10^{-4}$ & 0.054 & $2.5\times10^{-3}$ & $2.4\times10^{-3}$ \\
 & 10000 & $1.3\times10^{-4}$ & 0.036 & $8.3\times10^{-3}$ & $1.3\times10^{-4}$ & 0.02 & $2.8\times10^{-3}$ & $1.7\times10^{-3}$ \\
 & 31622 & $6.1\times10^{-5}$ & 0.051 & $7.2\times10^{-3}$ & $1.1\times10^{-4}$ & 0.01 & $1.2\times10^{-3}$ & $9.2\times10^{-4}$ \\
 & 100000 & $5.4\times10^{-5}$ & 0.062 & $5.7\times10^{-3}$ & $3.5\times10^{-5}$ & $2.6\times10^{-3}$ & $9.9\times10^{-4}$ & $6.8\times10^{-4}$ \\
\bottomrule
\end{tabular}
\caption{Variance of the estimated MI over 10 seeds versus training samples ($d=4$, ground-truth MI $=8$ nats).}
\label{tab:samples_d4_variance}
\end{table}

\begin{table}[!ht]
    \centering
    \scriptsize
    \begin{tabular}{llccccccccccc}
        \toprule
         &  & \multicolumn{11}{c}{Ground truth Mutual Information} \\
        Distribution & Estimator & 0 & 1 & 2 & 3 & 4 & 5 & 6 & 7 & 8 & 9 & 10 \\
        \midrule
        \multirow[l]{7}{*}{Circles} & InfoNCE & $ 0.00 $ & $ 0.07 $ & $ 0.13 $ & $ 0.24 $ & $ 0.48 $ & $ 1.00 $ & $ 1.68 $ & $ 2.51 $ & $ 3.38 $ & $ 4.27 $ & $ 5.20 $ \\
         & KSG & $ 0.00 $ & $ 0.06 $ & $ 0.25 $ & $ 0.63 $ & $ 1.19 $ & $ 1.87 $ & $ 2.60 $ & $ 3.37 $ & $ 4.17 $ & $ 4.99 $ & $ 5.83 $ \\
         & MINDE-C & $ 0.01 $ & $ 0.09 $ & $ 0.16 $ & $ 0.31 $ & $ 0.56 $ & $ 1.10 $ & $ 1.51 $ & $ 2.20 $ & $ 2.86 $ & $ 3.93 $ & $ 4.60 $ \\
         & MINDE-J & $ 0.05 $ & $ 0.04 $ & $ 0.06 $ & $ 0.14 $ & $ 0.33 $ & $ 0.50 $ & $ 0.93 $ & $ 1.63 $ & $ 2.24 $ & $ 2.87 $ & $ 3.67 $ \\
         & MINE-DV & $ 0.00 $ & $ 0.07 $ & $ 0.17 $ & $ 0.37 $ & $ 0.84 $ & $ 1.46 $ & $ 2.26 $ & $ 3.07 $ & $ 3.94 $ & $ 4.82 $ & $ 5.72 $ \\
         & MINE-NWJ & $ 0.00 $ & $ 0.07 $ & $ 0.19 $ & $ 0.43 $ & $ 0.88 $ & $ 1.59 $ & $ 2.46 $ & $ 3.29 $ & $ 4.18 $ & $ 5.16 $ & $ 5.97 $ \\
         & WKL & $ 0.05 $ & $ 0.03 $ & $ 0.03 $ & $ 0.27 $ & $ 0.71 $ & $ 1.32 $ & $ 2.02 $ & $ 2.75 $ & $ 3.50 $ & $ 4.29 $ & $ 5.10 $ \\
        \cmidrule{1-13}
        \multirow[l]{7}{*}{Moons} & KSG & $ 0.00 $ & $ 0.14 $ & $ 0.43 $ & $ 0.94 $ & $ 1.60 $ & $ 2.35 $ & $ 3.14 $ & $ 3.96 $ & $ 4.81 $ & $ 5.68 $ & $ 6.56 $ \\
         & InfoNCE & $ 0.00 $ & $ 0.20 $ & $ 0.37 $ & $ 0.76 $ & $ 1.35 $ & $ 2.08 $ & $ 2.86 $ & $ 3.67 $ & $ 4.52 $ & $ 5.38 $ & $ 6.27 $ \\
         & MINDE-C & $ 0.00 $ & $ 0.09 $ & $ 0.20 $ & $ 0.44 $ & $ 0.83 $ & $ 1.40 $ & $ 2.07 $ & $ 2.84 $ & $ 3.64 $ & $ 4.47 $ & $ 5.35 $ \\
         & MINDE-J & $ 0.02 $ & $ 0.05 $ & $ 0.12 $ & $ 0.27 $ & $ 0.54 $ & $ 0.96 $ & $ 1.51 $ & $ 2.17 $ & $ 2.90 $ & $ 3.67 $ & $ 4.53 $ \\
         & MINE-DV & $ 0.00 $ & $ 0.12 $ & $ 0.38 $ & $ 0.90 $ & $ 1.61 $ & $ 2.43 $ & $ 3.31 $ & $ 4.22 $ & $ 5.13 $ & $ 6.05 $ & $ 7.00 $ \\
         & MINE-NWJ & $ 0.00 $ & $ 0.12 $ & $ 0.39 $ & $ 0.91 $ & $ 1.63 $ & $ 2.46 $ & $ 3.35 $ & $ 4.27 $ & $ 5.20 $ & $ 6.14 $ & $ 7.06 $ \\
         & WKL & $ 0.04 $ & $ 0.03 $ & $ 0.01 $ & $ 0.18 $ & $ 0.55 $ & $ 1.11 $ & $ 1.77 $ & $ 2.48 $ & $ 3.23 $ & $ 4.00 $ & $ 4.79 $ \\
        \cmidrule{1-13}
        \multirow[l]{7}{*}{SCurve} & InfoNCE & $ 0.00 $ & $ 0.56 $ & $ 1.13 $ & $ 1.72 $ & $ 2.39 $ & $ 3.19 $ & $ 3.98 $ & $ 4.76 $ & $ 5.60 $ & $ 6.48 $ & $ 7.39 $ \\
         & KSG & $ 0.00 $ & $ 0.26 $ & $ 0.69 $ & $ 1.29 $ & $ 1.99 $ & $ 2.75 $ & $ 3.54 $ & $ 4.36 $ & $ 5.20 $ & $ 6.06 $ & $ 6.94 $ \\
         & MINDE-C & $ 0.01 $ & $ 0.24 $ & $ 0.54 $ & $ 0.98 $ & $ 1.52 $ & $ 2.18 $ & $ 2.84 $ & $ 3.70 $ & $ 4.47 $ & $ 4.52 $ & $ 5.35 $ \\
         & MINDE-J & $ 0.07 $ & $ 0.10 $ & $ 0.37 $ & $ 0.65 $ & $ 1.01 $ & $ 1.51 $ & $ 2.11 $ & $ 2.80 $ & $ 3.48 $ & $ 4.20 $ & $ 4.95 $ \\
         & MINE-DV & $ 0.00 $ & $ 0.27 $ & $ 0.61 $ & $ 1.16 $ & $ 1.86 $ & $ 2.63 $ & $ 3.47 $ & $ 4.31 $ & $ 5.18 $ & $ 6.07 $ & $ 6.99 $ \\
         & MINE-NWJ & $ 0.00 $ & $ 0.28 $ & $ 0.64 $ & $ 1.19 $ & $ 1.91 $ & $ 2.69 $ & $ 3.52 $ & $ 4.37 $ & $ 5.26 $ & $ 6.17 $ & $ 7.07 $ \\
         & WKL & $ 0.31 $ & $ 0.26 $ & $ 0.28 $ & $ 0.56 $ & $ 1.15 $ & $ 1.80 $ & $ 2.51 $ & $ 3.26 $ & $ 4.04 $ & $ 4.83 $ & $ 5.64 $ \\
        \bottomrule
    \end{tabular}
    \caption{Mutual information estimation error.}
    \label{tab:flow_numericals}
\end{table}

\begin{table}[!ht]
    \centering
    \scriptsize
    \begin{tabular}{llccccccccccc}
        \toprule
         &  & \multicolumn{11}{c}{Ground truth Mutual Information} \\
        Distribution & Estimator & 0 & 1 & 2 & 3 & 4 & 5 & 6 & 7 & 8 & 9 & 10 \\
        \midrule
        \multirow[l]{7}{*}{Uniformly Quantized} & KSG & $ 0.00 $ & $ 0.02 $ & $ 0.03 $ & $ 0.06 $ & $ 0.09 $ & $ 0.15 $ & $ 0.24 $ & $ 0.40 $ & $ 0.88 $ & $ 3.01 $ & $ 6.12 $ \\
         & WKL & $ 0.00 $ & $ 1.00 $ & $ 2.00 $ & $ 3.00 $ & $ 4.00 $ & $ 5.00 $ & $ 6.00 $ & $ 7.00 $ & $ 8.00 $ & $ 9.00 $ & $ 10.00 $ \\
         & MINE-DV & $ 0.00 $ & $ 0.03 $ & $ 0.06 $ & $ 0.15 $ & $ 0.25 $ & $ 0.54 $ & $ 1.02 $ & $ 1.54 $ & $ 2.25 $ & $ 3.09 $ & $ 4.07 $ \\
         & MINE-NWJ & $ 0.00 $ & $ 0.03 $ & $ 0.07 $ & $ 0.17 $ & $ 0.34 $ & $ 0.69 $ & $ 1.22 $ & $ 1.95 $ & $ 2.78 $ & $ 3.70 $ & $ 4.70 $ \\
         & InfoNCE & $ 0.00 $ & $ 0.02 $ & $ 0.03 $ & $ 0.06 $ & $ 0.13 $ & $ 0.27 $ & $ 0.59 $ & $ 1.18 $ & $ 2.02 $ & $ 2.96 $ & $ 3.93 $ \\
         & MINDE-C & $ 0.01 $ & $ 0.12 $ & $ 0.15 $ & $ 0.25 $ & $ 0.52 $ & $ 0.94 $ & $ 1.53 $ & $ 2.19 $ & $ 2.95 $ & $ 4.21 $ & $ 4.96 $ \\
         & MINDE-J & $ 0.15 $ & $ 1.85 $ & $ 3.49 $ & $ 3.30 $ & $ 3.05 $ & $ 3.47 $ & $ 4.04 $ & $ 3.64 $ & $ 4.08 $ & $ 4.30 $ & $ 4.49 $ \\
        \cmidrule{1-13}
        \multirow[l]{7}{*}{\makecell[l]{Smoothed \\ Discrete Uniform}} & KSG & $ 0.00 $ & $ 0.01 $ & $ 0.02 $ & $ 0.04 $ & $ 0.06 $ & $ 0.10 $ & $ 0.17 $ & $ 0.29 $ & $ 0.54 $ & $ 1.01 $ & $ 1.64 $ \\
         & WKL & $ 0.00 $ & $ 1.00 $ & $ 2.00 $ & $ 3.00 $ & $ 4.00 $ & $ 5.00 $ & $ 6.00 $ & $ 7.00 $ & $ 8.00 $ & $ 9.00 $ & $ 9.61 $ \\
         & MINE-DV & $ 0.00 $ & $ 0.04 $ & $ 0.04 $ & $ 0.08 $ & $ 0.14 $ & $ 0.32 $ & $ 0.51 $ & $ 0.97 $ & $ 1.46 $ & $ 2.22 $ & $ 3.05 $ \\
         & MINE-NWJ & $ 0.00 $ & $ 0.04 $ & $ 0.05 $ & $ 0.10 $ & $ 0.21 $ & $ 0.47 $ & $ 0.89 $ & $ 1.61 $ & $ 2.40 $ & $ 3.35 $ & $ 4.24 $ \\
         & InfoNCE & $ 0.00 $ & $ 0.04 $ & $ 0.03 $ & $ 0.07 $ & $ 0.10 $ & $ 0.23 $ & $ 0.52 $ & $ 1.12 $ & $ 1.94 $ & $ 2.89 $ & $ 3.88 $ \\
         & MINDE-C & $ 0.00 $ & $ 0.09 $ & $ 0.09 $ & $ 0.19 $ & $ 0.35 $ & $ 0.62 $ & $ 1.10 $ & $ 1.61 $ & $ 2.20 $ & $ 2.90 $ & $ 3.63 $ \\
         & MINDE-J & $ 0.03 $ & $ 0.07 $ & $ 0.05 $ & $ 0.14 $ & $ 0.30 $ & $ 0.51 $ & $ 0.86 $ & $ 1.29 $ & $ 1.79 $ & $ 2.40 $ & $ 3.07 $ \\
        \bottomrule
    \end{tabular}
    \caption{Mutual information estimation error, varying ground truth MI (10k samples).}
    \label{tab:mixture:base:distribution_method:GT:10k}
\end{table}

\begin{table}[t]
    \centering
    \scriptsize
    \begin{tabular}{llccccccccccc}
        \toprule
         &  & \multicolumn{11}{c}{Ground truth Mutual Information} \\
        Distribution & Estimator & 0 & 1 & 2 & 3 & 4 & 5 & 6 & 7 & 8 & 9 & 10 \\
        \midrule
        \multirow[l]{7}{*}{Correlated Normal} & KSG & $ 0.00 $ & $ 0.01 $ & $ 0.05 $ & $ 0.18 $ & $ 0.45 $ & $ 0.83 $ & $ 1.27 $ & $ 1.77 $ & $ 2.34 $ & $ 3.01 $ & $ 3.79 $ \\
         & WKL & $ 0.01 $ & $ 0.02 $ & $ 0.02 $ & $ 0.07 $ & $ 0.29 $ & $ 0.64 $ & $ 0.65 $ & $ 1.01 $ & $ 1.59 $ & $ 2.26 $ & $ 3.01 $ \\
         & MINE-DV & $ 0.00 $ & $ 0.02 $ & $ 0.02 $ & $ 0.05 $ & $ 0.13 $ & $ 0.30 $ & $ 0.66 $ & $ 1.31 $ & $ 2.12 $ & $ 3.00 $ & $ 3.88 $ \\
         & MINE-NWJ & $ 0.00 $ & $ 0.03 $ & $ 0.03 $ & $ 0.08 $ & $ 0.21 $ & $ 0.48 $ & $ 1.01 $ & $ 1.77 $ & $ 2.66 $ & $ 3.69 $ & $ 4.59 $ \\
         & InfoNCE & $ 0.00 $ & $ 0.02 $ & $ 0.03 $ & $ 0.06 $ & $ 0.14 $ & $ 0.30 $ & $ 0.63 $ & $ 1.19 $ & $ 1.98 $ & $ 2.91 $ & $ 3.89 $ \\
         & MINDE-C & $ 0.02 $ & $ 0.11 $ & $ 0.16 $ & $ 0.33 $ & $ 0.33 $ & $ 0.38 $ & $ 0.54 $ & $ 0.66 $ & $ 1.02 $ & $ 1.51 $ & $ 2.28 $ \\
         & MINDE-J & $ 0.07 $ & $ 0.61 $ & $ 1.22 $ & $ 1.70 $ & $ 1.46 $ & $ 1.17 $ & $ 1.21 $ & $ 1.26 $ & $ 1.36 $ & $ 1.71 $ & $ 1.94 $ \\
        \cmidrule{1-13}
        \multirow[l]{7}{*}{Correlated Student} & KSG & -- & $ 0.08 $ & $ 0.08 $ & $ 0.18 $ & $ 0.42 $ & $ 0.78 $ & $ 1.21 $ & $ 1.70 $ & $ 2.27 $ & $ 2.95 $ & $ 3.72 $ \\
         & WKL & -- & $ 0.03 $ & $ 0.03 $ & $ 0.07 $ & $ 0.26 $ & $ 0.49 $ & $ 0.66 $ & $ 1.19 $ & $ 1.77 $ & $ 2.43 $ & $ 3.19 $ \\
         & MINE-DV & -- & $ 0.12 $ & $ 0.18 $ & $ 0.24 $ & $ 0.34 $ & $ 0.60 $ & $ 1.12 $ & $ 1.90 $ & $ 2.75 $ & $ 3.64 $ & $ 4.63 $ \\
         & MINE-NWJ & -- & $ 0.13 $ & $ 0.19 $ & $ 0.28 $ & $ 0.47 $ & $ 0.85 $ & $ 1.43 $ & $ 2.26 $ & $ 3.12 $ & $ 4.10 $ & $ 5.02 $ \\
         & InfoNCE & -- & $ 0.11 $ & $ 0.12 $ & $ 0.16 $ & $ 0.24 $ & $ 0.42 $ & $ 0.75 $ & $ 1.29 $ & $ 2.04 $ & $ 2.94 $ & $ 3.90 $ \\
         & MINDE-C & -- & $ 0.13 $ & $ 0.18 $ & $ 0.31 $ & $ 0.51 $ & $ 0.49 $ & $ 0.64 $ & $ 1.03 $ & $ 1.55 $ & $ 2.07 $ & $ 2.78 $ \\
         & MINDE-J& -- & $ 0.67 $ & $ 1.26 $ & $ 1.75 $ & $ 1.58 $ & $ 1.42 $ & $ 1.19 $ & $ 1.33 $ & $ 1.55 $ & $ 1.83 $ & $ 2.26 $ \\
        \cmidrule{1-13}
        \multirow[l]{7}{*}{Correlated Uniform} & KSG & $ 0.00 $ & $ 0.03 $ & $ 0.11 $ & $ 0.29 $ & $ 0.60 $ & $ 0.99 $ & $ 1.45 $ & $ 1.95 $ & $ 2.53 $ & $ 3.21 $ & $ 3.99 $ \\
         & WKL & $ 0.00 $ & $ 0.05 $ & $ 0.07 $ & $ 0.19 $ & $ 0.44 $ & $ 0.79 $ & $ 1.22 $ & $ 1.67 $ & $ 2.19 $ & $ 2.80 $ & $ 3.51 $ \\
         & MINE-DV & $ 0.00 $ & $ 0.07 $ & $ 0.12 $ & $ 0.25 $ & $ 0.46 $ & $ 0.82 $ & $ 1.42 $ & $ 2.17 $ & $ 3.04 $ & $ 4.02 $ & $ 4.92 $ \\
         & MINE-NWJ & $ 0.00 $ & $ 0.07 $ & $ 0.14 $ & $ 0.28 $ & $ 0.53 $ & $ 0.96 $ & $ 1.58 $ & $ 2.43 $ & $ 3.30 $ & $ 4.24 $ & $ 5.22 $ \\
         & InfoNCE & $ 0.00 $ & $ 0.07 $ & $ 0.14 $ & $ 0.22 $ & $ 0.35 $ & $ 0.58 $ & $ 0.99 $ & $ 1.61 $ & $ 2.43 $ & $ 3.34 $ & $ 4.30 $ \\
         & MINDE-C & $ 0.06 $ & $ 0.16 $ & $ 0.33 $ & $ 0.50 $ & $ 0.68 $ & $ 0.72 $ & $ 0.68 $ & $ 0.75 $ & $ 0.81 $ & $ 1.40 $ & $ 2.14 $ \\
         & MINDE-J & $ 0.11 $ & $ 0.97 $ & $ 1.89 $ & $ 2.37 $ & $ 2.50 $ & $ 2.53 $ & $ 2.46 $ & $ 3.20 $ & $ 3.53 $ & $ 3.30 $ & $ 3.68 $ \\
        \cmidrule{1-13}
        \multirow[l]{7}{*}{Log-Gamma-Exp.} & KSG & $ 0.00 $ & $ 0.01 $ & $ 0.02 $ & $ 0.04 $ & $ 0.06 $ & $ 0.11 $ & $ 0.20 $ & $ 0.37 $ & $ 0.66 $ & $ 1.13 $ & $ 1.80 $ \\
         & WKL & $ 0.00 $ & $ 0.02 $ & $ 0.02 $ & $ 0.03 $ & $ 0.04 $ & $ 0.08 $ & $ 0.13 $ & $ 0.23 $ & $ 0.42 $ & $ 0.70 $ & $ 1.07 $ \\
         & MINE-DV & $ 0.00 $ & $ 0.02 $ & $ 0.07 $ & $ 0.11 $ & $ 0.24 $ & $ 0.50 $ & $ 0.92 $ & $ 1.75 $ & $ 2.54 $ & $ 3.39 $ & $ 4.35 $ \\
         & MINE-NWJ & $ 0.00 $ & $ 0.03 $ & $ 0.10 $ & $ 0.17 $ & $ 0.35 $ & $ 0.67 $ & $ 1.22 $ & $ 1.99 $ & $ 2.94 $ & $ 3.87 $ & $ 4.89 $ \\
         & InfoNCE & $ 0.00 $ & $ 0.03 $ & $ 0.05 $ & $ 0.09 $ & $ 0.18 $ & $ 0.37 $ & $ 0.70 $ & $ 1.27 $ & $ 2.04 $ & $ 2.97 $ & $ 3.94 $ \\
         & MINDE-C & $ 1.77 $ & $ 0.06 $ & $ 0.29 $ & $ 0.31 $ & $ 0.40 $ & $ 0.85 $ & $ 0.55 $ & $ 0.80 $ & $ 0.88 $ & $ 1.39 $ & $ 1.63 $ \\
         & MINDE-J & $ 0.16 $ & $ 0.81 $ & $ 1.63 $ & $ 1.93 $ & $ 1.79 $ & $ 1.44 $ & $ 0.83 $ & $ 1.76 $ & $ 1.73 $ & $ 1.63 $ & $ 2.75 $ \\
        \cmidrule{1-13}
        \multirow[l]{7}{*}{Smoothed Uniform} & KSG & $ 0.00 $ & $ 0.11 $ & $ 0.18 $ & $ 0.33 $ & $ 0.60 $ & $ 0.98 $ & $ 1.41 $ & $ 1.90 $ & $ 2.44 $ & $ 3.09 $ & $ 3.85 $ \\
         & WKL & $ 0.00 $ & $ 0.10 $ & $ 0.13 $ & $ 0.21 $ & $ 0.44 $ & $ 0.79 $ & $ 1.21 $ & $ 1.68 $ & $ 2.18 $ & $ 2.77 $ & $ 3.45 $ \\
         & MINE-DV & $ 0.00 $ & $ 0.07 $ & $ 0.08 $ & $ 0.12 $ & $ 0.19 $ & $ 0.38 $ & $ 0.68 $ & $ 1.29 $ & $ 2.07 $ & $ 2.99 $ & $ 3.89 $ \\
         & MINE-NWJ & $ 0.00 $ & $ 0.07 $ & $ 0.09 $ & $ 0.15 $ & $ 0.26 $ & $ 0.52 $ & $ 1.02 $ & $ 1.79 $ & $ 2.72 $ & $ 3.65 $ & $ 4.60 $ \\
         & InfoNCE & $ 0.00 $ & $ 0.07 $ & $ 0.07 $ & $ 0.09 $ & $ 0.14 $ & $ 0.26 $ & $ 0.56 $ & $ 1.16 $ & $ 1.97 $ & $ 2.93 $ & $ 3.92 $ \\
         & MINDE-C & $ 0.06 $ & $ 0.10 $ & $ 0.15 $ & $ 0.30 $ & $ 0.47 $ & $ 0.96 $ & $ 1.34 $ & $ 1.34 $ & $ 1.78 $ & $ 1.99 $ & $ 2.00 $ \\
         & MINDE-J & $ 0.11 $ & $ 0.53 $ & $ 0.90 $ & $ 1.28 $ & $ 1.43 $ & $ 1.64 $ & $ 1.49 $ & $ 1.53 $ & $ 1.62 $ & $ 1.43 $ & $ 1.79 $ \\
         \cmidrule{1-13}
         \multirow[l]{7}{*}{RareEventChannel} & KSG & -- & -- & -- & $ 0.63 $ & $ 1.16 $ & $ 1.84 $ & $ 2.65 $ & $ 3.54 $ & $ 4.49 $ & $ 5.45 $ & $ 6.43 $ \\
         & WKL & -- & -- & -- & $ 0.42 $ & $ 0.87 $ & $ 1.47 $ & $ 2.23 $ & $ 3.07 $ & $ 4.00 $ & $ 4.95 $ & $ 5.92 $ \\
         & MINE-DV & -- & -- & -- & $ 1.00 $ & $ 1.44 $ & $ 2.06 $ & $ 2.83 $ & $ 3.72 $ & $ 4.66 $ & $ 5.70 $ & $ 6.80 $ \\
         & MINE-NWJ & -- & -- & -- & $ 1.04 $ & $ 1.52 $ & $ 2.19 $ & $ 3.01 $ & $ 3.95 $ & $ 4.95 $ & $ 5.93 $ & $ 6.89 $ \\
         & InfoNCE & -- & -- & -- & $ 1.02 $ & $ 1.39 $ & $ 1.92 $ & $ 2.62 $ & $ 3.46 $ & $ 4.38 $ & $ 5.35 $ & $ 6.27 $ \\
         & MINDE-C & -- & -- & -- & $ 0.29 $ & $ 0.61 $ & $ 1.06 $ & $ 1.50 $ & $ 2.02 $ & $ 2.66 $ & $ 3.43 $ & $ 4.28 $ \\
         & MINDE-J & -- & -- & -- & $ 0.39 $ & $ 0.68 $ & $ 0.95 $ & $ 1.29 $ & $ 1.70 $ & $ 2.24 $ & $ 2.87 $ & $ 3.64 $ \\
        \bottomrule
    \end{tabular}
    \caption{Mutual information estimation error, varying ground truth MI (10k samples).}
    \label{tab:mi_estimate:base:distribution_method:GT:10k}
\end{table}

\subsection{Experimental details for synthetic experiments}
For \textsc{minde}, we employ a \textsc{mlp} with skip connections as in \citep{franzese2024minde}. For the diffusion process, we use both importance sampling and \gls{EMA}. We report the main hyperparameters (includeing the \gls{SDE} type) in \Cref{tab:hyperparameters_minde_combined}. We use the same configuration for both \textsc{minde-c} and \textsc{minde-j}.

%\input{tables/exp_configs/minde_synthetic}
%\subsection{Synthetic image experiments}
\subsection{Experimental details for MNIST image experiments}
For \textsc{minde}, we employ a 2D UNet backbone \citep{ronneberger2015unetconvolutionalnetworksbiomedical} from the \textit{diffusers} Python package \citep{von-platen-etal-2022-diffusers}. For the diffusion process, we use both importance sampling and \gls{EMA}. We report the main hyperparameters in \Cref{tab:hyperparameters_minde_combined}. We use the same configuration for both \textsc{minde-c} and \textsc{minde-j}.

\begin{table}[t]
    \centering
    \scriptsize
    \begin{subtable}[t]{0.49\textwidth}
        \centering
        %\begin{table}[htbp]
%\centering
%\scriptsize
\begin{tabular}{@{}lll@{}}
\toprule
\textbf{Category} & \textbf{Hyperparameter} & \textbf{Value} \\
\midrule
\textbf{Model} & Backbone & \texttt{MLP} \\
 & Hidden Dimensions & 32 \\
 & Total Parameters & 11 K \\
 & \gls{EMA} rate & 0.999 \\
\addlinespace
\textbf{SDE} & \gls{SDE} Type & \makecell[l]{Variance \\ Preserving} \\
 & $T$ & 1.0 \\
 & $\beta_{\min}$ & 0.1 \\
 & $\beta_{\max}$ & 20.0 \\
 & Importance Sampling & Yes \\
\addlinespace
\textbf{Training} & Optimizer & \texttt{Adam} \\
 & Learning Rate & 1e-2 \\
 & Train Batch Size & 512 \\
 & Training Steps & 10000 \\
\bottomrule
\end{tabular}
%\caption{Model and training hyperparameters for MINDE in synthetic experiments. We report key parameters for the estimator, SDE, and training process.}
%\label{tab:hyperparameters_minde_synthetic}
%\end{table}

    \end{subtable}
    \begin{subtable}[t]{0.49\textwidth}
        \centering
        %\begin{table}[htbp]
%\centering
%\scriptsize
\begin{tabular}{@{}lll@{}}
\toprule
\textbf{Category} & \textbf{Hyperparameter} & \textbf{Value} \\
\midrule
\textbf{Model} & Backbone & \texttt{UNet} \\
 & Hidden Dimensions & 64 \\
 & Layers per Block & 2 \\
 & Total Parameters & 14.4 M \\
 & \gls{EMA} rate & 0.9999 \\
\addlinespace
\textbf{SDE} & \gls{SDE} Type & \makecell[l]{Variance \\ Preserving} \\
 & $T$ & 1.0 \\
 & $\beta_{\min}$ & 0.1 \\
 & $\beta_{\max}$ & 20.0 \\
 & Importance Sampling & Yes \\
\addlinespace
\textbf{Training} & Optimizer & \texttt{Adam} \\
 & Learning Rate & 2e-4 \\
 & Train Batch Size & 512 \\
 & Training Steps & 500000 \\
\bottomrule
\end{tabular}
%\caption{Model and training hyperparameters for MINDE in image experiments. We report key parameters for the estimator, SDE, and training process (images).}
%\label{tab:hyperparameters_minde_mnist}
%\end{table}

    \end{subtable}
    \caption{Model and training hyperparameters for MINDE in synthetic experiments (left) and experiments with MNIST (right).
    We report key parameters for the estimator, SDE, and training process.}
\label{tab:hyperparameters_minde_combined}
\end{table}

\subsection{Additional details}\label{app:minde_sc_proof} In this section we provide the proof of \cref{prop:minde-sc}. By the Cauchy--Schwarz inequality, the cross term in $\mathcal{D}$ satisfies
\[
\Big|\int_{t_0}^{T}\mathbb{E}\langle s_c-s_u,\,e_c-e_u\rangle\,\mathrm{d}t\Big|
\le\Big(\int_{t_0}^{T}\mathbb{E}\|s_c-s_u\|^2\mathrm{d}t\Big)^{1/2}\Big(\int_{t_0}^{T}\mathbb{E}\|e_c-e_u\|^2\mathrm{d}t\Big)^{1/2}
=2\sqrt{I_{[t_0,T]}\,\Delta},
\]
which gives the first inequality.
For the second, the inequality $\|a-b\|^2\le2\|a\|^2+2\|b\|^2$ gives
\[
\Delta\le2\,(T-t_0)\cdot\frac{1}{T-t_0}\int_{t_0}^{T}\Big(\tfrac12\mathbb{E}\|e_c\|^2+\tfrac12\mathbb{E}\|e_u\|^2\Big)\mathrm{d}t .
\]
The loss in Proposition~\ref{prop:minde-sc} is the denoising score-matching loss of Theorem~\ref{thm:zlc} for a network with the additional inputs $(y,c)$, in which only $x$ is diffused.
By the denoising identity of \citep{vincent2011connection}, applied conditionally on $Y$ when $c=1$, its population minimiser is $s_c$ for $c=1$ and $s_u$ for $c=0$.
The corresponding $\ell^2(p_t)$ error is therefore $\tfrac12\mathbb{E}\|e_c\|^2+\tfrac12\mathbb{E}\|e_u\|^2$.
Theorem~\ref{thm:zlc} thus applies with $d$ replaced by $d_X$; the Lipschitz assumption on $(x,y)\mapsto s_c$ is what its approximation step requires.
With probability at least $1-2\delta-2L\exp(-\Omega(m))$, it gives
\[
\frac{1}{T-t_0}\int_{t_0}^{T}\Big(\tfrac12\mathbb{E}\|e_c\|^2+\tfrac12\mathbb{E}\|e_u\|^2\Big)\mathrm{d}t
\lesssim\frac{1}{n\varepsilon^2}\Big(\frac{d_X(T-\log t_0)}{T-t_0}+C_d^2\Big)\log\frac{\mathcal{N}_c(1/n,\mathcal{S})}{\delta}+\frac1n+d_X\varepsilon^2
=\frac{B_n(\varepsilon,T,t_0)}{T-t_0},
\]
hence $\Delta\lesssim2B_n$.
The last claim follows by substituting this bound into the first inequality and using $I_{[t_0,T]}\le I(X;Y)$.

\end{document}